\documentclass{article}
\usepackage{arxiv}
\usepackage{graphicx}
\usepackage{float}
\usepackage{hyperref}       
\usepackage{url}            
\usepackage{booktabs}       

\usepackage{tabularx}
\usepackage{amsfonts}       
\usepackage{nicefrac}       
\usepackage{microtype}      
\usepackage{xcolor}         
\usepackage{hyperref}       
\usepackage{url} 
\usepackage{diy}
\usepackage{enumitem}
\usepackage{subcaption}
\usepackage[english]{babel}
\usepackage{romannum}
\usepackage{cleveref}
\usepackage{parskip}
\usepackage{overpic}
\usepackage{geometry}
\theoremstyle{remark}
\newtheorem{remark}{Remark}
\theoremstyle{plain}
\newtheorem{theorem}{Theorem}[section]
\newtheorem{lemma}{Lemma}[section]

\theoremstyle{definition}
\newtheorem{definition}{Definition}[section]
\newtheorem{hypothesis}{Induction Hypothesis}[section]

\definecolor{myorange}{RGB}{245,156,74}
\hypersetup{
	colorlinks=true,
	filecolor=blue,      
	citecolor=-myorange,
}
\usepackage{tcolorbox}
\usepackage{titletoc}
\usepackage[page, header]{appendix}

\usepackage{colortbl}
\definecolor{gray}{rgb}{0.85, 0.85, 0.85}
\newcolumntype{L}{>{\hsize=.6\hsize}X} 
\newcolumntype{R}{>{\hsize=1.5\hsize}X}

\theoremstyle{definition}

\title{In-Context Convergence of Transformers}

\author{Yu Huang\thanks{Department of Statistics and Data Science, University of Pennsylvania. {\tt yuh42@wharton.upenn.edu} }
\and
Yuan Cheng\thanks{National University of Singapore. {\tt yuan.cheng@u.nus.edu}}
\and
Yingbin Liang \thanks{Department of Electrical and Computer Engineering, The Ohio State University. {\tt  liang.889@osu.edu}}
}
\begin{document}

\maketitle
\begin{center}
    \today
\end{center}
\vspace{0.3cm}
\pagenumbering{arabic}
\begin{abstract}
Transformers have recently revolutionized many domains in modern machine learning and one salient discovery is {their} remarkable in-context learning capability, where models can solve an unseen task by utilizing task-specific prompts without further parameter{s} fine-tuning. 
This also inspired recent theoretical studies aiming to understand the in-context learning mechanism of transformers, which however focused only on {\em linear} transformers. 
In this work, we take the first step toward studying the learning dynamics of a one-layer transformer with {\em softmax} attention trained via gradient descent in order to in-context learn linear function classes. 
We consider a structured data model, where each token is randomly sampled from a set of feature vectors in either balanced or imbalanced fashion. For data with balanced features, we establish the finite-time convergence guarantee with near-zero prediction error by navigating our analysis over two phases of the training dynamics of the attention map.
More notably, for data with imbalanced features, we show that the learning dynamics take a stage-wise convergence process, where the transformer first converges to a near-zero prediction error for the query tokens of dominant features, and then converges later to a near-zero prediction error for the query tokens of under-represented features, respectively via one and four training phases. Our proof features new techniques for analyzing the competing strengths of two types of attention weights, the change of which determines different training phases.

\end{abstract}

\section{Introduction}
Transformers~\cite{vaswani2017attention} have emerged as the foundational architectures in various domains, 
including natural language processing~\cite{devlin2018bert,openai2023gpt4}, computer vision~\cite{dosovitskiy2020image, he2022masked}, 
reinforcement learning~\cite{chen2021decision,janner2021offline}, and so on. Recently, large language models (LLMs) based on transformers have exhibited remarkable in-context learning capabilities, where the model can solve a new task solely through inference 
based on prompts of the task without further fine-tuning~\cite{brown2020language}.

Such striking abilities have inspired a recent line of research to understand the underlying mechanisms of in-context learning from various aspects~\cite{garg2022can,min2022rethinking,wei2023larger,von2023transformers, xie2021explanation}. Among these studies, the pioneering work of \cite{garg2022can} empirically studied in-context learning via an interpretable framework, highlighting the capacity of transformers to acquire in-context knowledge of {linear and some more complex function classes.} Specifically, 
they showed that an in-context trained model over a function class $\cF$ can accurately predict the function value $f\left(\xq\right)$ of a new query token $\xq$ for most $f\in\cF$ by using a prompt sequence including in-context input-label pairs along with the query token $\left(x_1, f\left(x_1\right), \ldots, x_N, f\left(x_N\right), \xq\right)$. 

Built on this theoretically amenable setting, many follow-up works explored theoretical properties of in-context learning of transformers from different perspectives such as expressive power~\cite{akyurek2022learning, giannou2023looped}, generalization~\cite{li2023transformers}, internal mechanisms~\cite{von2023transformers,bai2023transformers}, etc. Specially, a few recent studies~\cite{zhang2023trained,mahankali2023one,ahn2023transformers} made interesting progress towards understanding the training dynamics of transformers for in-context learning
. However, those studies focused only on `linear' transformers, 
and does not capture the crucial role of the `softmax' mapping, which lies in the core design of 
transformers to be advantageous over other network architectures.
Therefore, the following fundamental problem still remains largely open:
\begin{center}
   \it How do \textbf{softmax}-based transformers trained via gradient descent learn in-context?
\end{center}
This paper takes the first step toward addressing this problem by investigating the learning dynamics of a single-layer transformer with {\em softmax} attention trained by gradient descent (GD) for in-context learning. We focus on the setting with training prompts generated from linear regression models as in \cite{garg2022can}, and with structured input data, where each token is randomly selected from a set of feature vectors $\{v_k\}_{k=1}^{K}$ with probability $\{p_k\}_{k=1}^{K}$, respectively. We then train the transformer over the squared loss of prediction error using GD. We study the training dynamics under both balanced and imbalanced feature distributions, 
and characterize the in-context learning ability for both settings. 
We highlight our contributions as follows.
\paragraph{Our Contributions.}
\begin{itemize}
    \item We first establish the convergence guarantee for the setting with balanced features, where $p_{k}=\Theta\left(\frac{1}{K}\right)$ for each $k\in[K]$, and characterize the training evolution of the attention map into a two-phase dynamic process. In the first phase, for each $k\in[K]$, the parameters of the self-attention module undergo fast growth, aligning the query token featuring $v_k$ with input tokens featuring $v_k$ rapidly disregarding other feature directions. In the second phase, the loss of prediction error converges to a near-minimum value.
    \item 
    We then prove the convergence for the setting with imbalanced features, where one feature dominates, say $v_1$ with $p_1=\Theta(1)$, while others are under-represented with $p_k=\Theta\left(\frac{1}{K}\right)$ for $k>1$, which serves as a remarkable showcase of the in-context learning capabilities of transformers. We demonstrate that the learning dynamics display a {\em stage-wise} convergence process. Initially, the transformer quickly attains near-zero prediction error for the query tokens of dominant features, and then converges to near-zero prediction error for the query tokens of under-represented features, irrespective of their infrequent occurrence, through one and four phases, respectively.
    


    \item Our analysis hinges on a novel proof technique that 
    characterizes the {\em softmax} attention dynamics via the interplay between two types of bilinear attention weights: `weight of query token and its target feature' and `weight of query token and off-target features'. 
    Which weight plays a dominant role in the attention dynamics can change over the learning process, resulting in different training phases.
   Our analysis tools may be of independent interest and hold the potential to study various other problems 
   involving transformer architectures.


    
\end{itemize}
\textbf{Notations.} We let $[K]:=\{1,2, \ldots, K\}$. We use capital letters for matrices (e.g., $A$), and lowercase letters for vectors and scalars (e.g., $a$). For a general matrix $A$, we use $A_{i}$ to represent the $i$-th column of $A$ and $A_{i:j}$ to indicate a collection of columns spanning from $i$ to $j$.  
We use $\mathbf{1}\{\cdot\}$ to denote the indicator function.
We use $O(K)$, $\Omega(K)$, and $\Theta(K)$ to omit universal constants concerning the variable $K$. 
 We use $\poly(K)$ and $\polylog(K)$ to denote large constant-degree polynomials of $K$ and $\log(K)$, respectively. 
 Given $h(x)\leq 0$ and $g(x)>0$, we denote $h(x)=-\Omega(g(x))$ if there exists some constant $C_1>0$ and $a_1$, s.t. $|h(x)|\geq C_1g(x)$ for all $x\geq a_1$; $h(x)=-O(g(x))$ if there exists some constant $C_2>0$ and $a_2$, s.t. $|h(x)|\leq C_2g(x)$ for all $x\geq a_2$.
\section{Problem Setup}
In this section, we present our problem formulations, including the in-context learning framework, one-layer transformer architecture, and the training settings we consider in this paper. 
\subsection{In-Context Learning Framework}
We adopt the well-established in-context learning framework as given in~\cite{garg2022can}. The objective is to enable the training of models capable of in-context learning 
within a specified function class $\mathcal{F}$, where the functions and input data are sampled respectively by the distributions $D_{\mathcal{F}}$ and $D_{\mathcal{X}}$. Specifically, the process is initiated by generating random training prompts as follows. For each prompt, we first sample a random function $f$ from the class according to the distribution $D_{\mathcal{F}}$. We then create a set of random inputs $x_1, \ldots, x_{N}$ and query $x_{\text {query }}$, all drawn independently by $D_{\mathcal{X}}$. Finally, we compute the value of function $f$ on these inputs to construct the prompt $P=\left(x_1, y_1, \ldots, x_{N}, y_N, \xq\right)$, where $y_{i}=f(x_i)$. The goal for an in-context learner is to use the prompt to form a prediction $\widehat{y}\left(\xq\right)$ for the query such that $\widehat{y}\left(\xq\right) \approx f\left(\xq\right)$.

\paragraph{Task Distribution.} In this work, our focus is on the task of linear functions  defined as $\mathcal{F}=\left\{f: \cX\rightarrow \mathbb{R} 
 \mid f(x)=\langle w,  x \rangle \text{ with } w \in 
 \mathbb{R}^d, \cX\subset \mathbb{R}^d \right\}$, which is widely adopted in recent studies for in-context learning~\cite{ahn2023transformers,zhang2023trained,mahankali2023one}. For each prompt, the task-specific weight $w$ is independently drawn from a task distribution $\cD_{\Omega}$ with zero mean and identity covariance matrix $\mathbf{I}_{d\times d}$. 
\paragraph{Data Distribution.} To specify the data distribution $\cD_{\cX}$, we consider a set of distinct features $\{v_k\in \mathbb{R}^{d}, k=1,\ldots,K\}$, where all features are orthonormal vectors. Each data point $x$ is sampled from the feature set with the probability ${p_k}$ for sampling $v_k$, where $p_k\in (0,1)$ for $k\in[K]$ and $\sum_{k\in[K]}p_k=1$.
Such a data model has been widely employed in the theoretical studies 
of deep learning, including ensemble methods~\cite{allen2022towards}, multi-modal learning~\cite{huang2022modality}, vision transformers~\cite{li2023theoretical}, etc.

\subsection{One-Layer Transformer Architecture}
To present the one-layer transformer model we consider in this work, we first introduce the self-attention mechanism~\cite{bahdanau2014neural,vaswani2017attention} for the transformer model.

\begin{definition}[Self-Attention (SA) Mechanism]
    A self-attention layer~\cite{bahdanau2014neural,vaswani2017attention} in the single-head case with width $d_e$ consists of the following components: a key matrix $W_{K}\in \mathbb{R}^{d_e\times d_e}$, a query matrix $W_{Q}\in \mathbb{R}^{d_e\times d_e}$, and a value matrix $W_{V}\in \mathbb{R}^{d_e\times d_e}$.
    Given a prompt 
    $P$ of length $N$, let $E\in\mathbb{R}^{d_e\times d_N}$ be an embedding matrix of the prompt $P$, and the self-attention mechanism will output:
    \begin{align}
        F_{\text{SA}}\left(E ; W^K, W^Q, W^V\right) = W^V E \cdot \operatorname{softmax}\left({\left(W^K E\right)^{\top} W^Q E}\right),\label{SA}
    \end{align}
    where 
    the $\operatorname{softmax}(\cdot)$ function is applied column-wisely, i.e., for a vector input $z$, the $i$-th entry of $\operatorname{softmax}(z)$ is given by $e^{z_i} / \sum_s e^{z_s}$.
\end{definition}

\paragraph{Embeddings.} For in-context learning, given a prompt $P=\left(x_1, y_1, \ldots, x_N, y_N, x_{\text {query }}\right)$, a natural token embedding is to stack $x_i\in \mathbb{R}^{d}$ and $y_i$
into the first $N$ columns. The final column consists of  $\xq \in \mathbb{R}^{d}$ and $0$. Formally,
$$
E=E(P)=\left(\begin{array}{llllc}
x_1 & x_2 & \cdots & x_N & \xq \\
y_1 & y_2 & \cdots & y_N & 0
\end{array}\right) \in \mathbb{R}^{(d+1) \times(N+1)}.
$$
Therefore,  $d_{N}=N+1$ and $d_{e}=d+1$ in the above embedding. Let us further denote the first $d$ rows of $E$ as $E^x(P)\in\mathbb{R}^{d\times (N+1)}$  and the last row of $E$ as $E^y(P)\in\mathbb{R}^{1\times (N+1)}$. Then we write $E(P)=\{E^x(P),E^y(P)\}$. We omit the dependency on $P$ for $E(P)$, $E^x(P)$ and $E^x(P)$ when there is no ambiguity. 

We next instantiate additional operations and certain parameter settings based on the general SA mechanism (\ref{SA}) for our one-layer transformer model to mitigate unnecessary complications in theoretical analysis while keeping the most critical component of the SA mechanism. 
\paragraph{Masking.} Let $M(\cdot)$ denote the masking operation, which masks (removes) the last column of the entry matrix. In other words, for a given matrix $A\in \mathbb{R}^{(d+1)\times (N+1)}$, $M(A)$ yields  $A_{1:N}\in\mathbb{R}^{(d+1)\times N}$.  
We will first mask the embedding matrix $E$ before its multiplication with the key matrix $W^{K}$ and the value matrix $W^V$, which results in $W^{K}M(E)$ and $W^{V}M(E)$,
in order to prevent the query token from attending to itself. This approach has been commonly taken in previous works~\cite{tian2023scan,mahankali2023one,von2023transformers,kitaev2019reformer}.
\paragraph{Reparameterization.}
We consolidate the query and key matrices into one matrix denoted as $W^{KQ}\in \mathbb{R}^{(d+1)\times (d+1)}$, often taken in recent theoretical frameworks 
~\cite{zhang2023trained,jelassi2022vision,tian2023scan}. Furthermore, we consider $W^V$ and $W^{KQ}$ in the following specific forms:
\begin{align}
    W^{V}=\left(\begin{array}{cc}
        0_{d \times d} & 0_d \\
        0_d^{\top} & \nu
        \end{array}\right), \quad W^{K Q}=\left(\begin{array}{cc}
    Q & 0_d \\
        0_d^{\top} & 0
        \end{array}\right),\label{specialform}
\end{align}
where $\nu\in\mathbb{R}$ and $Q\in \mathbb{R}^{d\times d}$. The above structures of $W^{V}$ and $W^{KQ}$ are inspired by the recent study~\cite{zhang2023trained}, which showed that such structured matrices achieve the global optimum in the {\em linear} SA model. Furthermore, we set $\nu=1$ (where $\nu$ is the only parameter in $W^V$) and do not update it during the training. The reason is twofold: 1) this aligns with the common practice in theoretical studies of deep learning, where the last linear layer is often kept fixed to focus on the analysis of hidden layers. Our objective remains highly nonconvex and challenging even with a fixed $\nu$; and 2) the form of the global optimum outlined in recent work~\cite{zhang2023trained} suggests that for {\em linear} SA, the optimal solution for $\nu$ serves as a scaling factor to normalize the output of linear attention. In our case, the output of {\em softmax} attention is already inherently normalized.   
\begin{remark}[Nealy no loss of optimality]
\em Despite the specific form of $\{W^V, W^{KQ}\}$ that we take, the minimum of the loss function $L^{*}=\Theta(e^{-\poly(K)})$ (as shown in \Cref{thm: bal K}) implies that such a specific form at most incurs an error of $\Theta(e^{-\poly(K)})$ that vanishes exponentially with $K$,  compared to the minimum loss over the general parameter space $\{W^V, W^{K}, W^{Q}\}$. Therefore,  for our nonlinear {\em softmax} SA, such specific parameterization does not lose optimality. 
\end{remark}



With the aforementioned masking operations and reparameterization, the overall transformer model consisting of a single SA layer can be recast in the parameterization of $\theta=\{1, Q\}$ as follows: 
\begin{equation}\label{model}
F_{\text {SA}}\left(E ;\theta\right)= {M(E^y)}\cdot \operatorname{softmax}\left({{M(E^x)}^{\top} Q E^x}
\right).
\end{equation}
Such a reparameterization separates the label $E^{y}$ from the $\operatorname{softmax}$ operator while maintaining simultaneous processing of both input $E^{x}$ and label $E^{y}$ information. 
The prediction for the token $\xq$ 
will be the last entry of $F_{\text{SA}}$, namely,
$$
\yq=\yq(E;\theta)=\left[F_{\text{SA}}(E ; \theta)\right]_{(N+1)} \text {. }
$$

Henceforth, we may omit the reference to $E$ and $\theta$,  and use  $\widehat{y}_{\text{query}}$ if it is not ambiguous.

\subsection{Training Settings}
\paragraph{Loss Function.} To train the transformer model $F_{\text{SA}}$ over linear regression tasks, we minimize the following squared loss of the prediction error, which has also been taken by~\cite{zhang2023trained,ahn2023transformers}:
\begin{align}
    L(\theta)=\frac{1}{2} \mathbb{E}_{w\sim \cD_{\Omega},\left\{x_{i}\right\}_{i=1}^N \cup\left\{x_{\text {query }}\right\}\sim\mathcal{D}_{\cX}^{N+1}}\left[\left(\widehat{y}_{\text { query }}-\left\langle w, x_{\text { query }}\right\rangle\right)^2\right]\label{eq:obj}
\end{align}
where the expectation is taken with respect to the prompt $P$ including input and query tokens $\left\{x_{i}\right\}_{i=1}^N \cup\left\{x_{\text {query }}\right\}$ and the weight vector $w$. In the following, we omit subscripts of the expectation to simplify the notation.

\paragraph{Training Algorithm.} The above learning objective in \cref{eq:obj} is minimized via GD with the learning rate $\eta$. At $t=0$, we initialize $Q^{(0)}$ as zero matrix $\mathbf{0}_{d\times d}$. The parameter is updated as follows:
\begin{align*}
   \theta^{(t+1)} =\theta^{(t)}-\eta \nabla_{\theta} L(\theta^{(t)}).
\end{align*}

\section{Main Results}


In this section, we characterize the convergence of in-context learning by GD for the settings with balanced and imbalanced features, respectively.

To measure the degree to which the query token $\xq$ attends to the specific input token and to a certain class of features, we define the following notions of the attention scores.
\begin{definition}[Attention Score]\label{attn} Given a prompt $P=(x_1,y_1,\cdots,x_N,y_N, x_{\text{query}})$ and its corresponding embedding $E$, then at time $t$, for $F_{\text{SA}}$ with parameter $\theta^{(t)}$, we define the attention score as follows.
  \begin{enumerate}[label={\arabic*.}]
      \item Given $i\in[N]$, the attention score for the $i$-th token $x_{i}$ is 
$$\attn_{i}(\theta^{(t)};E):=\left[\operatorname{softmax}(M(E^x)^{\top}Q^{(t)}E^{x})\right]_{i}=\frac{e^{{E^{x}_{i}}^{\top}Q^{(t)}E^{x}_{N+1}}}{\sum_{j\in [N]}e^{{E^{x}_{j}}^{\top}Q^{(t)}E^{x}_{N+1}}}.
$$

\item For $k\in[K]$, denote $\cV_k(P)\subset[N]$ as the index set for input tokens, such that $x_i=v_k$ for $i\in\cV_k(P)$. 
Then the attention score for the $k$-th feature is given by
$$
\Attn_{k}(\theta^{(t)};E):=\sum_{i\in\cV_k(P)}\attn_{i}(\theta^{(t)};E).
$$
  \end{enumerate}
  For simplicity, we represent $\attn_{i}(\theta^{(t)};E)$ and $\Attn_{k}(\theta^{(t)};E)$ as $\attn_{i}^{(t)}$ and $\Attn_{k}^{(t)}$, respectively, and denote $\cV_k(P)$ as $\cV_{k}$. We also rewrite the prediction output at time $t$ as follows:
  \begin{align}\label{eq:output}
      \hat{y}^{(t)}_{\text {query}}
      &  =\sum_{i\in[N]} \attn^{(t)}_{i} {y}_{i}=\sum_{k\in[K]}\Attn^{(t)}_k\langle w,v_k\rangle.
  \end{align}
\end{definition}

\subsection{In-Context Learning with Balanced Features}\label{sec3.1}

In this subsection, we study in-context learning with {\em balanced} features, where the probabilities of sampling all $K$ features are in the same order, i.e., $p_k=\Theta\left(\frac{1}{K}\right)$ for each $k \in [K]$. In such a setting, each feature appears equally likely in the prompt, ensuring their equal recognition. 
The following theorem characterizes the convergence of GD.
\begin{theorem}[In-context Learning with Balanced Features]\label{thm: bal K} 
Suppose $p_k=\Theta\left(\frac{1}{K}\right)$ for $k \in [K]$.
For any $0<\epsilon<1$, suppose $N\geq \poly(K)$ and $\polylog(K)\gg \log(\frac{1}{\epsilon}) $. We apply GD to train the loss function given in \cref{eq:obj}.
Then with at most $T^*=O(\frac{\log(K)K^2}{\eta}+\frac{K \log\left(K\epsilon^{-\frac{1}{2}}\right)}{\epsilon\eta})$ iterations, we have
\begin{enumerate}[label={\arabic*.}]
\item The loss converges: $L(\theta^{(T^*)})-L^* \leq \epsilon$, where $L^{*}=\Theta(e^{-\poly(K)})$ is the global minimum of the population loss in \cref{eq:obj}. 
\item Attention score concentrates: if $\xq=v_k$, then with probability at least  $1-e^{-\Omega(\poly(K))}$\footnote{The randomness originates from the first $N$ input tokens in the test prompt.},
 the one-layer transformer nearly ``pays all attention" to input tokens featuring $v_k$, i.e., $(1-\Attn_k^{(T^*)})^2 \leq O(\epsilon)$. 
\end{enumerate}
\end{theorem}
\Cref{thm: bal K} shows that training a one-layer transformer with softmax attention can converge to the minimum of the objective loss in the reparameterization space via GD, with polynomial time efficiency with respect to $K$ and $\frac{1}{\epsilon}$. The learning dynamics for such a case with balanced features exhibit a {\bf two-phase behavior}. (i) The first term of $T^{*}$ captures the duration of phase I, where the network actively aligns the query token (suppose $\xq=v_k$) with those tokens featuring $v_k$ itself, thus substantially increasing $\Attn_{k}^{(t)}$ to a constant level. 
(ii) The second term captures the duration of phase II, where the loss converges to the near-zero prediction error. 

 \textbf{In-context Learning Ability}. For the obtained model with $\theta^{{(T^{*})}}$, let us evaluate a test prompt associated with a linear task $w$, which might not be drawn from the support of $\cD_{\Omega}$ (i.e., $w$ may not be present in the training process), but has its data drawn by $\cD_{\cX}$. Suppose the query token is $\xq=v_k$. Following from the attention score concentration principle in \Cref{thm: bal K}, \cref{eq:output} yields that with high probability the query prediction is given by 
 \begin{align*}
     \yq^{{(T^{*})}}= \Attn^{{(T^{*})}}_{k}\langle w,v_k\rangle+\sum_{m\not=k}\Attn^{{(T^{*})}}_{m}\langle w,v_m\rangle\approx\langle w,v_k\rangle.
 \end{align*}
This implies that the in-context learned model can still well approximate the test prompt even if the task model $w$ does not lie in the support of the training task distribution $\cD_{\Omega}$ and was {\em unseen} during training. This showcases the remarkable in-context learning capability of trained transformers.

\subsection{ In-Context Learning with Imbalanced Features}
In real-world datasets, skewed distributions are common, where a few classes or features dominate in data while others are under-represented. It is typically difficult to train models to perform well on features that have limited representation in those datasets~\cite{cui2019class,chou2020remix}. In this subsection, we investigate the setting with imbalanced features, where the dominant feature $v_1$ is sampled with the probability $p_1=\Theta(1)$, and all other features are sampled with $p_k=\Theta\left(\frac{1}{K}\right)$ for $2\leq k \leq K$. We will show that somewhat remarkably, in-context learning is less sensitive to imbalanced features and can achieve a near-zero error even when the query token takes an under-represented feature.


To investigate the performance for the imbalanced scenario, we focus on the following prediction error for each feature $v_k$:
\begin{align}
    \cL_{k}(\theta) &= \frac{1}{2}\mathbb{E}\left[\left(\yq-\left\langle w, \xq \right\rangle\right)^2\big| \xq=v_k\right].\label{eq-obj-k}
\end{align}
The following theorem characterizes the convergence of GD.
\begin{theorem}[In-context Learning with Imbalanced Features]\label{thm: unblc} 
Suppose $p_1=\Theta(1)$ and $p_k=\Theta\left(\frac{1}{K}\right)$ for $2\leq k \leq K$. For any $0<\epsilon<1$, suppose $N\geq \poly(K)$, and $\polylog(K)\gg \log(\frac{1}{\epsilon})$. We apply GD to train the loss function given in \cref{eq:obj}. Then the following results hold.
\begin{enumerate}[label={\arabic*.}]
    \item The prediction error for the \textbf{dominant} feature converges: for $v_1$, with at most $T_1=O(\frac{\log(\epsilon^{-\frac{1}{2}})}{\eta\epsilon})$ GD iterations, 
    $\cL_1(\theta^{(T_1)})\leq \cL_{1}^{*}+\epsilon$, where $\cL^{*}_{1}=\Theta(e^{-\poly(K)})$ is the global minimum of \cref{eq-obj-k} for $k=1$; 
    \item The prediction error for the \textbf{under-represented} features converges: for $v_k$ with $2\leq k \leq K$, 
    with at most 
    $T_k=O(\frac{\log(K)K^2}{\eta}+\frac{K \log\left(K\epsilon^{-\frac{1}{2}}\right)}{\epsilon\eta})$ GD iterations, 
    $\cL_k(\theta^{(T_k)})\leq \cL_{k}^{*}+\epsilon$, where $\cL^{*}_{k}=\Theta(e^{-\poly(K)})$ is the global minimum of \cref{eq-obj-k};
    \item Attention score concentrates: for each
 $k\in[K]$, if the query token is $v_k$, then after $T_{k}$ 
 iterations, with probability at least  $1-e^{-\Omega(\poly(K))}$, the one-layer transformer nearly ``pays all attention" to input tokens featuring $v_k$:  $(1-\Attn^{(T_{k})}_k)^2 \leq O(\epsilon)$.
\end{enumerate}
%
    \end{theorem}

    
    \Cref{thm: unblc} shows that the GD dynamics of the in-context training exhibit {\em `stage-wise'} convergence. The trained transformer rapidly (within $T_1$) converges to a model that achieves a near-zero prediction error $\cL_1$ for the dominant feature; and then takes a much longer time (up to $T_k\gg T_1$) to converge to a model that attains a near-zero prediction error $\cL_k$ for the under-represented features. Our analysis captures the later learning dynamics associated with the under-represented features into a four-phase behavior as further described in the subsequent section. Despite the longer convergence time it takes, in-context learning still achieves the same accurate prediction for under-represented features as that for the dominant feature.

    
  

\section{Overview of Training Phases}
In this section, we explain our key ideas for analyzing the in-context learning capabilities of transformers. We will first characterize the training process of the setting with imbalanced features for under-represented features in \Cref{sec:under}, which comprehensively exhibits four phases. Other scenarios take only one or two of those phases, which we will briefly describe in \Cref{sec:other}. The complete proofs of all the results are provided in the appendix.



We will first provide the general training dynamics for the {\em bilinear attention weights} (defined in \Cref{def:dynamics} below), which is useful for analyzing all learning phases. These quantities are the key elements in the attention scores $\attn^{(t)}_i$ for $1\leq i \leq N$, which play an important role in determining the prediction $\yq^{(t)}$. Hence, our analysis mainly tracks the training dynamics of those bilinear attention weights. 
\begin{definition}\label{def:dynamics}(Bilinear Attention Weights)
  Given $k,n\in[K]$, where $k\not=n$, for $t\geq 0$, we define the bilinear attention weights as follows:
\begin{align*}
&A_{k}^{(t)}:=v_{k}^{\top}Q^{(t)} v_{k},
\quad 
B_{k,n}^{(t)}:=v_{n}^{\top}Q^{(t)} v_{k}.
\end{align*}
By our initialization, we have $A_{k}^{(0)}=B_{k,n}^{(0)}=0$.
\end{definition}
To further interpret these weights, suppose the query token corresponds to the feature $v_k$. Then $e^{A_{k}^{(t)}}$ serves as the (un-normalized) weight for the input token featuring $v_k$, while $e^{B_{k,n}^{(t)}}$ captures the weight for the input token featuring a different vector $v_{n}$ with $n\not=k$. Having a larger $A_{k}^{(t)}$ compared to other $B_{k,n}^{(t)}$ indicates a better capture of the target feature $v_k$. As shown in \cref{eq:output}, this condition implies a higher `attention' towards input tokens featuring $v_k$, resulting in $\yq^{(t)}\approx \sum\limits_{i\in\cV_k}\attn^{(t)}_iy_i\approx\langle w,v_k\rangle$, where the prediction well approximates the ground truth. 

The following lemma provides the GD updates of the bilinear attention weights $A_{k}^{(t)}$ and $B_{k,n}^{(t)}$.
\begin{lemma}\label{lem:gd}
    Let $t\geq 0$. For $k,n\in[K]$, where $k\not=n$, $A_{k}^{(t)}$ and $B_{k,n}^{(t)}$ satisfy:
    \begin{align*}
&A_{k}^{(t+1)}=A_{k}^{(t)}+\eta\alpha_{k}^{(t)},  \qquad B_{k,n}^{(t+1)}=B_{k,n}^{(t)}+\eta\beta_{k,n}^{(t)},\\
&\alpha_{k}^{(t)}=\mathbb{E}\left[\mathbf{1}\{x_{\text {query }}=v_k\}\Attn^{(t)}_{k }\cdot \left(
 \sum_{m\not= k}{\Attn^{(t)}_{m}}^2+(1-\Attn^{(t)}_k)^2\right)\right],\\
 &\beta_{k,n}^{(t)}=\mathbb{E}\left[\mathbf{1}\{x_{\text {query }}=v_k\}\Attn^{(t)}_{n}\cdot \left(
 \sum_{m\not= k}{\Attn^{(t)}_{m}}^2-\Attn^{(t)}_{n} -\Attn^{(t)}_{k}(1-\Attn^{(t)}_k)\right)\right].
    \end{align*}
\end{lemma}
\Cref{lem:gd} shows that $A_{k}^{(t)}$ is monotonically increasing at any time since $\alpha_{k}^{(t)}\geq 0$, whereas the monotonicity does not always hold for $B_{k,n}^{(t)}$. Therefore, we need to analyze whether $B_{k,n}^{(t)}$ decreases and determine its rate of change compared to $A_{k}^{(t)}$. Such a comparison between $B_{k,n}^{(t)}$ and $A_{k}^{(t)}$ determines which bilinear weight plays a dominant role in the attention dynamics, and the change of the leading weight over the learning process results in different training phases.



\begin{figure}
    \centering
    \begin{overpic}[width=\textwidth]{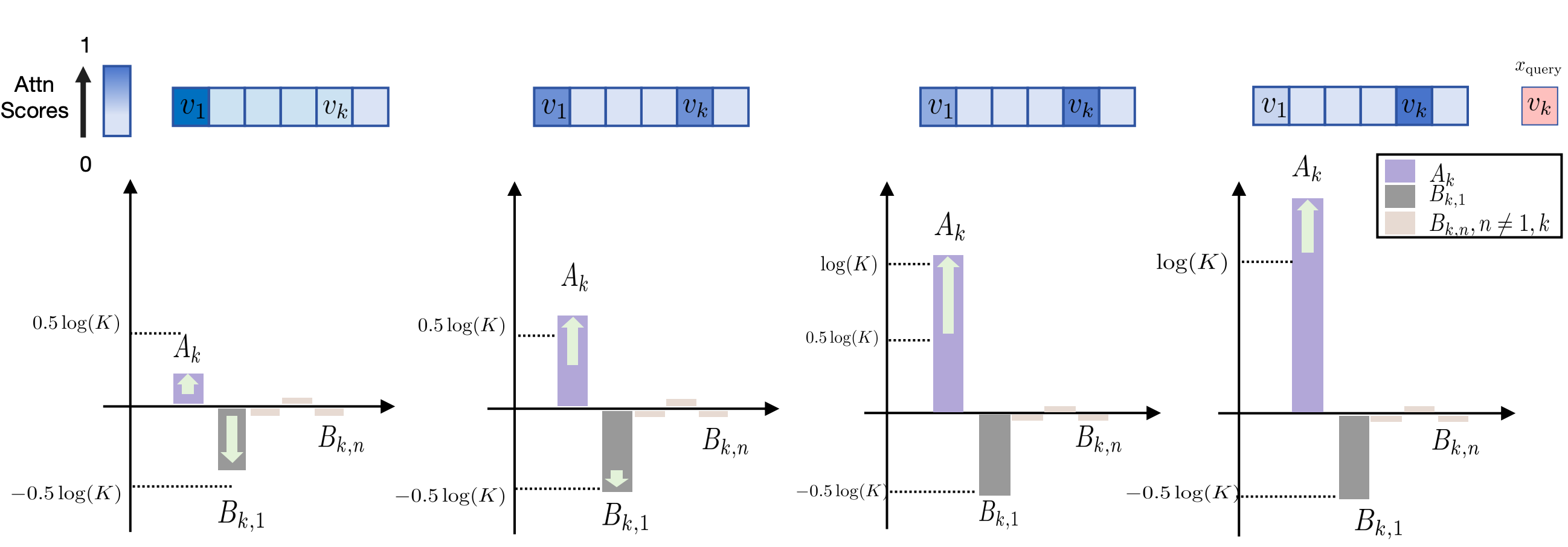} 
    \put(7.5,32){ \fontsize{6.3pt}{10pt}\selectfont(a) Decrease of Dominant Feature} 
    \put(32,32){ \fontsize{6.3pt}{10pt}\selectfont (b) Switching of Leading Influence}
    \put(56,32){ \fontsize{6.3pt}{10pt}\selectfont (c) Growth of Target Feature}
    \put(78,32){ \fontsize{6.3pt}{10pt}\selectfont (d) Convergence}
    \end{overpic}
    \caption{Overview of the dynamics of attention scores and bilinear attention weights for under-represented features. Assume the query token is $v_k$ with $2\leq k\leq K$. 
    The {\em top} row depicts the trend of the attention score $\Attn^{(t)}_m$ for each feature $v_m$, where a darker color corresponds to a higher score. The {\em bottom} row shows the interplay and leading effect among bilinear attention weights $A^{(t)}_k,B^{(t)}_{k,1}$, and $B^{(t)}_{k,n}$ (where $n \neq 1,k$) in different training phases. 
    \textbf{(a)} Phase I: $B^{(t)}_{k,1}$ significantly decreases and the attention on tokens with the dominant feature $v_1$ is suppressed (\Cref{sssec: 4-1-p1}); 
     \textbf{(b)} Phase II: With the suppression of $\Attn_{1}^{(t)}$, the decreasing rate for $B^{(t)}_{k,1}$ drops and the growth of $A_{k}^{(t)}$ becomes the leading influence (\Cref{sssec: 4-1-p2}); 
    \textbf{(c)} Phase III: $A^{(t)}_k$ rapidly grows and $\Attn^{(t)}_k$ reaches $\Omega(1)$  (\Cref{sssec: 4-1-p3}); \textbf{(d)} Phase IV: $\Attn^{(t)}_k$ nearly grows to $1$ and the prediction error converges to a global minimum (\Cref{sssec: 4-1-p4}).}
    \label{fig:enter-label}
    \vspace{-0.2cm}
\end{figure}
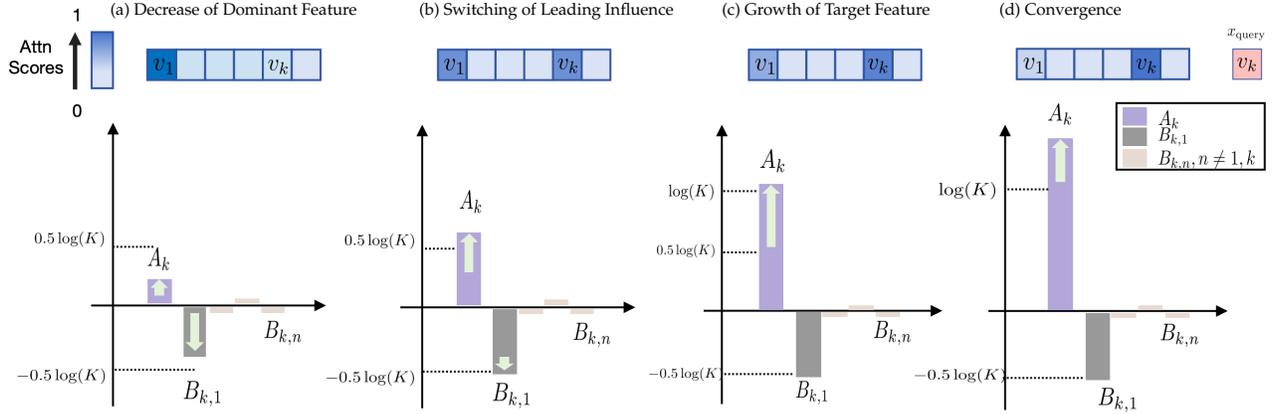
\subsection{Learning Process for Under-represented Features}\label{sec:under}

We consider the setting with imbalanced features and focus on the under-represented features.

Given a prompt $P=(x_1,y_1,\cdots,x_N,y_N, x_{\text{query}})$, denote $\pit$ to be the collection of input tokes, i.e., $\{x_i\}_{i=1}^{N}$. Recall that $|\cV_k|$ is the number of input tokens featuring $v_k$. Based on our data generation setup, we can show that for imbalanced data, with high probability, $\pit$ belongs to  
\begin{align*}   \esi:=\left\{\pit: |\cV_1|=\Theta(N), |\cV_k|=\Theta\Big(\frac{N}{K}\Big)\text{ for   } 2\leq k\leq K\right\}.
\end{align*}
In the following, we focus on the event that $\pit\in\esi$ unless otherwise specified. We next characterize the learning process for under-represented features $v_k$ with $k>1$ by four phases. An illustration of these four phases is provided in \Cref{fig:enter-label}.

\subsubsection{Phase I: Decrease of Dominant Feature.} \label{sssec: 4-1-p1}

Consider the query token featuring $v_k$ for some $k>1$. At $t=0$, $A_{k}^{(0)}=B_{k,n}^{(0)}=0$, and hence $\attn^{(0)}_i=\frac{1}{N}$ for $i\in[N]$ 
which implies that the transformer equally attends each input token. However, due to the imbalanced occurrence of features in $\esi$, the number of tokens featuring $v_1$ is much larger than others. Hence, $\Attn_{1}^{(0)}= \frac{|\cV_1|}{N}\geq \Omega(1)$ while $\Attn_{m}^{(0)}= \Theta\left(\frac{1}{K}\right)$ for $m>1$. Therefore, by \Cref{lem:gd}, we obtain
\begin{align*}
   \beta_{k,1}^{(0)}=&\mathbb{E}\left[\mathbf{1}\{x_{\text {query }}=v_k\}\Attn^{(0)}_{1}\right.\\
   & \left.\cdot \left(
 \sum_{m\not= k,1}{\Attn^{(0)}_{m}}^2-\Attn^{(0)}_{1}(1-\Attn^{(0)}_{1}) -\Attn^{(0)}_{k}(1-\Attn^{(0)}_k)\right)\right]\leq -\Omega\left(\frac{1}{K}\right),
\end{align*}
whereas $\alpha^{(0)}_{k}, |\beta^{(0)}_{k,n}|\approx \Theta(\frac{1}{K^2})$ for $n\not=k,1$. Therefore, $B_{k,1}^{(t)}$ enjoys a much larger decreasing rate initially. It can be shown that the decrease of $B_{k,1}^{(t)}$ will dominate 
for a certain time period that defines phase I. 
The following lemma
summarizes our main result in this phase.
\begin{lemma}[Informal]\label{lem:im:p1}
    Under the same conditions as \Cref{thm: unblc}, given $k>1$, there exists $T_{1,k}=O(\frac{\log(K)K^{1.98}}{\eta})$, such that
  for all  $0\leq  t\leq T_{1,k}$
          \begin{align*}
        \beta_{k,1}^{(t)}\leq -\Omega\left(\frac{1}{K^{1.98}}\right),\quad \alpha^{(t)}_k =\Theta\left(\frac{1}{K^2}\right), \quad  
 |\beta^{(t)}_{k,n}|\leq O\left(\frac{\alpha^{(t)}_{k}+|\beta_{k,1}^{(t)}|}{K}\right) \quad\text{  for all }n\not=k,1.
     \end{align*}
$ B_{k,1}^{(T_{1,k}+1)}\leq -0.49\log(K)$, while $A_{k}^{(T_{1,k}+1)}$ and $B^{(T_{1,k}+1)}_{k,n}$ for $n\not=k,1$ remain close to zero.  
\end{lemma}
During phase I, $B_{k,1}^{(t)}$ significantly decreases, leading to a reduction in $\Attn_{1}^{(t)}$, whereas other $\Attn_{n}^{(t)}$ with $n>1$ remain at the level of $\Theta\left(\frac{1}{K}\right)$.  By the end of this phase, $(\Attn_{1}^{(t)})^2$ drops to $O(\frac{1}{K^{0.98}})$, resulting in a decrease in $|\beta^{(t)}_{k,1}|$ as it approaches $\alpha_{k}^{(t)}$. Phase II then begins.


\subsubsection{Phase II: Switching of Leading Influence}\label{sssec: 4-1-p2}
Soon after entering this phase, 
the dominance role of $B_{k,1}^{(t)}$ diminishes as $|\beta^{(t)}_{k,1}|$ reaches the same order of magnitude as $\alpha_{k}^{(t)}$. The following result captures the shift of the leading influence, where the growth of $A_{k}^{(t)}$ takes dominance. 
\begin{lemma}[Informal]\label{lem:im:p2}
   Under the same conditions as \Cref{thm: unblc}, given $k>1$, there exists $T_{2,k}= T_{1,k}+O(\frac{\log(K)K^{2}}{\eta})$, such that at iteration $t=T_{2,k}+1$, we have 
   \begin{align*}
       A_{k}^{(T_{2,k}+1)}\geq 0.5\log(K),\quad B_{k,1}^{(T_{2,k}+1)}\in[-0.51\log(K), -0.49\log(K)]
   \end{align*}
  and $B^{(T_{2,k}+1)}_{k,n}$ for $n\not=k,1$ remain close to zero.
\end{lemma}
\Cref{lem:im:p2} shows that by the end of phase II, $A_{k}^{(t)}$ matches the magnitude of $B_{k,1}^{(t)}$,
and during phase II $B_{k,1}^{(t)}$ changes only slightly from the end of phase I. This suggests that, at certain moments in this phase, $A_{k}^{(t)}$ significantly increases and its growth becomes the dominant factor. We next provide some insights into the reasons behind this transition. Once $B_{k,1}^{(t)}$ decreases to $-0.5\log(K)$, we observe that $|\beta^{(t)}_{k,1}|\approx \alpha_{k}^{(t)}=\Theta(\frac{1}{K^2})$. After this point, it becomes challenging for $B_{k,1}^{(t)}$ to decrease significantly compared to the increase in $A_{k}^{(t)}$. 
To illustrate, let us suppose a minimal decrease of $B_{k,1}^{(t)}$ by an amount of $0.01\log(K)$. This would yield that $\Attn^{(t)}_1\leq O(\frac{1}{K^{0.501}})$ and $\beta_{k,1}^{(t)}\leq O(\frac{1}{K^{2.01}})$, while $\Attn^{(t)}_k\geq \Omega\left(\frac{1}{K}\right)$ and $\alpha_{k}^{(t)}\geq \Omega\left(\frac{1}{K^2}\right)$, establishing a situation where $\alpha_{k}^{(t)}\gg \beta_{k,1}^{(t)}$. Such a discrepancy leads to the switching of the dominant effect.


\subsubsection{Phase III: Growth of Target Feature }\label{sssec: 4-1-p3} After a transition phase, we observe that  $A_{k}^{(t)}$ enjoys a larger gradient 
  $\alpha^{(t)}_{k}\approx \Theta(\frac{1}{K^{1.5}})$ compared to  $|\beta^{(t)}_{k,1}|\leq O(\frac{1}{K^{1.98}})$
and $|\beta^{(t)}_{k,n}|\leq O\left(\frac{1}{K^3}\right)$ with $n\not=k,1$. This gap between $\alpha_{k}^{(t)}$ and $\beta_{k,n}^{(t)}$ remains over the period, and the gradient $\alpha_{k}^{(t)}$ continues to grow, driving the rapid growth of $A_{k}^{(t)}$ with $B_{k,n}^{(t)}$ being relatively unchanged. The following lemma
summarizes our main results in this phase.
\begin{lemma}[Informal]\label{lem:im:p3}
    Under the same conditions as \Cref{thm: unblc}, given $k>1$,  there exists $T_{3,k}=O(\frac{\log(K)K^{1.5}}{\eta})$, such that for all $T_{2,k}<t\leq T_{3,k}$
    \begin{align*}
\alpha_{k}^{(t)}\geq \Omega\left(\frac{1}{K^{1.5}}\right), \beta_{k,1}^{(t)}\in  \left[- O\left(\frac{\alpha^{(t)}_k}{K^{0.48}}\right),-\Omega\left(\frac{1}{K^{2.01}}\right)\right], 
 |\beta_{k,n}^{(t)}|\leq O\left(\frac{\alpha_{k}^{(t)}+|\beta_{k,1}^{(t)}|}{K}\right)\text{ with } n\not=k,1.
    \end{align*}
   At time $t=T_{3,k}+1$, we have $A_{k}^{(T_{3,k}+1)}\geq \log(K)$.
\end{lemma}
\Cref{lem:im:p3} follows because the continuous growth of $\alpha_{k}^{(t)}$ is mainly driven by $\Attn_{k}^{(t)}$, where $1-\Attn_{k}^{(t)}$ remains at the constant order.
However, as $A_{k}^{(t)}$ reaches $\log(K)$, $\Attn_{k}^{(t)}$ is above $\Omega(1)$, necessitating a more detailed analysis to control $\alpha_{k}^{(t)}$, which starts the final phase.


\subsubsection{Phase IV: Convergence}\label{sssec: 4-1-p4}

After learning the target feature $v_k$ at a certain level, 
the prediction error converges.
We characterize this in the following lemma, where we establish a connection between $\alpha^{(t)}_k$ and the prediction error via analyzing the change of $1-\Attn^{(t)}_{k}$ that diminishes during this phase.


\begin{lemma}[Informal]\label{lem:im:p4}
    Under the same conditions as \Cref{thm: unblc}, given $0<\epsilon<1$, for each $k>1$, there exists $T_{4,k}=T_{3,k}+O(\frac{K\log(K\epsilon^{-\frac{1}{2}})}{\eta\epsilon})$, such that  for all  $T_{3,k}< t\leq T_{4,k}$
  $$    \alpha^{(t)}_k\geq \Omega\left(\frac{\epsilon}{K}\right),\quad \beta^{(t)}_{k,n}\in \left[- O\left(\frac{\alpha^{(t)}_k}{K^{0.49}}\right),0\right],\quad \beta^{(t)}_{k,n}\in \left[- O\left(\frac{\alpha^{(t)}_k}{K}\right),0\right] \text{ 
 with } n\not=k,1. $$
    At time $t=T_{4,k}+1$, we have $\cL_{k}(\theta^{(T_{4,k}+1)})-\cL_{k}^{*}<\epsilon$ and $(1-\Attn_k^{(t)})^2 \leq O(\epsilon)$, if $\xq=v_k$ and $\pit\in\esi$.
\end{lemma}
The convergence result for $k>1$ stated in \Cref{thm: unblc} directly follows by choosing $T_{k}^*=T_{4,k}+1$. 
\subsection{Training Dynamics of Other Settings}\label{sec:other}
We next describe the training dynamics of other settings, which take the phases similar to those discussed in \Cref{sec:under}.

\paragraph{Imbalanced Setting for the Dominant Feature.}

For the dominant feature $v_1$ in the imbalanced setting, since the overall attention $\Attn_{1}^{(0)}$ to the target feature already reaches $\Omega(1)$ due to the abundance of tokens featuring $v_1$ in $\esi$, the training directly enters the convergence stage, as summarized in the following lemma.
\begin{lemma}[Informal]\label{lem:im:v1}
    Under the same conditions as \Cref{thm: unblc}, given $k>1$,  there exists $T_{1}=O(\frac{\log(\epsilon^{-\frac{1}{2}})}{\eta\epsilon})$, such that for all $t\leq T_{1}$
    \begin{align*}        \alpha_{1}^{(t)}\geq \Omega(\epsilon), \quad \beta_{1,n}^{(t)}\in [- O\left(\frac{\alpha^{(t)}_n}{K}\right),0] \text{ with } n>1.
    \end{align*} 
Further $\cL_{1}(\theta^{(T_{1}+1)})-\cL_{1}^{*}<\epsilon$,  and $(1-\Attn_1^{(T_{1}+1)})^2 \leq O(\epsilon)$ if $\xq=v_1$ and $\pit\in\esi$.
\end{lemma}
\paragraph{Balanced Scenarios.}
Similarly to imbalanced settings, we can show that for balanced data, with high probability, $\pit$ belongs to 
$$\esb:=\left\{\pit: |\cV_k|=\Theta\left(\frac{N}{K}\right)\text{ for all  }k\in[K]\right\}.$$
At initialization, the transformer uniformly assigns attention to each token, i.e., $\attn^{(0)}_i=\frac{1}{N}$ for $i\in[N]$. Unlike the imbalanced case, here, due to $\pit\in\esb$, we have that  $\Attn_{m}^{(0)}=\Theta\left(\frac{1}{K}\right)$ for $m\in[K]$, indicating nearly equal attention to each feature. Consequently, as \Cref{lem:gd}, we observe a significantly larger gradient in $A_{k}^{(t)}$ at the outset, with $\alpha^{(0)}_{k}\approx \Theta(\frac{1}{K^2})$, compared to $|\beta^{(0)}_{k,n}|\approx \Theta(\frac{1}{K^3})$ for $n\not=k$. This behavior mirrors the observations from phase III for under-represented features, allowing us to directly generalize the analysis.

\section{Additional Related Work}\label{App:A}
\paragraph{In-Context Learning.} 
Recent studies explored theoretical properties of transformers for in-context learning from various perspectives. Focusing on expressive capacity,  \cite{akyurek2022learning} studied linear regression tasks and showed that trained in-context learners can represent GD of ridge regression and exact least-squares regression. \cite{giannou2023looped} proved the existence of a looped transformer that can emulate in-context learning algorithms. \cite{von2023transformers,dai2023can} also showed that transformer trained in-context implements the GD. \cite{bai2023transformers} further provided comprehensive results of transformers including the expressive power, in-context prediction power, and sample complexity of pretraining, and then constructed two general mechanisms for algorithm selection.
 \cite{li2023transformers} analyzed the generalization error of trained in-context learning transformers. Another line of work considered in-context learning from a different perspective within the Bayesian framework~\cite{xie2021explanation,zhang2023and,wang2023large,jiang2023latent,han2023context,wies2023learnability,ahuja2023context}. 

Closely related to our work is the line of research by \cite{zhang2023trained,mahankali2023one,ahn2023transformers}, which investigated the training dynamics of in-context learning. Specifically, \cite{mahankali2023one} considered linear regression tasks and showed that the one-layer transformer that minimizes the pre-training loss implements one step of GD. 
\cite{zhang2023trained} investigated in-context learning of transformers with a single linear self-attention layer trained by gradient flow on linear regression tasks, and showed that gradient flow finds a global minimum. \cite{ahn2023transformers} investigated the landscape of the loss function for linear transformers trained over random instances of linear regression. 
However, all those works considered only transformers with {\em linear} self-attention layers and do not capture the crucial role of the {\em softmax} mapping, which lies in the core design of 
transformers to be advantageous over other network architectures.
Our work focuses on nonlinear transformers with {\em softmax attention} and characterizes their training dynamics for in-context learning.

\paragraph{Training Dynamics of Transformers.} 

 
 \cite{jelassi2022vision} proposed a simplified Vision Transformers (ViT) model in which the attention matrix solely depends on the positional embeddings and showed that the trained model by GD can learn spatial structure. \cite{li2023theoretical} studied the training of shallow ViT for a classification task and 
 characterized the sample complexity to achieve a desirable generalization performance. However, their analysis relied on a good initialization near the target pattern, which may not be feasible in practice.
 \cite{tian2023scan} analyzed the SGD training dynamics for a one-layer transformer with one self-attention plus one decoder layer and showed how the self-attention layer combines input tokens during the training, but this work did not provide the convergence guarantee for SGD. Recently, \cite{tarzanagh2023transformers} established an equivalence between the optimization geometry of self-attention and a hard-margin SVM problem that separates optimal input tokens from non-optimal tokens using linear constraints on the outer-products of token pairs. While the mathematical setup of these problems is different from in-context learning, some of our analysis techniques may be useful for studying the training dynamics of these problems.


\section{Conclusions}
In this work, we investigated the training dynamics of a one-layer transformer with softmax attention trained by GD for in-context learning.
We analyzed two settings respectively with balanced and imbalanced features, and proved the guaranteed convergence to a vanishing in-context prediction error by detailing the evolution of attention dynamics for both settings. Interestingly, we characterized a four-phase behavior for the imbalanced settings that sheds light on the intricate attention dynamics between dominant and target under-represented features during training. To our knowledge, this is the first work that rigorously analyzed the {\em softmax} attention dynamics for in-context learning. Our approach features novel ideas for phase decomposition based on the changes of the dominant role between two types of bilinear attention weights in the learning process, and has the potential to facilitate further theoretical understanding of how transformers perform in other algorithms and learning paradigms.

\bibliography{mybib}
\bibliographystyle{plain}
\newpage
\appendix

\appendixpage
\allowdisplaybreaks
\startcontents[section]
{
\hypersetup{linkcolor=blue}
\printcontents[section]{l}{1}{\setcounter{tocdepth}{2}}
}
\newpage
\section{Preliminaries}
In this section, we will introduce warm-up gradient computations and probabilistic lemmas that establish essential properties of the data and the loss function, which are pivotal for the technical proofs in the upcoming sections. Towards the conclusion of this section, we will also provide a summary of the key notations introduced in both the main content and these preliminary sections. These notations will be frequently adopted in our subsequent analyses.



\subsection{Gradient Computations}
We first calculate the gradient with respect to $Q$ (note that we do not update the parameter $\nu$ during the training). We omit the superscript `$(t)$' and write $L(\theta)$ as $L$ here for simplicity.
\begin{lemma}\label{app:lem:gdQ}
The gradient of the loss function with respect to $Q$ is given by
    \begin{align*}
  \nabla_{Q} L
&=\mathbb{E}\left[\left(\widehat{y}_{\text {query}}-\left\langle w, \xq \right\rangle\right)\sum_{i,j\in [N]}\attn_{ i} \attn_{j} (E^{x}_{i}-E^{x}_{j}){E^{x}_{N+1}}^{\top}y_i\right]. 
\end{align*}
\end{lemma}
\begin{proof}
We obtain:
    \begin{align}
    \nabla_{Q} L &= \mathbb{E}[\left(\yq-\left\langle w, \xq \right\rangle\right)\frac{\partial \yq }{\partial Q}]=\mathbb{E}\left[\left(\yq-\left\langle w, \xq \right\rangle\right)\sum_{i\in[N]}\frac{\partial \attn_{i} }{\partial Q}y_i\right].\label{eq1:gradient}
\end{align}
Denote $Q_{j,k}$ as the entry in $j$-th row and $k$-th column of $Q$, and define $f: \mathbb{R}^{d \times d} \to  \mathbb{R}^{N}$ as $f(Q)=\left(e^{{E_1^x}^\top Q{E_{N+1}^x}},\cdots,e^{{E_N^x}^\top Q{E_{N+1}^x}}\right)^\top$, and $g: \mathbb{R}^{N} \to \mathbb{R}$ as $ g(y)=\frac{y_i}{\sum_{j \in [N]}y_j}$. By the chain rule, we have
\begin{align*}
    \frac{\partial \attn_{ i} }{\partial Q_{j,k}}&= \mathrm{Tr}\left[(\frac{\partial g(y)}{\partial y} \big |_{y=f(Q)})^\top \frac{\partial f(Q)}{\partial Q_{j,k}}\right]\\
    & = \sum_{n\neq i}-\frac{e^{{E_i^x}^\top Q{E_{N+1}^x}}}{\left(\sum_{n\in[N]}e^{{E_n^x}^\top Q{E_{N+1}^x}}\right)^2} \cdot e^{{E_n^x}^\top Q{E_{N+1}^x}}(E_n^x)_j(E_{N+1}^x)_k\\
    & \quad + \frac{\sum_{n \in [N]}e^{{E_n^x}^\top Q{E_{N+1}^x}}-e^{{E_i^x}^\top Q{E_{N+1}^x}}}{\left(\sum_{n\in[N]}e^{{E_n^x}^\top Q{E_{N+1}^x}}\right)^2} \cdot e^{{E_i^x}^\top Q{E_{N+1}^x}}(E_i^x)_j(E_{N+1}^x)_k\\
    & = \attn_i \left((E_i^x)_j(E_{N+1}^x)_k-\sum_{n=1}^N\attn_n (E_n^x)_n=j(E_{N+1}^x)_k\right) \\
    & = \attn_i \left(\sum_{n=1}^N\attn_n\left((E_i^x)_j- (E_n^x)_j\right)(E_{N+1}^x)_k\right). 
\end{align*}
Then we reorganize these derivatives into a matrix, and have
\begin{align*}
    \frac{\partial \attn_{ i} }{\partial Q}=\attn_{i}\sum_{j\in[N]} \attn_{j}(E^{x}_{i}-E^{x}_{j}){E^{x}_{N+1}}^{\top}.
\end{align*}

Substituting the above equation into \cref{eq1:gradient}, we have 
\begin{align*}
  \nabla_{Q} L
&=\mathbb{E}\left[\left(\yq-\left\langle w_\tau, \xq\right\rangle\right)\sum_{i,j\in [N]}\attn_{ i} \attn_{j} (E^{x}_{i}-E^{x}_{j}){E^{x}_{N+1}}^{\top}y_i\right]. 
\end{align*}
\end{proof}
Recall that the quantities $A_{k}$ and $B_{k,n}$ are defined in \Cref{def:dynamics}. These quantities are associated with the attention weights for each token, and they play a crucial role in our analysis of learning dynamics. We will restate their definitions here for clarity.
\begin{definition}\label{app:def:dynamics}
  For $k,n\in[K]$ and $n\not=k$, define the following  quantities for $t\geq 0$:
\begin{align*}
A_{k}^{(t)}&:=v_{k}^{\top}Q^{(t)} v_{k}\qquad\alpha_{k}^{(t)}=-v_{k}^{\top}{\nabla_{Q} L(Q^{(t)})}v_{k} \\    B_{k,n}^{(t)}&:=v_{n}^{\top}Q^{(t)} v_{k}\qquad \beta_{k,n}^{(t)}=-v_{n}^{\top}\nabla_{Q} L(Q^{(t)})  v_{k}
\end{align*}
By GD update, we have
\begin{align*}
   A_{k}^{(t+1)}&:= A_{k}^{(t)}+\eta \alpha_{k}^{(t)}\\
   B_{k,n}^{(t+1)}&:= B_{k,n}^{(t)}+\eta \beta_{k,n}^{(t)}
\end{align*}
Moreover, by our initialization of $Q^{(0)}=\mathbf{0}_{d\times d}$,   we have $A_{k}^{(0)}=B_{k,n}^{(0)}=0$ for all $k,n\in[K]$ with $n\not=k$.

\end{definition}

Next, we apply the expression in \Cref{app:lem:gdQ} to compute the gradient projected onto the feature directions, i.e., $\alpha_{k}^{(t)}$ and $\beta_{k,n}^{(t)}$.
\begin{lemma}\label{app:lem:gd}
For $k,k^{\prime}\in[K]$, where $k\not=k^{\prime}$, we have
    \begin{align*}
&\alpha_{k}^{(t)}=\mathbb{E}\left[\mathbf{1}\{x_{\text {query }}=v_k\}\Attn^{(t)}_{k }\cdot \left(
 \sum_{m\not= k}{\Attn^{(t)}_{m}}^2+(1-\Attn^{(t)}_k)^2\right)\right]\\
 &\beta_{k,k^{\prime}}^{(t)}=\mathbb{E}\left[\mathbf{1}\{x_{\text {query }}=v_k\}\Attn^{(t)}_{k^{\prime}}\cdot \left(
 \sum_{m\not= k}{\Attn^{(t)}_{m}}^2-\Attn^{(t)}_{k^{\prime}} -\Attn^{(t)}_{k}(1-\Attn^{(t)}_k)\right)\right].
    \end{align*}
\end{lemma}

\begin{proof}
For any $k,k^{\prime}\in[K]$, apply the previous gradient expression in \Cref{app:lem:gdQ}, and note that only when $E_{N+1}^{x}=\xq=v_k$, we have ${E_{N+1}^{x}}^{\top}v_k\not=0$. Thus, we obtain  
\begin{align*}
&  v_{k^{\prime}}^{\top}\nabla_{Q} L  v_k
\\
&=\mathbb{E}\left[\mathbf{1}\{\xq=v_k\}\left(\yq-\left\langle w, \xq \right\rangle\right)\sum_{i,j\in[N]}\attn_{i} \attn_{ j}{y}_i v_{k^{\prime}}^{\top}(E^{x}_{i}-E^{x}_{j})\right]\\
&=\mathbb{E}\left[\mathbf{1}\{\xq=v_k\}\left(\yq-\left\langle w, 
\xq \right\rangle\right)\sum_{m,n\in[K]}\sum_{i\in \cV_{m}}\sum_{j\in\cV_n}\attn_{i} \attn_{ j}{y}_i v_{k^{\prime}}^{\top}(v_m-v_n)\right]\\
&=\mathbb{E}\left[\mathbf{1}\{\xq=v_k\}\left(\yq-\left\langle w, \xq\right\rangle\right)\sum_{n\in[K]}\sum_{i\in \cV_{k^{\prime}}}\sum_{j\in\cV_n}\attn_{i} \attn_{ j}{y}_i v_{k^{\prime}}^{\top}(v_{k^{\prime}}-v_n)\right]\\
&\quad +\mathbb{E}\left[\mathbf{1}\{\xq=v_k\}\left(\yq-\left\langle w, \xq\right\rangle\right)\sum_{m\in[K]}\sum_{i\in \cV_{m}}\sum_{j\in\cV_{k^{\prime}}}\attn_{i} \attn_{ j}{y}_i v_{k^{\prime}}^{\top}(v_m-v_{k^{\prime}})\right]\\
&=\mathbb{E}\left[\mathbf{1}\{\xq=v_k\}\left(\yq-\left\langle w, \xq \right\rangle\right)\Attn_{k^{\prime}}\langle w,v_{k^{\prime}}\rangle\sum_{n\in[K]} \Attn_{n} \right]\\
&\quad-\mathbb{E}\left[\mathbf{1}\{\xq=v_k\}\left(\yq-\left\langle w, \xq \right\rangle\right)\Attn_{k^{\prime}}\sum_{m\in[K]} \Attn_{m} \langle w,v_{m}\rangle\right]\\
&=\mathbb{E}\left[\mathbf{1}\{\xq=v_k\}\left(\yq-\left\langle w, \xq\right\rangle\right)\Attn_{k^{\prime}}\sum_{m\in[K]} \Attn_{m} \langle w,v_{k^{\prime}} -v_{m}\rangle\right].
\end{align*}
 Note that $$\yq
    =\sum_{i\in[N]} \attn_{ i} {y}_{i}=\sum_{m\in[K]} \Attn_{m}\langle w,v_m\rangle.$$ Thus when $\xq=v_k$, we have $$\yq-\left\langle w, \xq\right\rangle=-\sum_{m\in[K] } \Attn_{m} \langle w,v_{k} -v_{m}\rangle .$$
Substituting this into the above equation, we have
\begin{align*}
v_{k^{\prime}}^{\top}&\nabla_{Q}L v_{k}\\
&=-\mathbb{E}\left[\mathbf{1}\{\xq =v_k\}\Attn_{k^{\prime}}\left(\sum_{n\in[K] } \Attn_{n} \langle w,v_{k} -v_{n}\rangle\right)\left(\sum_{m\in[K]} \Attn_{m} \langle w,v_{k^{\prime}} -v_{m}\rangle\right)\right]\\
&=-\mathbb{E}\left[\mathbf{1}\{\xq=v_k\}\Attn_{k^{\prime}}\left(\sum_{n\in[K] }\sum_{m\in[K]} \Attn_{m} \Attn_{n} \langle w,v_{k} -v_{n}\rangle \langle w,v_{k^{\prime}} -v_{m}\rangle\right)\right]\\
&=-\mathbb{E}\left[\mathbf{1}\{\xq=v_k\}\Attn_{k^{\prime}}\left(\sum_{n\in[K]}\sum_{m\in[K]} \Attn_{m} \Attn_{n} (v_{k} -v_{n})^{\top} w w^{\top}(v_{k^{\prime}} -v_{m})\right)\right]\\
&=-\mathbb{E}\left[\mathbf{1}\{\xq=v_k\}\Attn_{k^{\prime}}\cdot\right.\\
&\qquad\left.\left(\sum_{n\in[K]}\sum_{m\in[K]} \Attn_{m} \Attn_{n} (v_{k} -v_{n})^{\top} \mathbb{E}[w w^{\top}\mid \pit\cup \{\xq\}](v_{k^{\prime}} -v_{m})\right)\right]\\
&=-\mathbb{E}\left[\mathbf{1}\{\xq=v_k\}\Attn_{k^{\prime}}\left(\sum_{n\in[K]}\sum_{m\in[K]} \Attn_{m} \Attn_{n} (v_{k} -v_{n})^{\top} (v_{k^{\prime}} -v_{m})\right)\right]\\
&=-\mathbb{E}\left[\mathbf{1}\{\xq=v_k\}\Attn_{k^{\prime}}\left( (v_{k} -\sum_{n\in[K]} \Attn_{n}v_{n})^{\top} (v_{k^{\prime}} -\sum_{m\in[K]}\Attn_{m} v_{m})\right)\right].
\end{align*} 
When $k^{\prime}=k$, we obtain

\begin{align*}
 \alpha_{k}= -v_{k}^{\top}\nabla_{Q}Lv_{k}
&=\mathbb{E}\left[\mathbf{1}\{\xq=v_k\}\Attn_{k}\|v_{k} -\sum_{n} \Attn_{n}v_{n}\|^2\right]\\
&=\mathbb{E}\left[\mathbf{1}\{\xq=v_k\}\Attn_{k}\left((1-\Attn_{k})^2+\sum_{m\not= k } \Attn_{m}^2\right)\right].
\end{align*} 
When $k^{\prime}\not=k$, we have
\begin{align*}
\beta_{k,k^{\prime}}&=-v_{k^{\prime}}^{\top}\nabla_{Q}L v_{k}\\
&=\mathbb{E}\left[\mathbf{1}\{\xq=v_k\}\Attn_{k^{\prime}}\left(\sum_{m\not=k,k^{\prime}} \Attn_{m}^2-\Attn_{k}(1-\Attn_k)-\Attn_{k^{\prime}}(1-\Attn_{k^{\prime}}) \right)\right]\\
&=\mathbb{E}\left[\mathbf{1}\{\xq=v_k\}\Attn_{k^{\prime}}\left(\sum_{m\not=k} \Attn_{m}^2-\Attn_{k}(1-\Attn_k)-\Attn_{k^{\prime}}\right)\right].
\end{align*}
\end{proof}

\subsection{Useful Probabilistic Lemmas for Prompt}
Recall that given a prompt $P = (x_1, y_1, \ldots, x_N, y_N, x_{\text{query}})$, we denote $\pit$ as the collection of input tokens, i.e., $\{x_i\}_{i=1}^{N}$. It is worth noting that, based on our data distribution, the occurrence count of the $k$-th feature in the first $N$ input tokens from $\pit$, denoted as $|\cV_k|$, follows a multinomial distribution. Leveraging the concentration property inherent to multinomial distributions, we can identify a high-probability event to which $\pit$ belongs. This event constitutes the crux of our subsequent analysis.
 
We first introduce the following tail bound for multinomial distributions.
\begin{lemma}[Tail Bound of Multinomial Distribution~\cite{devroye1983equivalence}]\label{app:lem:tail}
 Let $\left(X_1, \cdots, X_K\right)$ be a multinomial $\left(N, p_1, \cdots, p_K\right)$ random vector. For all $\varepsilon \in(0,1)$ and all $K$ satisfying $ K/ N \leq \varepsilon^2 / 20$, we have
$$
P\left(\sum_{i=1}^K\left|X_i-\mathbb{E}\left(X_i\right)\right|>N \varepsilon\right) \leq 3 \exp \left(-N \varepsilon^2 / 25\right).
$$
\end{lemma}

Now we present our characterization of a high-probability event for $\pit$.
\begin{lemma}[High-probability Event for Balanced
 Data]\label{app:lem:prob-b}
  Suppose that $p_k=\Theta\left(\frac{1}{K}\right)$ for any $k\in[K]$ and  $K^3\ll N$. For some constant  $\cb\geqslant \sqrt{\frac{20 K^3}{N}} $, define
$$
\esb:=\left\{\pit: |\cV_k|\in \left[p_k N-\frac{\cb N}{K}, p_k N+\frac{\cb N}{K}\right]\text{ for }k\in[K]\right\}.
$$Then 
, we have
\begin{align*}
    \mathbb{P}(\pit\in\esb )\geq 1-3 \exp \left(-\frac{\cb^2 N}{25 K^2}\right).
\end{align*}
Let us denote $\Lba_k=p_k K-\cb$ and $\Uba_k= p_k K+\cb$. Note that $\Lba_k, \Uba_k$ are at the order of  the constant level since $p_k=\Theta\left(\frac{1}{K}\right)$. Then for any $\pit$ belonging to $\esb$, $|\cV_k|\in[\frac{\Lba_kN}{K},\frac{\Uba_kN}{K}]=\Theta(\frac{N}{K})$. Note that  we can properly choose $\cb$ to guarantee $\Lba_k> 0$ for $k\in[K]$.
\end{lemma}
\begin{proof}
    Denote $\left|\cV_k\right|=X_k$. Then $\left(X_1, \cdots, X_K\right)\sim$ multinomial $\left(N, p_1, \cdots, p_K\right)$.  Noting that $\frac{\cb^2}{20K^2}\geq \frac{K}{N}$ by our choice of $\cb$, and then letting $\epsilon=\frac{\cb}{K}$, we have $\epsilon^2/20\geq \frac{K}{N}$. By multinomial tail bound in \Cref{app:lem:tail}, we obtain
$$
P\left(\sum_{i=1}^K\left|X_i-\mathbb{E}\left(X_i\right)\right|>\cb \frac{N}{K}\right) \leq 3 \exp \left(-\frac{\cb^2 N}{25 K^2}\right).
$$
Then, since $\mathbb{E}\left(X_i\right)=p_i N$, we have 
\begin{align*}
  P\left(\cap_{i=1}^K\left\{\left|X_i-p_iN\right|>\frac{\cb N}{K}\right\}\right)& \leq P\left(\sum_{i=1}^K\left|X_i-\mathbb{E}\left(X_i\right)\right|>\cb \frac{N}{K}\right) \\
    &\leq 3 \exp \left(-\frac{\cb^2 N}{25 K^2}\right). 
\end{align*}
\end{proof}

\begin{lemma}[High-probability Event for Imbalanced Data]\label{app:lem:prob-im}
    Suppose that $p_1=\Theta(1)$, $p_k=\Theta\left(\frac{1}{K}\right)$ for $2\leq k\leq K$, and $K^3\ll N$. Then for some constant  $\ci\geqslant \sqrt{\frac{20 K^3}{N}}$,  there exist constants $\Ui_{k}>\Li_{k}>0$ for any $k\in[K]$, such that letting
    \begin{align*}
        \esi:=\left\{\pit: |\cV_1|\in [\Li_{1}N, \Ui_1 N] \text{ and } |\cV_k|\in \left[\frac{\Li_{k}N}{K}, \frac{\Ui_{k}N}{K}\right]\text{ for }2\leq k\leq K\right\},
    \end{align*}
    we have
    \begin{align*}
        \mathbb{P}(\pit\in\esi)\geq 1-3 \exp \left(-\frac{\ci^2 N}{25 K^2}\right).
    \end{align*}

\end{lemma}

\begin{proof}
    Similarly to the proof for \Cref{app:lem:prob-b}, we have 
    \begin{align*}
P\left(\cap_{i=1}^K\left\{\left|X_i-p_iN\right|>\frac{\ci  N}{K}\right\}\right) \leq 3 \exp \left(-\frac{\ci^2 N}{25 K^2}\right).
    \end{align*}
 For $k>1$, let us denote $\Li_k=p_kK-\ci$ and $\Ui_k= p_kK+\ci$. Since $p_k=\Theta\left(\frac{1}{K}\right)$, we can easily conclude that $\Li_k, \Ui_k$ for $k>1$ are constant level.  Furthermore, for $k=1$, let $\Li_1=p_1-0.01\ci$ and $\Ui_1=p_1+0.01\ci$. Since $p_1$ is at the order of the $\Theta(1)$,   we have $$\left[p_1N-\frac{\ci  N}{K},p_1N+\frac{\ci  N}{K}\right]=\left[(p_1- \frac{\ci}{K})N, (p_1+\frac{\ci}{K})p_1N\right]\subset\left[\Li_1N,\Ui_1N\right]$$ for sufficiently large $K$. 
\end{proof}

\subsection{Properties of Loss Function and Prediction Error}
Recall the population loss we consider is given by:
\begin{align}
    L(\theta)=\frac{1}{2} \mathbb{E}
    \left[\left(\widehat{y}_{\text { query }}-\left\langle w, x_{\text { query }}\right\rangle\right)^2\right]. \label{app:eq:obj}
\end{align}
In this part, we will present several important lemmas for such a training objective. We first introduce the following lemma, which connects the loss form with the attention score when the query token takes a certain feature.
\begin{lemma}[Loss Calculation]\label{app:lem:loss1}
    The population loss $ L(\theta)$ can be decomposed into the following form:
        \begin{align*}
       L(\theta)
&=\frac{1}{2}\sum_{k=1}^{K}
\mathbb{E}\left[\mathbf{1}\{\xq=v_{k}\}\left(
 \sum_{m\not= k}\Attn^2_{m}+(1-\Attn_k)^2\right)\right].
 \end{align*}
\end{lemma}
\begin{proof}
    Following the calculations similar to those in \Cref{app:lem:gd}, we have 
    \begin{align*}
     L(\theta)&=\frac{1}{2}\sum_{k=1}^{K}
\mathbb{E}\left[\mathbf{1}\{\xq=v_{k}\}\left(\yq-\left\langle w, \xq\right\rangle\right)^2\right]\\
&=\frac{1}{2}\sum_{k=1}^{K} \mathbb{E}\left[\mathbf{1}\{\xq=v_k\}\left(\sum_{n\in[K] } \Attn_{n} \langle w,v_{k} -v_{n}\rangle\right)\left(\sum_{m\in[K] } \Attn_{m} \langle w,v_{k} -v_{m}\rangle\right)\right]\\
&=\frac{1}{2}\sum_{k=1}^{K}\mathbb{E}\left[\mathbf{1}\{\xq=v_k\}\|v_{k} -\sum_{n\in[K]} \Attn_{n}v_{n}\|^2\right]\\
&=\frac{1}{2}\sum_{k=1}^{K}\mathbb{E}\left[\mathbf{1}\{\xq=v_k\}\left((1-\Attn_{k})^2+\sum_{m\not= k } \Attn_{m}^2\right)\right].
    \end{align*}
\end{proof}

\subsubsection{Loss Characterization for the Balanced Case}
We first introduce some additional crucial notations for the loss objectives.
\paragraph{Notations for the balanced case.}
\begin{align}
L^{*}&=\min_{\theta}L(\theta)=\min_{\theta}\frac{1}{2} \mathbb{E}
    \left[\left(\widehat{y}_{\text { query }}-\left\langle w, x_{\text { query }}\right\rangle\right)^2\right],\label{app:eq:infb}\\
    \Lol&=\frac{1}{2}\left(1+\frac{1}{K-1}\right)\sum_{k=1}^{K}
\mathbb{P}\left(\xq=v_{k}\cap|\cV_{k}|=0\right).\label{app:eq:lowb}
\end{align}
$L^{*}$ denotes the minimum value of the population loss in \cref{app:eq:obj} by minimizing over $\theta$ in the form of $\{1, Q\}$, and $\Lol$ represents the sum of unavoidable errors for each $k\in[K]$, given that the query token is the $k$-th feature but has not been seen in the first $N$ training samples. We will show that  $\Lol$ serves as a lower bound for $L^{*}$, and demonstrate that the network trained with GD will attain nearly zero error compared to $\Lol$. Our convergence will be established by the suboptimality gap with respect to $\Lol$, which necessarily implies the convergence to $L^{*}$. (It also implies $L^{*}-\Lol$ is small.)
We further introduce the following quantities to facilitate our analysis of the loss function.
\begin{align*}
  & L(\theta) = \sum_{k=1}^{K}  L_{k}(\theta),\\
   \text{ where }
 &L_{k}(\theta)=
\frac{1}{2}\mathbb{E}\left[\mathbf{1}\{\xq=v_{k}\}\left(\yq-\left\langle w, \xq\right\rangle\right)^2\right].\\
&\Lol_k=\frac{1}{2}\left(1+\frac{1}{K-1}\right)
\mathbb{P}\left(\xq=v_{k}\cap|\cV_{k}|=0\right),\\
    &\tilde{L}_{k}(\theta)=
\frac{1}{2}\mathbb{E}\left[\mathbf{1}\{\xq=v_{k}\cap \pit\in\esb\}\left(\yq-\left\langle w, \xq\right\rangle\right)^2\right].
\end{align*}
\begin{lemma}\label{app:lem:optb1}
    For $L^{*}$ and $\Lol$ defined in \cref{app:eq:infb} and \cref{app:eq:lowb}, respectively,  we have $\Lol\leq L^{*}$ and they are both at the order of $\Theta(e^{-\operatorname{poly}(K)})$ for the balanced data.
\end{lemma}
\begin{proof}
   We first prove $\Lol\leq L^{*}$:
\begin{align*}
    L^{*}&=\min_{\theta}\frac{1}{2}\sum_{k=1}^{K}\mathbb{E}\left[\mathbf{1}\{\xq=v_{k}\}\left(\yq-\left\langle w, \xq\right\rangle\right)^2\right]\\
    &\geq \min_{\theta}\frac{1}{2}\sum_{k=1}^{K}\mathbb{E}\left[\mathbf{1}\{\xq=v_{k}\cap |\cV_{k}|=0\}\left(\yq-\left\langle w, \xq\right\rangle\right)^2\right]\\
&=\min_{\theta}\frac{1}{2}\sum_{k=1}^{K}
\mathbb{E}\left[\mathbf{1}\{\xq=v_{k}\cap|\cV_{k}|=0\}\left(
 \sum_{m\not= k}\Attn^2_{m}+(1-\Attn_k)^2\right)\right]
 \end{align*}
 Notice that when the query token is the $k$-th feature but has not been seen in the first $N$ training samples, $\Attn_k=0$. Moreover, $\sum\limits_{m\not= k}\Attn^2_{m}\geq \frac{1}{K-1}$
   by Cauchy–Schwarz inequality. Thus
 \begin{align*}
      L^{*}\geq \frac{1}{2}\left(1+\frac{1}{K-1}\right)\sum_{k=1}^{K}
\mathbb{E}\left[\mathbf{1}\{\xq=v_{k}\cap|\cV_{k}|=0\}\right]=\Lol.
 \end{align*}
 Furthermore, since $\xq$ and $\pit$ are independently sampled, 
    \begin{align*}
        \Lol
&={K}\cdot \Theta\left(\frac{1}{K}\right)\cdot \left(1-\Theta\left(\frac{1}{K}\right)\right)^{N}=\Theta\left(e^{-\poly(K)}\right).
 \end{align*}
where the last equality follows because $N\gg K^3$, and hence  $(1-\Theta\left(\frac{1}{K}\right))^{N}=\Theta\left(e^{-\poly(K)}\right)$.

We next only need to show $L^{*}=O(e^{-\poly(K)})$. We have 
\begin{align*}
    L^{*}=&\min_{\theta}\left(\frac{1}{2}\sum_{k=1}^{K}\mathbb{E}\left[\mathbf{1}\{\xq=v_{k}\cap |\cV_{k}|>0\}\left(
 \sum_{m\not= k}\Attn^2_{m}+(1-\Attn_k)^2\right)\right]\right.\\
    &+\left.\frac{1}{2}\sum_{k=1}^{K}\mathbb{E}\left[\mathbf{1}\{\xq=v_{k}\cap |\cV_{k}|=0\}\left(
 \sum_{m\not= k}\Attn^2_{m}+1\right)\right]\right)
\end{align*}
Consider $Q=\sigma \mathbf{I}_{d\times d}$. If $\xq=v_{k}\cap |\cV_{k}|>0$ holds,  we have 
\begin{align*}
    \sum_{m\not= k}&\Attn^2_{m}+(1-\Attn_k)^2\\
   &\leq (1-\Attn_k)\max_{m\not=k }\Attn_{m}+(1-\Attn_k)^2\\
    &\leq 2(1-\Attn_k)^2=2\left(\frac{N-|\cV_k|}{N-|\cV_k|+|V_k|e^{\sigma}}\right)^2
\leq 2 \left(\frac{N}{N+e^{\sigma}}\right)^2
\end{align*}
Taking $\sigma=\poly(N)$, then we have 
\begin{align*}
    L^{*}\leq O(e^{-\poly(N)})+O(e^{-\poly(K)})=O(e^{-\poly(K)}).
\end{align*}
\end{proof}
\begin{lemma}\label{app:lem:optb2}
For the balanced data, given $k\in[K]$,  for any $\theta$,  we have
    \begin{align*}
     \tilde{L}_{k}(\theta)\leq L_{k}(\theta)-\Lol_k \leq   \tilde{L}_{k}(\theta)+3 p_k\exp \left(-\frac{\cb^2 N}{25 K^2}\right).
    \end{align*}
\end{lemma}
\begin{proof} We proceed the derivation as follows.
    \begin{align*}
        L_{k}(\theta)- \tilde{L}_{k}(\theta)&=\frac{1}{2}\mathbb{E}\left[\mathbf{1}\{\xq=v_{k}\cap \pit\in{\esb}^c\}\left(\yq-\left\langle w, \xq\right\rangle\right)^2\right]\\
        &=\frac{1}{2}\mathbb{E}\left[\mathbf{1}\{\xq=v_{k}\cap \pit\in{\esb}^c\}\left(
            \sum_{m\not= k}\Attn^2_{m}+(1-\Attn_k)^2\right)\right]\\
            &\stackrel{(a)}{\leq} \frac{1}{2}\cdot 2\mathbb{P}\left(\xq=v_{k}\cap \pit\in{\esb}^c\right)\\
            &\stackrel{(b)}{\leq} p_k\cdot 3 \exp \left(-\frac{\cb^2 N}{25 K^2}\right)\\&=3p_k\exp \left(-\frac{\cb^2 N}{25 K^2}\right).
    \end{align*}  
    where $(a)$ follows from the fact that  $$\sum_{m\not= k}\Attn^2_{m}+(1-\Attn_k)^2\leq  (1-\Attn_k)\max_{m\not= k}\Attn_{m}+(1-\Attn_k)^2\leq 2,$$
    and $(b)$ holds by \Cref{app:lem:prob-b}.

    On the other hand,     \begin{align*}
        L_{k}(\theta)- \tilde{L}_{k}(\theta)&\geq \frac{1}{2}\mathbb{E}\left[\mathbf{1}\{\xq=v_{k}\cap |\cV_k|=0\}\left(
            \sum_{m\not= k}\Attn^2_{m}+(1-\Attn_k)^2\right)\right]\\
            &\geq  \frac{1}{2}\frac{K}{K-1}\mathbb{E}\left[\mathbf{1}\{\xq=v_{k}\cap |\cV_k|=0\}\right]=\Lol_{k}.
    \end{align*}  
\end{proof}
Consequently, for each $k\in[K]$, $\tilde{L}_{k}(\theta)$ closely tracks the deviation between ${L}_{k}(\theta)$ and $\Lol_{k}$, which is what we will primarily focus on bounding in the subsequent analysis.
\subsubsection{Loss Characterization for the Imbalanced Case}
\paragraph{Notations for the imbalanced case.}In the imbalanced case, we are interested in the prediction error for the query corresponding to each given feature $k\in [K]$. Thus we consider the following conditional prediction error for each $k\in[K]$:
\begin{align}
    \cL_{k}(\theta) &= \frac{1}{2}\mathbb{E}\left[\left(\yq-\left\langle w, \xq\right\rangle\right)^2\bigg|\xq=v_k\right].\label{app:eq:obj-i}
\end{align}
Similarly, we define the minimum and the unavoidable values for such conditional prediction error:
    \begin{align}
        \cL^{*}_{k}&=
\min_{\theta}\frac{1}{2}\mathbb{E}\left[\left(\yq-\left\langle w, \xq\right\rangle\right)^2\bigg|\xq=v_{k} \right],\label{app:eq:infi}
\\
     \Loi_{k}&=
\frac{1}{2}\left(1+\frac{1}{K-1}\right)\mathbb{P}\left(|\cV_{k}|=0\right),\label{app:eq:lowi}
 \end{align}
\begin{align*}
&\tilde{\cL}_{k}(\theta)=\frac{1}{2}
\mathbb{E}\left[\mathbf{1}\{\pit\in\esi\}\left(\yq-\left\langle w, \xq\right\rangle\right)^2\bigg|\xq=v_{k}\right].
\end{align*}

\begin{lemma}\label{app:lem:opti1}
Given $k\in[K]$, for $\cL_k^{*}$ and $\Loi_k$ defined in \cref{app:eq:infi} and \cref{app:eq:lowi}, respectively, we have $\Loi_k\leq \cL_k^{*}$ and they are both at the order of $\Theta(e^{-\operatorname{poly}(K)})$ for the imbalanced data.
\end{lemma}
\begin{proof}
 The analysis is similar as \Cref{app:lem:optb1}, we only  show $\Loi_k=\Theta(e^{-\poly(K)})$.
    \begin{align*}
        \Loi_{k}&=
\frac{1}{2}\left(1+\frac{1}{K-1}\right)
\mathbb{P}\left(|\cV_{k}|=0\right)\\
&=\Theta(1) (1-p_k)^{N}.
 \end{align*}
 For $k=1$, $(1-p_1)^{N}=\Theta(\exp(-N))=\Theta\left(e^{-\poly(K)}\right)$. 
For $k>1$, since $N\gg K^3$, then $(1-p_k)^{N}=(1-\Theta\left(\frac{1}{K}\right))^{N}=\Theta\left(e^{-\poly(K)}\right)$, which completes the proof.
\end{proof}
\begin{lemma}\label{app:lem:opti2}
   For the imbalanced data, given $k\in[K]$, for any $\theta$, we have
    \begin{align*}
       \tilde{\cL}_{k}(\theta) \leq \cL_{k}(\theta)-\Loi_{k} \leq   \tilde{\cL}_{k}(\theta)+3 \exp \left(-\frac{\ci^2 N}{25 K^2}\right).
    \end{align*}
\end{lemma}
\begin{proof}
The proof of the first inequality is similar to that for \Cref{app:lem:optb2}. We next show the second inequality.
    \begin{align*}
        \cL_{k}(\theta)- \Loi_{k}&\leq   \tilde{\cL}_{k}(\theta)+\frac{1}{2}\mathbb{E}\left[\mathbf{1}\{ \pit\in{\esi}^c\}\left(\yq-\left\langle w, \xq\right\rangle\right)^2\mid \xq=v_{k}\right]\\
        &=\tilde{\cL}_{k}(\theta)+\frac{1}{2}\mathbb{E}\left[\mathbf{1}\{\pit\in{\esi}^c\}\left(
            \sum_{m\not= k}\Attn^2_{m}+(1-\Attn_k)^2\right)\mid \xq=v_{k}\right]\\
            &\leq \tilde{\cL}_{k}(\theta)+ \mathbb{P}\left(\pit\in{\esi}^c\right)\\
            &\leq \tilde{\cL}_{k}(\theta)+  3 \exp \left(-\frac{\ci^2 N}{25 K^2}\right).
    \end{align*}  
\end{proof}
\subsection{Notations and Parameters}
In \Cref{tab:my_table}, we summarize the notations introduced throughout the main content and in the preliminary section. Throughout all the proofs in our paper, we consider $N=\poly(K)\gg K^3$,  and $K$ is sufficiently large.


\begin{table}
    \centering
    \caption{Summary of Notations}
    \begin{tabularx}{\textwidth}{L|R}
        \toprule
        \rowcolor{gray}
        \textbf{Notations} & \textbf{Descriptions} \\
        \midrule
        $\attn_{i}^{(t)}$, $\Attn_{k}^{(t)}$ & The attention scores for the $i$-token and $k$-th feature, where $i\in[N]$ and $k\in [K]$. \\
         \midrule
        $A_{k}^{(t)}$, $B_{k,n}^{(t)}$& The bilinear attention  weights when  $\xq=v_k$: $A_{k}^{(t)}=e^{v_k^{\top}Q^{(t)}v_k}$, $B_{k,n}^{(t)}=e^{v_n^{\top}Q^{(t)}v_k}$ for $n\not=k$.\\\midrule
        $\alpha_{k}^{(t)}$, $\beta_{k,n}^{(t)}$& 
        The gradient updates respectively for $A_{k}^{(t)}$ and $B_{k,n}^{(t)}$.\\\midrule
        $\pit$ & The input tokens in the prompt, i.e., $\{x_i\}_{i=1}^{N}$.
        \\\midrule
        $\esb$, $\esi$ & The high-probability events that $\pit$ belongs to respectively for the balanced and imbalanced data.
        \\\midrule
        $L^{*}$, $\Lol$& The minimum value and lower bound on the population loss $L(\theta)$ (\ref{app:eq:obj}).
        \\\midrule $L_{k}(\theta)$, $\tilde{L}_{k}(\theta)$, $\Lol_k$ & The loss functions on the event $\{\xq=v_k\}$, $\{\xq=v_k\}\cap\{\pit\in\esb\}$, and the lower bound on $L_k$.   \\\midrule
           $\cL_k^{*}$, $\Loi_k$ ({\bf Imbalanced})& The minimum value and lower bound of prediction error conditioned on $\xq=v_k$, i.e., $\cL_k(\theta)$ (\ref{app:eq:obj-i}).
           \\\midrule
          $\tilde{\cL}_k(\theta)$ ({\bf  Imbalanced})& The conditional prediction error on the event $\{\pit\in\esi\}$.
        \\
        \bottomrule
    \end{tabularx}
    \label{tab:my_table}
\end{table}

\clearpage

\section{Analysis for the Balanced Case}

In this section, we present the analysis for the balanced case, we first discuss the outline of our proof.

\subsection{Roadmap of the Proof}
We will analyze the convergence of the training process via two phases of dynamics. 
At the beginning of each phase, we will establish an induction hypothesis, which we expect to remain valid throughout that phase. Subsequently, we will analyze the dynamics under such a hypothesis within the phase, aiming to provide proof of the hypothesis by the end of the phase.

The main idea of the proof lies in analyzing the GD dynamics of $A_{k}^{(t)}$ and $B_{k,n}^{(t)}$. From \Cref{app:def:dynamics,app:lem:gd}, we have
\begin{align*}
   A_{k}^{(t+1)}&= A_{k}^{(t)}+\eta \alpha_{k}^{(t)},\\
   B_{k,n}^{(t+1)}&= B_{k,n}^{(t)}+\eta \beta_{k,n}^{(t)},
\end{align*}
where
    \begin{align*}
&\alpha_{k}^{(t)}=\mathbb{E}\left[\mathbf{1}\{\xq=v_k\}\Attn^{(t)}_{k }\cdot \left(
 \sum_{m\not= k}{\Attn^{(t)}_{m}}^2+(1-\Attn^{(t)}_k)^2\right)\right],\\
 &\beta_{k,n}^{(t)}=\mathbb{E}\left[\mathbf{1}\{\xq=v_k\}\Attn^{(t)}_{n}\cdot \left(
 \sum_{m\not= k}{\Attn^{(t)}_{m}}^2-\Attn^{(t)}_{n} -\Attn^{(t)}_{k}(1-\Attn^{(t)}_k)\right)\right].
    \end{align*}
We divide the learning process of any feature $k$ in the balanced case into the following two phases.
\begin{itemize}
    \item \textbf{Phase I} ($t\in[0,T_{1,k}]$, \Cref{app:bal:p1}): At initialization, $A_k^{(t)}$ keeps growing at a rate at least of $\frac{\eta}{K^2}$, while $B_{k,n}^{(t)}$ oscillates with a smaller rate of $\frac{\eta}{K^3}$. Therefore, the increase in $A_k^{(t)}$ will dominate the learning dynamics during phase I.
    \item \textbf{Phase II} ($t\in(T_{1,k},T_{2,k}^\epsilon]$, \Cref{app:bal:p2-s1,app:bal:p2-s2}): After rapid growth of self-attention module parameters in phase I, the query token featuring $v_k$ is aligned with these input tokens also featuring $v_k$ effectively and disregards other features. Then the process proceeds to the convergence phase, where $A_k^{(t)}$ monotonically increases and $B_{k,n}^{(t)}$ monotonically decreases, which finally contributes to the convergence of the loss. Based on the variation rates of $A_k^{(t)}$ and $B_{k,n}^{(t)}$, the convergence phase further has two sub-stages as follows. 
    \begin{itemize}
        \item \textbf{Stage I} 
         ($t\in(T_{1,k},\tilde{T}_{2,k}^\epsilon]$, \Cref{app:bal:p2-s1}): the increase of $A_k^{(t)}$ is as fast as $\Omega(\frac{\epsilon}{K})$ while the decrease of $B_{k,n}^{(t)}$ is slow, and the gap $A
_k^{(t)}-\max_{m \neq k}B_{k,m}^{(t)}$ stays within $O(\log(\frac{K}{\epsilon^{\frac{1}{2}}}))$. 
         \item \textbf{Stage II} 
         ($t\in(\tilde{T}_{2,k}^\epsilon,T_{2,k}^\epsilon]$, \Cref{app:bal:p2-s2}): the increase of $A_k^{(t)}$ and the decrease of $B_{k,n}^{(t)}$ both are relatively steady and the attention nearly focuses on the target feature, leading to the convergence of the loss. 
    \end{itemize}
\end{itemize}
We finally combine all results in the above two phases to prove the convergence of the training process given in \Cref{thm: bal K} (\Cref{app:bal:proof}).

\subsection{Phase I: Growth of Target Feature}\label{app:bal:p1}
In this section, we shall study the initial phase of learning the relationship between the query token and its corresponding feature. Firstly, we present the induction hypothesis in this phase.   For the $k$-th feature $v_k$, we define the \textbf{Phase I} as all iterations $0 \leq t\leq T_{1,k}$, where
$$
T_{1,k} \triangleq \max \left\{t: A_{k}^{(t)} \leq \log(K)\right\}.
$$
We state the following induction hypothesis, which will hold throughout Phase I. This hypothesis is ultimately proved in \Cref{app:b:p1}.
\begin{hypothesis}\label{hp1}
For each $0 \leq t \leq T_{1,k}$, the following holds:
\begin{enumerate}[label={\alph*}.]
    \item $A_{k}^{(t)}$ is monotonically increasing and $A_{k}^{(t)}\in [0,\log(K)]$;
    \item $|B_{k,n}^{(t)}|=O(\frac{A_{k}^{(t)}}{K})$ for any $n\not=k$.
\end{enumerate}
\end{hypothesis} 
 
\subsubsection{Technical Lemmas}
We first introduce several useful technical lemmas. 
\begin{lemma}\label{lem1t1}
    Suppose \Cref{hp1} holds at iteration $0\leq t\leq T_{k,1}$. 
    If $\xq=v_k$ and $\pit\in\esb$, the following holds
    \begin{enumerate}
    \item  $\Attn^{(t)}_k=\Omega\left(\frac{1}{K}\right)$;
        \item $1-\Attn^{(t)}_{k}\geq \Omega(1)$.
    \end{enumerate}
\end{lemma}
\begin{proof}
    Since $\xq=v_k$, then we have
    \begin{align*}
\Attn^{(t)}_{k}&=\frac{|\cV_k|e^{{v_k}^{\top}Q^{(t)}v_{k}}}{\sum_{j\in [N]}e^{{E^{x}_{j}}^{\top}Q^{(t)}v_{k}}}\\
&=\frac{|\cV_k|\exp(A^{(t)}_k)}{\sum_{m\not=k}|\cV_m|\exp(B^{(t)}_{k,m})+|\cV_k|\exp(A^{(t)}_k)}\\
        &=\frac{1}{\sum_{m\not=k}\frac{|\cV_m|}{|\cV_k|}\exp(B^{(t)}_{k,m}-A^{(t)}_k)+1}.
    \end{align*}
    By \Cref{hp1}, $$e^{-\left(\log(K)+O(\frac{\log(K)}{K})\right)}\leq \exp(B^{(t)}_{k,m}-A^{(t)}_k)\leq e^{O(\frac{\log(K)}{K})}.$$ Thus
     \begin{align*}
\Attn^{(t)}_{k} &\geq\frac{1}{e^{O(\frac{\log(K)}{K})}(\frac{N}{|\cV_k|}-1)+1}\geq \frac{1}{e^{O(\frac{\log(K)}{K})}(K/\Lba_k-1)+1}=\Omega\left(\frac{1}{K}\right),
    \end{align*}
    where the second inequality follows because $\pit \in \esb$.

    On the other hand, 
    \begin{align*}
\Attn^{(t)}_{k} &\leq\frac{1}{e^{-\left(\log(K)+O(\frac{\log(K)}{K})\right)}(\frac{N}{|\cV_k|}-1)+1}\leq \frac{1}{e^{-1}(\frac{1}{\Uba_k}-\frac{1}{K})+1}.
    \end{align*}
    Considering $\Uba = \Theta (1)$, we have 
    \begin{align*}
       1- \Attn^{(t)}_{k} \geq \frac{(\frac{1}{\Uba_k}-\frac{1}{K})}{(\frac{1}{\Uba_k}-\frac{1}{K})+e}\geq \Omega(1).
            \end{align*}
\end{proof}
\begin{lemma}\label{lem1t2}
     Suppose \Cref{hp1} holds at iteration $0 \leq t\leq T_{1,k}$. If $\xq=v_k$ and $\pit \in \esb$,  for $n\not=k$, the following holds
        $$\Attn^{(t)}_n=\Theta\left(\frac{1-\Attn^{(t)}_{k}}{K}\right)=\Theta\left(\frac{1}{K}\right).$$
\end{lemma}
\begin{proof}
    To show the first equality, since $\xq=v_k$,  we have
    \begin{align*}
\Attn^{(t)}_{n}&=\frac{|\cV_n|e^{{v_n}^{\top}Q^{(t)}v_{k}}}{\sum_{j\in [N]}e^{{E^{x}_{j}}^{\top}Q^{(t)}v_{k}}}\\&=\frac{|\cV_n|\exp(B^{(t)}_{k,n})}{\sum_{m\not=k}|\cV_m|\exp(B^{(t)}_{k,m})+|\cV_k|\exp(A^{(t)}_k)}.
    \end{align*}
    By \Cref{hp1}, $e^{-O(\frac{\log(K)}{K})}\leq \exp(B^{(t)}_{k,m}-B^{(t)}_{k,n})\leq e^{O(\frac{\log(K)}{K})}$. Combining with the fact that $\frac{|\cV_m|}{|\cV_n|}=\Theta(1)$ when $\pit \in \esb$, 
    we have
    \begin{align*}
       \frac{\Attn^{(t)}_n}{1-\Attn^{(t)}_{k}} =\frac{|\cV_n|\exp(B^{(t)}_{k,n})}{\sum_{m\not=k}|\cV_m|\exp(B^{(t)}_{k,m})}= \frac{1}{\sum_{m\not=k}\frac{|\cV_m|}{|\cV_n|}\exp(B^{(t)}_{k,m}-B^{(t)}_{k,n})}=\Theta\left(\frac{1}{K}\right).
    \end{align*}
    Combining with the \Cref{lem1t1}, we immediately have $\Attn^{(t)}_n=\Theta\left(\frac{1}{K}\right) $.
\end{proof}
\subsubsection{Controlling Gradient Updates in Phase I}
\begin{lemma}\label{p1a}
    Given any fixed $k \in [K]$, if \Cref{hp1} holds at iteration $0 \leq t\leq T_{1,k}$, then $\alpha_k^{(t)}\geq0$ and satisfies 
    \begin{align*}
        \alpha_{k}^{(t)}\geq \Omega\left(\frac{1}{K^2}\right).
    \end{align*}
\end{lemma}
\begin{proof}
    By the gradient expression in \Cref{app:lem:gd},
    \begin{align}
  \alpha_{k}^{(t)}&=
\mathbb{E}\left[\mathbf{1}\{\xq=v_k\}\Attn^{(t)}_{k }\cdot \left(
 \sum_{m\not= k}{\Attn_{m}^{(t)}}^2+(1-\Attn^{(t)}_k)^2\right)\right]\nonumber\\
&=\mathbb{E}\left[\mathbf{1}\{\xq=v_k\cap\pit \in \esb\}\Attn^{(t)}_{k }\cdot \left(
 \sum_{m\not= k}{\Attn_{m}^{(t)}}^2+(1-\Attn^{(t)}_k)^2\right)\right]\nonumber\\
 &\quad +\mathbb{E}\left[\mathbf{1}\{\xq=v_k\cap{\pit \in \esb}^c\}\Attn^{(t)}_{k }\cdot \left(
 \sum_{m\not= k}{\Attn_{m}^{(t)}}^2+(1-\Attn^{(t)}_k)^2\right)\right]\nonumber\\
 &\overset{(a)}{\geq} p_k\cdot\mathbb{P}(\pit \in \esb) \nonumber\\
 & \quad  \times\mathbb{E}\left[\Attn^{(t)}_{k }\cdot \left(
 \sum_{m\not= k}{\Attn_{m}^{(t)}}^2+(1-\Attn^{(t)}_k)^2\right)\bigg|\{\xq=v_k\} \cap \{\pit \in \esb\} \right]\nonumber\\
&\geq p_k\cdot\mathbb{P}(\pit \in \esb)\times \mathbb{E}\left[\Attn^{(t)}_{k }\cdot (1-\Attn^{(t)}_k)^2\bigg|\{\xq=v_k\} \cap \{\pit \in \esb\} \right] \label{app: eq: bal-alpha}\\
&\overset{(b)}{\geq} \Omega\left(\frac{1}{K^2}\right),\nonumber
\end{align}
where $(a)$ follows from the fact that $\xq$ is independent with $\pit$ and the second term is non-negative, $(b)$ follows from \Cref{app:lem:prob-b},  \Cref{lem1t1} and the fact that $p_k = \Theta\left(\frac{1}{K}\right)$ in the balanced case and $N \gg K^3$.
\end{proof}
\begin{lemma}\label{p1b}
    Given any fixed $k \in [K]$, if \Cref{hp1} holds at iteration $0 \leq t\leq T_{1,k}$, then for any $n\not=k$, $\beta_{k,n}^{(t)}$ satisfies 
    \begin{align*}
        |\beta_{k,n}^{(t)}|\leq O\left(\frac{\alpha^{(t)}_k}{K}\right).
    \end{align*}
\end{lemma}
\begin{proof}
    By the gradient expression in \Cref{app:lem:gd}, we have 
    \begin{align}
        \beta_{k,n}^{(t)}&\leq \mathbb{E}\left[\mathbf{1}\{\xq=v_k\}\Attn^{(t)}_{n}\cdot \left(
 \sum_{m\not= k}{\Attn^{(t)}_{m}}^2\right)\right],\label{I1}\\
 -\beta_{k,n}^{(t)}&\leq \mathbb{E}\left[\mathbf{1}\{\xq=v_k\}\Attn^{(t)}_{n}\cdot \left(
 \Attn^{(t)}_{n} +\Attn^{(t)}_{k}(1-\Attn^{(t)}_k)\right)\right].\label{I2}
    \end{align}
    For \cref{I1},  we further derive
    \begin{align}
      \beta_{k,n}^{(t)}&\leq \mathbb{E}\left[\mathbf{1}\{\xq=v_k\cap \pit \in \esb\}\Attn^{(t)}_{n}\cdot \left(
        \sum_{m\not= k}{\Attn^{(t)}_{m}}^2\right)\right]\nonumber\\
        &\quad +\mathbb{E}\left[\mathbf{1}\{\xq=v_k\cap {\pit \in \esb}^{c}\}\Attn^{(t)}_{n}\cdot \left(
        \sum_{m\not= k}{\Attn^{(t)}_{m}}^2\right)\right]\nonumber\\
        &\stackrel{(a)}{\leq}p_k\cdot\mathbb{P}(\pit \in \esb)\cdot\mathbb{E}\left[\Attn^{(t)}_{n}\cdot \left(
 \max_{m\not= k}{\Attn^{(t)}_{m}}\right)\bigg|\{\xq=v_k\} \cap \{\pit \in \esb\} \right]\nonumber\\
  & \quad +p_k\cdot\mathbb{P}(\pit \in {\esb}^{c})\nonumber\\
 &\stackrel{(b)}{\leq} p_k\mathbb{E}\left[\Attn^{(t)}_{n}\cdot \left(
 \max_{m\not= k}\Attn^{(t)}_{m}\right)\bigg|\{\xq=v_k\} \cap \{\pit \in \esb\} \right]+3 p_k\exp \left(-\frac{\cb^2 N}{25 K^2}\right)\nonumber\\
 &\overset{(c)}{\leq} O\left(\frac{1}{K^3}\right), \label{app:eq:bal-beta-1}
    \end{align}
    where $(a)$ follows from the fact that $\xq$ is independent with $\pit$, $\Attn^{(t)}_{n}\leq 1$ and $ \sum_{m\not= k}{\Attn^{(t)}_{m}}^2\leq \max_{m\not= k}\Attn^{(t)}_{m}\cdot \sum_{m\not= k}{\Attn^{(t)}_{m}}\leq \max_{m\not= k}\Attn^{(t)}_{m}$, $(b)$ follows from \Cref{app:lem:prob-b}, and $(c)$ follows from \Cref{lem1t2} and the fact that $p_k = \Theta\left(\frac{1}{K}\right)$ and $N \gg K^3$.

    For \cref{I2}, similarly to the derivation above, we have 
    \begin{align}
      &-\beta_{k,n}^{(t)}\nonumber\\
      & \leq p_k\mathbb{E}\left[\Attn^{(t)}_{n}\cdot \left(
 \Attn^{(t)}_{n} +\Attn^{(t)}_{k}(1-\Attn^{(t)}_k)\right)\bigg|\{\xq=v_k\} \cap \esb \right] +2p_k\cdot\mathbb{P}(\pit \in {\esb}^{c})\nonumber\\
 &\stackrel{(a)}{=}2p_k\cdot\mathbb{P}({\pit \in \esb}^{c})+ p_k\cdot\mathbb{P}({\pit \in \esb}) \times\nonumber \\
 & \quad \mathbb{E}\left[\Theta(\frac{1-\Attn^{(t)}_{k}}{K})\cdot \left( \Theta(\frac{1-\Attn^{(t)}_{k}}{K})+
\Attn^{(t)}_{k}(1-\Attn^{(t)}_k)\right)\bigg|\{\xq=v_k\} \cap \esb\right]\nonumber\\
 &\stackrel{(b)}{\leq}p_k\cdot\mathbb{P}({\pit \in \esb})\mathbb{E}\left[O(\frac{\Attn^{(t)}_{k}(1-\Attn^{(t)}_{k})^2}{K})\bigg|\{\xq=v_k\} \cap  \esb\right] +6p_k\exp \left(-\frac{\cb^2 N}{25 K^2}\right)\nonumber\\
 &\overset{(c)}{\leq} O\left(\frac{\alpha_{k}^{(t)}}{K}+\frac{1}{K}\exp \left(-\frac{\cb^2 N}{25 K^2}\right)\right) \label{app:eq:bal-beta-2}
    \end{align}
    where $(a)$ follows from  \Cref{lem1t2} and $(b)$ follows from \Cref{lem1t1} and \Cref{app:lem:prob-b}, and $(c)$ follows from \Cref{app: eq: bal-alpha}.
    
    From \Cref{p1a} and the choice of $N \gg K^3$, we have 
    \begin{align}
       \alpha_{k}^{(t)}&\geq \Omega\left(\frac{1}{K^2}\right)\gg 6\exp \left(-\frac{\cb^2 N}{25 K^2}\right). \label{app:eq:bal-beta-3}
    \end{align}
    Thus, combining \cref{app:eq:bal-beta-1,app:eq:bal-beta-2,app:eq:bal-beta-3}, we have
    $$
|\beta_{n,k}^{(t)}|\leq \max\left\{O(\frac{\alpha_{k}^{(t)}}{K}),O\left(\frac{1}{K^3}\right)\right\}=O\left(\frac{\alpha_{k}^{(t)}}{K}\right).
    $$
\end{proof}
\subsubsection{End of Phase I}\label{app:b:p1}
\begin{lemma}
    Given any fixed $k\in[K]$, \Cref{hp1} holds for all iterations  $0 \leq t\leq T_{1,k}$, where $T_{1,k}$ is at most $O(\frac{\log(K)K^2}{\eta})$, and at iteration $t=T_{1,k}+1$, we have 
    \begin{enumerate}[label={\alph*}.]
    \item $A_{k}^{(T_{1,k}+1)}\geq \log(K)$;
    \item $\Attn_{k}^{(T_{1,k}+1)}=\Omega(1)$ if $\xq=v_k$ and $ \pit \in \esb$.
    \end{enumerate}
\end{lemma}
\begin{proof}
If \Cref{hp1} holds, the  existence of $T_{1,k}=O(\frac{\log(K)K^2}{\eta})$ directly follows from  \Cref{p1a}.

We next prove \Cref{hp1}. It is easy to verify \Cref{hp1} holds at $t=0$.  Now we suppose \Cref{hp1} holds for all iterations $\leq t-1$, and prove it holds at $t$. 

By \Cref{p1a}, we have $\alpha_{k}^{(t-1)}\geq 0$. Thus $A_{k}^{(t)}= A_{k}^{(t-1)}+\eta\alpha_{k}^{(t-1)}\geq 0$. Moreover, by the definition of $T_{1,k}$, we immediately obtain $A_{k}^{(t)}\leq \log(K)$.

By \Cref{p1b}, we have $|\beta_{k,n}^{(t-1)}|\leq O\left(\frac{\alpha_{k}^{(t-1)}}{K}\right)$. Thus,
\begin{align*}
   |B_{k,n}^{(t)}|&\leq|B_{k,n}^{(t-1)}|+\eta O\left(\frac{\alpha_{k}^{(t-1)}}{K}\right)\\
   &\leq O\left(\frac{A_{k}^{(t-1)}}{K}\right)+\eta O\left(\frac{\alpha_{k}^{(t-1)}}{K}\right)\\
   &\leq O\left(\frac{A_{k}^{(t)}}{K}\right).
\end{align*}

The first statement follows the definition of $T_{1,k}$. 
Moreover, $\Attn_{k}^{(T_{1,k}+1)}=\Omega(1)$ can be derived from \Cref{lem2t1} in the subsequent section.

\end{proof} 

\subsection{Phase II: Convergence: Stage I}\label{app:bal:p2-s1}
After rapid growth of self-attention module parameters in phase I, the query token featuring $v_k$ is aligned with these input tokens also featuring $v_k$ effectively and disregards other features. Then the process proceeds to the convergence phase, where $A_k^{(t)}$ monotonically increases and $B_{k,n}^{(t)}$ monotonically decreases, which finally contributes to the convergence of the loss. Based on the variation rates of $A_k^{(t)}$ and $B_{k,n}^{(t)}$, the convergence phase further has two sub-stages as follows. 

Given any  $0<\epsilon <1$, 
    for $k\in[K]$, 
    define 
  \begin{align*}
    \tilde{T}^{\epsilon}_{2,k}:= \max\left\{t>T_{1,k}: A_{k}^{(t)}-\max_{m\not=k} B_{k,m}^{(t)}\leq \log\left(\left(\frac{K}{\Lba_k}-1\right)\left(\left(\frac{3}{\epsilon}\right)^{\frac{1}{2}}-1\right)\right) \right\}.
  \end{align*}
\begin{hypothesis}\label{hp2}
    For $T_{1,k}<t\leq \tilde{T}^{\epsilon}_{2,k}$, suppose $\operatorname{polylog}(K)\gg \log(\frac{1}{\epsilon})$, and the following holds
\begin{enumerate}[label={\alph*}.]
    \item $A_{k}^{(t)}$ is monotonically increasing and $A_{k}^{(t)}\in [\log(K), O(\log(K/\epsilon))]$;
    \item $B_{k,n}^{(t)}$ is monotonically decreasing and $|B_{k,n}^{(t)}|=O(\frac{A_{k}^{(t)}}{K})$ for any $n\not=k$.
\end{enumerate}
\end{hypothesis}
\subsubsection{Technical Lemmas}
We first introduce several useful technical lemmas.
\begin{lemma}\label{lem2t1}
    Suppose \Cref{hp2} holds at iteration $T_{1,k}<t\leq \tilde{T}^{\epsilon}_{2,k}$. 
    If $\xq=v_k$ and $\pit \in \esb$, the following holds
    \begin{enumerate}
    \item  $\Attn^{(t)}_k=\Omega(1)$;
        \item $(1-\Attn^{(t)}_{k})^2\geq \Omega(\epsilon)=\Omega(
            \exp\left(-\operatorname{polylog}(K)\right))$.
    \end{enumerate}
\end{lemma}
\begin{proof}
    Since $\xq=v_k$, we have
    \begin{align*}
\Attn^{(t)}_{k}&=\frac{|\cV_k|\exp(A^{(t)}_k)}{\sum_{m\not=k}|\cV_m|\exp(B^{(t)}_{k,m})+|\cV_k|\exp(A^{(t)}_k)}\\
        &=\frac{1}{\sum_{m\not=k}\frac{|\cV_m|}{|\cV_k|}\exp(B^{(t)}_{k,m}-A^{(t)}_k)+1}.
    \end{align*}
    By \Cref{hp2}, we obtain $$\exp(B^{(t)}_{k,m}-A^{(t)}_k)\leq e^{O(\frac{\log(K/\epsilon)}{K})-\log(K)}\leq e^{O(\frac{\log(K)+\operatorname{polylog}(K)}{K})-\log(K)}\leq O\left(\frac{1}{K}\right).$$ 
Therefore,
    \begin{align*}
\Attn^{(t)}_{k} &\geq\frac{1}{O\left(\frac{1}{K}\right)(\frac{N}{|\cV_k|}-1)+1}\geq \frac{1}{O(\frac{1}{\Lba_k}-\frac{1}{K})+1}\geq \Omega(1).
    \end{align*}
   On the other hand, by the definition of $\tilde{T}^{\epsilon}_{2,k}$, we have
   \begin{align*}
   1-\Attn^{(t)}_{k}
            &=\frac{\sum_{m\not=k}\frac{|\cV_m|}{|\cV_k|}\exp(B^{(t)}_{k,m}-A^{(t)}_k)}{\sum_{m\not=k}\frac{|\cV_m|}{|\cV_k|}\exp(B^{(t)}_{k,m}-A^{(t)}_k)+1}\\
            &{\geq} \frac{ \exp( \min_{m\not= k} B^{(t)}_{k,m}-A^{(t)}_k) (\frac{N}{|\cV_k|}-1)}{\exp( \min_{m\not= k} B^{(t)}_{k,m}-A^{(t)}_k) (\frac{N}{|\cV_k|}-1)+1}\\
            &\geq \frac{ \exp( \min_{m\not= k} B^{(t)}_{k,m}-A^{(t)}_k) (\frac{K}{\Uba_k}-1)}{\exp( \min_{m\not= k} B^{(t)}_{k,m}-A^{(t)}_k) (\frac{K}{\Uba_k}-1)+1}\\
            &= \frac{ \exp( \max_{m\not= k} B^{(t)}_{k,m}-A^{(t)}_k-\Delta B_{k}^{(t)}) (\frac{K}{\Uba_{k}}-1)}{\exp( \max_{m\not= k} B^{(t)}_{k,m}-A^{(t)}_k-\Delta B_{k}^{(t)}) (\frac{K}{\Uba_{k}}-1)+1}\\
            &\geq \frac{(\frac{K}{\Lba_k}-1)^{-1}(\epsilon^{-\frac{1}{2}}-1)^{-1}\cdot e^{-O(\frac{\operatorname{polylog}(K)}{K})}(\frac{K}{\Uba_{k}}-1)}{(\frac{K}{\Lba_k}-1)^{-1}(\epsilon^{-\frac{1}{2}}-1)^{-1}e^{-O(\frac{\operatorname{polylog}(K)}{K})}(\frac{K}{\Uba_{k}}-1)+1}\\
            &\geq \Omega(\epsilon^{\frac{1}{2}}),
        \end{align*}
        where $\Delta B_{k}^{(t)}=\max_{m\not= k}B^{(t)}_{k,m}-\min_{m\not= k}B^{(t)}_{k,m}=O(\frac{A_{k}^{(t)}}{K})$, and the first and second inequalities follow from the fact that $\frac{x}{1+x}$ monotonically increases w.r.t. $x \geq 0$, and the third inequality follows from the definition of $\tilde{T}^{\epsilon}_{2,k}$ and \Cref{hp2}. 
\end{proof}
\begin{lemma}\label{lem2t2}
     Suppose \Cref{hp2} holds at iteration $ T_{1,k}< t\leq \tilde{T}^{\epsilon}_{2,k}$. If $\xq=v_k$ and $\pit \in\esb$,  for $n\not=k$, then the following holds
     $$\Attn_n^{(t)}=\Theta\left(\frac{1-\Attn_{k}^{(t)}}{K}\right).$$
\end{lemma}
\begin{proof}
By definition,
    \begin{align*}
\Attn^{(t)}_{n}&=\frac{|\cV_n|\exp(B^{(t)}_{k,n})}{\sum_{m\not=k}|\cV_m|\exp(B^{(t)}_{k,m})+|\cV_k|\exp(A^{(t)}_k)}.
    \end{align*}
   By \Cref{hp2}, we have $$e^{-O(\frac{\log(K)-\log(\epsilon)}{K})}\leq \exp(B^{(t)}_{k,m}-B^{(t)}_{k,n})\leq e^{O(\frac{\log(K)-\log(\epsilon)}{K})}.$$ 
   Further combining with the fact that $-\log(\epsilon)\ll \operatorname{polylog}(K)$, 
    we have
    \begin{align*}
       \frac{\Attn_n^{(t)}}{1-\Attn^{(t)}_{k}} =\frac{|\cV_n|\exp(B^{(t)}_{k,n})}{\sum_{m\not=k}|\cV_m|\exp(B^{(t)}_{k,m})}= \frac{1}{\sum_{m\not=k}\frac{|\cV_m|}{|\cV_n|}\exp(B^{(t)}_{k,m}-B^{(t)}_{k,n})}=\Theta\left(\frac{1}{K}\right).
    \end{align*}
\end{proof}
\subsubsection{Controlling Gradient Updates in Stage I of Phase II}
\begin{lemma}\label{p2a}
    At each iteration $T_{1,k}<t \leq \tilde{T}^{\epsilon}_{2,k}$, if \Cref{hp2} holds, then $\alpha_k^{(t)}\geq0$ and satisfies 
    \begin{align*}
        \alpha_{k}^{(t)}\geq \Omega\left(\frac{\epsilon}{K}\right).
    \end{align*}
\end{lemma}
\begin{proof}
    The analysis is similar to that for  \Cref{p1a}, but we need to be more careful about the lower bound of $1-\Attn^{(t)}_{k}$.
    By gradient expression in \Cref{app:lem:gd}, we obtain
    \begin{align*}
         \alpha_{k}^{(t)}& =
        \mathbb{E}\left[\mathbf{1}\{\xq=v_k\}\Attn^{(t)}_{k }\cdot \left(
         \sum_{m\not= k}{\Attn_{m}^{(t)}}^2+(1-\Attn^{(t)}_k)^2\right)\right]\\
         &\geq p_k\cdot\mathbb{P}(\pit\in\esb)\mathbb{E}\left[\Attn^{(t)}_{k }\cdot \left(
         \sum_{m\not= k}{\Attn_{m}^{(t)}}^2+(1-\Attn^{(t)}_k)^2\right)\mid\{\xq=v_k\} \cap \esb\right]\\
        &\geq p_k\cdot\mathbb{P}(\pit\in\esb)\mathbb{E}\left[\Attn^{(t)}_{k }\cdot (1-\Attn^{(t)}_k)^2\mid\{\xq=v_k\} \cap \esb\right]\\
        &\geq \Omega(\frac{\epsilon}{K}),
        \end{align*}
where the last inequality follows from \Cref{lem2t1,app:lem:prob-b} and the fact that $p_k = \Theta\left(\frac{1}{K}\right)$ in the balanced case.
\end{proof}
\begin{lemma}\label{p2b}
    At each iteration $ T_{1,k}<t\leq \tilde{T}^{\epsilon}_{2,k}$, if \Cref{hp2} holds,  then given $k\in[K]$, for any $n\not=k$, $\beta_{k,n}^{(t)}$ satisfies 
    \begin{align*}
       - O\left(\frac{\alpha^{(t)}_k}{K}\right)\leq  \beta_{k,n}^{(t)}\leq 0.
    \end{align*}
\end{lemma}
\begin{proof}
        Note that  conditioned on the event $\{\xq=v_k\} \cap \{\pit \in \esb\}$, by \Cref{lem2t1,lem2t2}, we have $\Attn_k^{(t)}=\Omega(1)$, $\max_{m\not=k}\Attn_{m}=O\left(\frac{1}{K}\right)$, and thus 
\begin{align} \sum_{m\not= k}{\Attn^{(t)}}^2_{m}-\Attn^{(t)}_{n} -\Attn^{(t)}_{k}(1-\Attn^{(t)}_k)&\leq  \max_{m\not=k}\Attn^{(t)}_{m}\sum_{m\not= k}\Attn^{(t)}_{m}-\Attn^{(t)}_{k}(1-\Attn^{(t)}_k)\nonumber\\
&=-(1-\Attn^{(t)}_k)(\Attn^{(t)}_{k}-\max_{m\not=k}\Attn^{(t)}_{m})\nonumber\\
&\leq -\Omega (1-\Attn^{(t)}_k). \label{eq3}
\end{align}
Therefore, by combining with \Cref{app:lem:gd}, we obtain
\begin{align*}
    \beta_{k,n}^{(t)}&\leq \mathbb{E}\left[\mathbf{1}\{\xq=v_k\cap \esb\}\Attn^{(t)}_{n}\cdot \left(
      \sum_{m\not= k}{\Attn^{(t)}_{m}}^2-\Attn^{(t)}_{n} -\Attn^{(t)}_{k}(1-\Attn^{(t)}_k)\right)\right]\\
      &\quad +\mathbb{E}\left[\mathbf{1}\{\xq=v_k\cap {\esb}^{c}\}\Attn^{(t)}_{n}\cdot \left(
      \sum_{m\not= k}{\Attn^{(t)}_{m}}^2\right)\right]\\
      &\stackrel{(a)}{\leq}p_k\cdot\mathbb{P}(\pit\in\esb)\cdot\mathbb{E}\left[-\Omega(\frac{(1-\Attn^{(t)}_k)^2}{K})\mid\{\xq=v_k\} \cap \esb\right]+p_k\cdot\mathbb{P}({\esb}^{c})\\
&\stackrel{(b)}{\leq} p_k\cdot \left(-\Omega(\frac{\epsilon}{K})\right)+3 p_k\exp \left(-\frac{\cb^2 N}{25 K^2}\right)\\
&{\leq} 0,
  \end{align*}
  where $(a)$ follows from \cref{eq3} and \Cref{lem2t2}, $(b)$ follows from \Cref{lem2t1,app:lem:prob-b}, and the last inequality holds since
\begin{align*}
\frac{\epsilon}{K}\gg \frac{\exp(-\operatorname{polylog}(K))}{K}\gg \exp \left(-\frac{\cb^2 N}{25 K^2}\right).
\end{align*}
Moreover, following the analysis similar to that for \Cref{p1b}, we have 
    \begin{align*}
     - \beta_{k,n}^{(t)}&\leq p_k\mathbb{E}\left[\Attn^{(t)}_{n}\cdot \left(
 \Attn^{(t)}_{n} +\Attn^{(t)}_{k}(1-\Attn^{(t)}_k)\right)\bigg|\{\xq=v_k\} \cap \esb\right]+p_k\mathbb{P}({\esb}^{c})\\
 &\leq p_k\mathbb{E}\left[\Theta(\frac{1-\Attn^{(t)}_{k}}{K})\cdot O\left(
\Attn^{(t)}_{k}(1-\Attn^{(t)}_k)\right)\bigg|\{\xq=v_k\} \cap \esb\right]\\
 & \qquad \qquad +6 p_k\exp \left(-\frac{\cb^2 N}{25 K^2}\right)\\
 &=p_k\mathbb{E}\left[O(\frac{\Attn^{(t)}_{k}(1-\Attn^{(t)}_{k})^2}{K})\bigg|\{\xq=v_k\} \cap \esb\right]+6 p_k\exp \left(-\frac{\cb^2 N}{25 K^2}\right)\\
 &\leq O(\frac{\alpha_{k}^{(t)}}{K}).
    \end{align*}
\end{proof}
\subsubsection{End of Stage I of Phase II}
\begin{lemma}\label{end2}
    Given $k\in[K]$, and $0<\epsilon<1$, suppose $\operatorname{polylog}(K)\gg \log(\frac{1}{\epsilon})$. Then \Cref{hp2} holds for at least all  $T_{1,k}<t\leq \tilde{T}^{\epsilon}_{2,k}=T_{1,k}+O\left(\frac{K\log(K\epsilon^{-\frac{1}{2}})}{\eta\epsilon}\right)$, and at iteration $t=\tilde{T}^{\epsilon}_{2,k}+1$, we have  $A_{k}^{(\tilde{T}^{\epsilon}_{2,k}+1)}\geq \Omega\left(\log(\frac{K}{\epsilon})\right)$.
\end{lemma}
\begin{proof}
    We first prove the existence of $\tilde{T}_{2,k}^{\epsilon}$. 
    Recall that
    \begin{align*}
        \tilde{T}^{\epsilon}_{2,k}:= \max\left\{t>T_{1,k}: A_{k}^{(t)}-\max_{m\not=k} B_{k,m}^{(t)}\leq \log\left(\left(\frac{K}{\Lba_k}-1\right)\left((\frac{3}{\epsilon})^{\frac{1}{2}}-1\right)\right) \right\}.
      \end{align*}
    
      When $t\in(T_{1,k},\tilde{T}^{\epsilon}_{2,k}]$, consider 

      \begin{align*}
          &\left(A_{k}^{(t+1)}-\max_{m\not=k} B_{k,m}^{(t+1)}\right)-\left(A_{k}^{(t)}-\max_{m\not=k} B_{k,m}^{(t)}\right)\\
          & \qquad \geq  
\eta(1-O\left(\frac{1}{K}\right))\alpha_k^{(t)}=\Omega\left(\frac{\eta\epsilon}{K}\right),
      \end{align*}
    where the inequality follows from \Cref{p2b} and the last equation follows from \Cref{p2a}. Therefore,  at most $$\tilde{T}^{\epsilon}_{2,k}-T_{1,k}=O(\frac{K\log\left((\frac{K}{\Lba_k}-1)((\frac{3}{\epsilon})^{\frac{1}{2}}-1)\right)}{\eta\epsilon})=O(\frac{K\log(K\epsilon^{-\frac{1}{2}})}{\eta\epsilon})$$ iterations are needed before $A_{k}^{(t)}-\max_{m\not=k} B_{k,m}^{(t)}$ exceeds $\log\left(\left(\frac{K}{\Lba_k}-1\right)\left((\frac{3}{\epsilon})^{\frac{1}{2}}-1\right)\right)$.
    
      It is easy to verify \Cref{hp2} holds at $t=T_{1,k}+1$.  Now we suppose \Cref{hp2} holds for all iterations in $[T_{1,k}+1, t-1]$, and prove it holds at $t$.
      
      By \Cref{p2a}, we have $\alpha_{k}^{(t-1)}\geq 0$. Thus $A_{k}^{(t)}\geq A_{k}^{(t-1)}\geq \log(K)$. By \Cref{p2b}, we have $-O\left(\frac{\alpha_{k}^{(t-1)}}{K}\right) \leq \beta_{k,n}^{(t-1)}\leq 0$. Thus,
      \begin{align*}
         |B_{k,n}^{(t)}|&\leq|B_{k,n}^{(t-1)}|+\eta O\left(\frac{\alpha_{k}^{(t-1)}}{K}\right)\\
         &\leq O\left(\frac{A_{k}^{(t-1)}}{K}\right)+\eta O\left(\frac{\alpha_{k}^{(t-1)}}{K}\right)\\
         &\leq O\left(\frac{A_{k}^{(t)}}{K}\right).
      \end{align*}

 Moreover, by the definition of $\tilde{T}^{\epsilon}_{2,k}$, for any $T_{1,k} < t\leq \tilde{T}^{\epsilon}_{2,k}$ we immediately have 

\begin{align*}
   \left(1-O\left(\frac{1}{K}\right)\right)A_{k}^{(t)}\leq A_{k}^{(t)}-\max_{m\not=k} B_{k,m}^{(t)}\leq \log\left(\left(\frac{K}{\Lba_k}-1\right)\left(\left(\frac{3}{\epsilon}\right)^{\frac{1}{2}}-1\right)\right).
\end{align*}

Therefore, $A_{k}^{(t)}\leq O(\log(\frac{K}{\epsilon}))$ for any $T_{1,k} < t\leq \tilde{T}^{\epsilon}_{2,k}$.

At iteration $t=\tilde{T}^{\epsilon}_{2,k}+1$, we have $A_{k}^{(\tilde{T}^{\epsilon}_{2,k}+1)}-\max_{m\not=k} B_{k,m}^{(\tilde{T}^{\epsilon}_{2,k}+1)}> \log\left((\frac{K}{\Lba_k}-1)(\left(\frac{3}{\epsilon}\right)^{\frac{1}{2}}-1)\right)$. Thus $A_{k}^{(\tilde{T}^{\epsilon}_{2,k}+1)} \geq \Omega(\log(\frac{K}{\epsilon}))$.

When $\{\xq=v_{k}\}\cap \{\pit \in \esb\}$, we obtain
\begin{align*}
    1-\Attn^{(\tilde{T}^{\epsilon}_{2,k}+1)}_{k}
             &=\frac{\sum_{m\not=k}\frac{|\cV_m|}{|\cV_k|}\exp(B^{(t)}_{k,m}-A^{(t)}_k)}{\sum_{m\not=k}\frac{|\cV_m|}{|\cV_k|}\exp(B^{(t)}_{k,m}-A^{(t)}_k)+1}\\
             &\leq \frac{ \exp( \max_{m\not= k} B^{(t)}_{k,m}-A^{(t)}_k) (\frac{N}{|\cV_k|}-1)}{\exp( \max_{m\not= k} B^{(t)}_{k,m}-A^{(t)}_k) (\frac{N}{|\cV_k|}-1)+1}\\
             &\leq \frac{ \exp( \max_{m\not= k} B^{(t)}_{k,m}-A^{(t)}_k) (\frac{K}{\Lba_{k}}-1)}{\exp( \max_{m\not= k} B^{(t)}_{k,m}-A^{(t)}_k) (\frac{K}{\Lba_{k}}-1)+1}\\
             &\leq  \frac{ \left((\frac{K}{\Lba_k}-1)((\frac{3}{\epsilon})^{\frac{1}{2}}-1)\right)^{-1}(\frac{K}{\Lba_{k}}-1)}{\left((\frac{K}{\Lba_k}-1)((\frac{3}{\epsilon})^{\frac{1}{2}}-1)\right)^{-1}(\frac{K}{\Lba_{k}}-1)+1}\\
             &=(\epsilon/3)^{\frac{1}{2}},
         \end{align*}
where the first inequality follows from the fact that $\frac{x}{1+x}$ monotonically increases w.r.t. $x \geq 0$.
\end{proof} 

\subsection{Phase II: Convergence: Stage II}\label{app:bal:p2-s2}
Given $k\in[K]$, define 
$$T^{\epsilon}_{2,k}:= \tilde{T}^{\epsilon}_{2,k}+O\left(\frac{K \log\left(K\epsilon^{-\frac{1}{2}}\right)}{\epsilon\eta}\right).$$
\begin{hypothesis}\label{hp3}
    Suppose $\operatorname{polylog}(K)\gg \log(\frac{1}{\epsilon})$ for $t\in (\tilde{T}^{\epsilon}_{2,k}, T^{\epsilon}_{2,k}]$. The following holds:
\begin{enumerate}[label={\alph*}.]
    \item $A_{k}^{(t)}$ is monotonically increasing but cannot exceed $O(\log(K/\epsilon))$;
    \item $B_{k,m}^{(t)}$ is monotonically decreasing and  $|B_{k,m}^{(t)}|=O(\frac{A^{(t)}_{k}}{K})$ for any $m\not=k$.
\end{enumerate}
\end{hypothesis}
\subsubsection{Technical Lemmas}
We first introduce several useful technical lemmas. 
\begin{lemma}\label{lem3t1}
    Suppose \Cref{hp3} holds at iteration $t\in (\tilde{T}^{\epsilon}_{2,k}, T^{\epsilon}_{2,k}]$. 
    If $\xq=v_k$ and $\pit\in\esb$, the following holds
    \begin{enumerate}
    \item  $\Attn^{(t)}_k=\Omega(1)$;
        \item $(1-\Attn^{(t)}_{k})^2\in[\Omega(\exp(-\operatorname{polylog}(K))),\epsilon]$. 
    \end{enumerate}
\end{lemma}
\begin{proof}
    Since $\xq=v_k$,  we have
    \begin{align*}
\Attn^{(t)}_{k}&=\frac{|\cV_k|\exp(A^{(t)}_k)}{\sum_{m\not=k}|\cV_m|\exp(B^{(t)}_{k,m})+|\cV_k|\exp(A^{(t)}_k)}\\
        &=\frac{1}{\sum_{m\not=k}\frac{|\cV_m|}{|\cV_k|}\exp(B^{(t)}_{k,m}-A^{(t)}_k)+1}.
    \end{align*}
    By \Cref{hp3}, $$\exp(B^{(t)}_{k,m}-A^{(t)}_k)\leq e^{O(\frac{\log(K/\epsilon)}{K})-\log(K)}\leq e^{O(\frac{\log(K)+\operatorname{polylog}(K)}{K})-\log(K)}\leq O\left(\frac{1}{K}\right).$$ 
Therefore,
    \begin{align*}
\Attn^{(t)}_{k} &\geq\frac{1}{O\left(\frac{1}{K}\right)(\frac{N}{|\cV_k|}-1)+1}\geq \frac{1}{O(\frac{1}{\Lba_k}-\frac{1}{K})+1}\geq \Omega(1).
    \end{align*}
We first upper-bound $1-\Attn^{(t)}_{k}$ as
\begin{align*}
    1-\Attn^{(t)}_{k}
             &=\frac{\sum_{m\not=k}\frac{|\cV_m|}{|\cV_k|}\exp(B^{(t)}_{k,m}-A^{(t)}_k)}{\sum_{m\not=k}\frac{|\cV_m|}{|\cV_k|}\exp(B^{(t)}_{k,m}-A^{(t)}_k)+1}\\
             &{\leq }\frac{ \exp( \max_{m\not= k} B^{(t)}_{k,m}-A^{(t)}_k) (\frac{N}{|\cV_k|}-1)}{\exp( \max_{m\not= k} B^{(t)}_{k,m}-A^{(t)}_k) (\frac{N}{|\cV_k|}-1)+1}\\
             &\stackrel{(a)}{\leq }\frac{ \exp( \max_{m\not= k} B^{(\tilde{T}_{2,k}^{\epsilon}+1)}_{k,m}-A^{(\tilde{T}_{2,k}^{\epsilon}+1)}_k) (\frac{N}{|\cV_k|}-1)}{\exp( \max_{m\not= k} B^{(\tilde{T}_{2,k}^{\epsilon}+1)}_{k,m}-A^{(\tilde{T}_{2,k}^{\epsilon}+1)}_k) (\frac{N}{|\cV_k|}-1)+1}\\
             &\overset{(b)}{<} \left(\frac{\epsilon}{3}\right)^{\frac{1}{2}},
         \end{align*}
         where $(a)$ holds since $\max_{m\not= k} B^{(t)}_{k,m}-A^{(t)}_k$ is non-increasing by \Cref{hp3}, and $(b)$ follows from the definition of $\tilde{T}_{2,k}^{\epsilon}$.

         Then we lower-bound $1-\Attn^{(t)}_{k}$ following the analysis  similar to that for  \Cref{lem2t1}:
   \begin{align*}
   1-\Attn^{(t)}_{k}
            &=\frac{\sum_{m\not=k}\frac{|\cV_m|}{|\cV_k|}\exp(B^{(t)}_{k,m}-A^{(t)}_k)}{\sum_{m\not=k}\frac{|\cV_m|}{|\cV_k|}\exp(B^{(t)}_{k,m}-A^{(t)}_k)+1}\\
            &\geq \frac{ \exp( \min_{m\not= k} B^{(t)}_{k,m}-A^{(t)}_k) (\frac{N}{|\cV_k|}-1)}{\exp( \min_{m\not= k} B^{(t)}_{k,m}-A^{(t)}_k) (\frac{N}{|\cV_k|}-1)+1}\\
            &\geq \frac{ \exp( \min_{m\not= k} B^{(t)}_{k,m}-A^{(t)}_k) (\frac{K}{\Uba_k}-1)}{\exp( \min_{m\not= k} B^{(t)}_{k,m}-A^{(t)}_k) (\frac{K}{\Uba_k}-1)+1}\\
            &\geq \frac{\frac{1}{e^{O(\log(K/\epsilon))}}(\frac{K}{\Uba_k}-1)}{\frac{1}{e^{O(\log(K/\epsilon))}}(\frac{K}{\Uba_k}-1)+1}\\
            &\geq \frac{\frac{1}{e^{O(\operatorname{polylog}(K))}}(\frac{K}{\Uba_k}-1)}{\frac{1}{e^{O(\operatorname{polylog}(K))}}(\frac{K}{\Uba_k}-1)+1}\\
            &\geq \Omega(\exp(-\operatorname{polylog}(K))),
        \end{align*}
        where the first three inequalities follow from the fact that $\frac{x}{1+x}$ monotonically increases w.r.t. $x \geq 0$ and $A_k^{(t)} \leq O(\log(K/\epsilon))$.
 \end{proof}
\begin{lemma}\label{lem3t2}
     Suppose \Cref{hp3} holds at iteration $t\in (\tilde{T}^{\epsilon}_{2,k}, T^{\epsilon}_{2,k}]$.  If $\xq=v_k$ and $\pit\in\esb$,  for $n\not=k$, the following holds
 $$\Attn^{(t)}_n=\Theta\left(\frac{1-\Attn^{(t)}_{k}}{K}\right).$$
    
\end{lemma}
\begin{proof}
By definition,
    \begin{align*}
\Attn^{(t)}_{n}&=\frac{|\cV_n|\exp(B^{(t)}_{k,n})}{\sum_{m\not=k}|\cV_m|\exp(B^{(t)}_{k,m})+|\cV_k|\exp(A^{(t)}_k)}.
    \end{align*}
    By \Cref{hp3}, $$e^{-O(\frac{\log(K)-\log(\epsilon)}{K})}\leq \exp(B^{(t)}_{k,m}-B^{(t)}_{k,n})\leq e^{O(\frac{\log(K)-\log(\epsilon)}{K})}.$$ 
    
    Combining with the fact that $-\log(\epsilon)\ll \operatorname{polylog}(K)$, we obtain 
   
    \begin{align*}
       \frac{\Attn^{(t)}_n}{1-\Attn^{(t)}_{k}} =\frac{|\cV_n|\exp(B^{(t)}_{k,n})}{\sum_{m\not=k}|\cV_m|\exp(B^{(t)}_{k,m})}= \frac{1}{\sum_{m\not=k}\frac{|\cV_m|}{|\cV_n|}\exp(B^{(t)}_{k,m}-B^{(t)}_{k,n})}=\Theta\left(\frac{1}{K}\right).
    \end{align*}
\end{proof}
\subsubsection{Controlling Gradient Updates in Stage II of Phase II}
\begin{lemma}\label{p3a}
    At each iteration $t\in (\tilde{T}^{\epsilon}_{2,k}, T^{\epsilon}_{2,k}]$. If \Cref{hp3} holds for $t$, then $\alpha_k^{(t)}\geq0$ and satisfies 
    \begin{align*}
        \alpha_{k}^{(t)}\leq O\left(\frac{\epsilon}{K}\right).
    \end{align*}
\end{lemma}
\begin{proof}
    By the gradient expression in  \Cref{app:lem:gd},
    \begin{align*}
  \alpha_{k}^{(t)}&=
\mathbb{E}\left[\mathbf{1}\{\xq=v_k\}\Attn^{(t)}_{k }\cdot \left(
 \sum_{m\not= k}{\Attn^{(t)}_{m}}^2+(1-\Attn^{(t)}_k)^2\right)\right]\\
 &\leq p_k\mathbb{E}\left[\Attn^{(t)}_{k }\cdot \left(
 \sum_{m\not= k}{\Attn^{(t)}_{m}}^2+(1-\Attn^{(t)}_k)^2\right)\mid\{\xq=v_k\} \cap \esb\right]\\
 & \quad  +6p_k\exp \left(-\frac{\cb^2 N}{25 K^2}\right)\\
&\leq  p_k\cdot O(\epsilon)+6p_k\exp \left(-\frac{\cb^2 N}{25 K^2}\right)\\
&\leq O\left(\frac{\epsilon}{K}\right),
\end{align*}
where the second inequality follows from \Cref{lem3t1,lem3t2}, and the last inequality follows from the fact that $p_k=\Theta\left(\frac{1}{K}\right)$ and $\epsilon=\Omega(\exp(-\operatorname{polylog}(K)))\gg 6\exp \left(-\frac{\cb^2 N}{25 K^2}\right)$.
\end{proof}
\begin{lemma}\label{p3b}
    At each iteration $t\in (\tilde{T}^{\epsilon}_{2,k}, T^{\epsilon}_{2,k}]$, if \Cref{hp3} holds for $t$, for any $n\not=k$, $\beta_{k,n}^{(t)}$ satisfies 
    \begin{align*}
       - O\left(\frac{\alpha^{(t)}_k}{K}\right) \leq \beta_{k,n}^{(t)}\leq 0.
    \end{align*}
\end{lemma}
\begin{proof}
        Note that  conditioned on the event $\{\xq=v_k\} \cap \{\pit \in \esb\}$, $\Attn_k^{(t)}=\Omega(1)$, and $\max_{m\not=k}\Attn^{(t)}_{m}=O(\frac{\epsilon^{\frac{1}{2}}}{K})$. Thus, 
\begin{align*} \sum_{m\not= k}{\Attn^{(t)}_{m}}^2-\Attn^{(t)}_{n} -\Attn^{(t)}_{k}(1-\Attn^{(t)}_k)&\leq  \max_{m\not=k}\Attn^{(t)}_{m}\sum_{m\not= k}\Attn^{(t)}_{m}-\Attn^{(t)}_{k}(1-\Attn^{(t)}_k)\\
&=-(1-\Attn^{(t)}_k)(\Attn^{(t)}_{k}-\max_{m\not=k}\Attn^{(t)}_{m})\\
&\leq -\Omega(1-\Attn^{(t)}_k) \leq -\Omega\left(\exp(-\operatorname{polylog}(K))\right).
\end{align*}
Therefore, by the gradient expression in \Cref{app:lem:gd} and the fact that $N \gg K^3$,
\begin{align*}
\beta_{k,n}^{(t)}\leq 6 \exp \left(-\frac{\cb^2 N}{25 K^2}\right)-\Omega(\exp(-\operatorname{polylog}(K)))<0.
\end{align*}
Moreover, following the analysis similar to that for  \Cref{p2b}, we have 
    \begin{align*}
     - \beta_{k,n}^{(t)}&\leq p_k\mathbb{E}\left[\Attn^{(t)}_{n}\cdot \left(
 \Attn^{(t)}_{n} +\Attn^{(t)}_{k}(1-\Attn^{(t)}_k)\right)\mid\{\xq=v_k\} \cap \esb\right]+p_k\mathbb{P}({\esb}^{c})\\
 &\leq p_k\mathbb{E}\left[\Theta(\frac{1-\Attn^{(t)}_{k}}{K})\cdot O\left(
\Attn^{(t)}_{k}(1-\Attn^{(t)}_k)\right)\mid\{\xq=v_k\} \cap \cE^*\right]\\
 & \qquad \qquad +6 p_k\exp \left(-\frac{\cb^2 N}{25 K^2}\right)\\
 &=p_k\mathbb{E}\left[O(\frac{\Attn^{(t)}_{k}(1-\Attn^{(t)}_{k})^2}{K})\mid\{\xq=v_k\} \cap \cE^*\right]+6 p_k\exp \left(-\frac{\cb^2 N}{25 K^2}\right)\\
 &\leq O\left(\frac{\alpha_{k}^{(t)}}{K}\right),
    \end{align*}
    where the last inequality follows from the gradient expression of $\alpha_k^{(t)}$ in \Cref{app:lem:gd} and because $\alpha_{k}^{(t)}\gg 6 \exp \left(-\frac{\cb^2 N}{25 K^2}\right)$.
\end{proof}
\subsubsection{Controlling Loss in Stage II of Phase II}
\begin{lemma}\label{p3c}
  Given $k\in[K]$, and $0<\epsilon<1$, suppose $\operatorname{polylog}(K)\gg \log(\frac{1}{\epsilon})$. At each iteration $t\in (\tilde{T}^{\epsilon}_{2,k}, T^{\epsilon}_{2,k}]$, if \Cref{hp3} holds for $t$, then we have  $\tilde{L}_{k}(\theta^{(t)})<\frac{p_k\epsilon}{2}$.  
\end{lemma}
\begin{proof}
By the gradient expression in \Cref{app:lem:gd}, we have
    \begin{align*}
  \tilde{L}_{k}(\theta^{(t)})
  &=\frac{1}{2}\mathbb{E}\left[\mathbf{1}\{\xq=v_{k}\cap \pit\in\esb\}\left(\yq-\left\langle w, \xq\right\rangle\right)^2\right]\\
    &=\frac{1}{2}\mathbb{E}\left[\mathbf{1}\{\xq=v_{k}\cap \pit\in\esb\}\left(
        \sum_{m\not= k}{\Attn^{(t)}_{m}}^2+(1-\Attn^{(t)}_k)^2\right)\right]\\
        &\leq \frac{1}{2}p_k\mathbb{P}\left(\pit\in\esb\right)\cdot \mathbb{E}\left[\left(O\left(\frac{1}{K}\right)+1\right)\left(1-\Attn^{(t)}_k\right)^2\bigg| \xq=v_{k}\cap \pit\in\esb \right]\\
        &\leq \frac{1}{2}p_k\cdot \left(1+O\left(\frac{1}{K}\right)\right)\cdot {\epsilon} \\
        &\leq \frac{2p_k\epsilon}{3},
\end{align*} 
where the first inequality follows from \Cref{lem3t2}, and the second inequality follows from \Cref{lem3t1}.
\end{proof}
\subsubsection{End of Stage II of Phase II}
\begin{lemma}\label{end3}
    Given $k\in[K]$, and $0<\epsilon<1$, suppose $\operatorname{polylog}(K)\gg \log(\frac{1}{\epsilon})$. Then \Cref{hp3} holds for all  $\tilde{T}^{\epsilon}_{2,k}<t\leq T_{2,k}^\epsilon=\tilde{T}^{\epsilon}_{2,k}+O\left(\frac{K \log\left(K\epsilon^{-\frac{1}{2}}\right)}{\epsilon\eta}\right)$.
\end{lemma}
\begin{proof}
It is easy to verify \Cref{hp3} holds at $t=\tilde{T}^{\epsilon}_{2,k}+1$.  Now we suppose \Cref{hp3} holds for all iterations $\tilde{T}^{\epsilon}_{2,k}\leq t-1$, and prove it holds at $t$.

For the first claim, we can upper-bound the update of $A_{k}^{(t)}$ by \Cref{p3a} as follows:
\begin{align*}
    A_{k}^{(t)}&\leq A_{k}^{(t-1)}+\eta\cdot O(\frac{\epsilon}{K})\\
    &\leq A_{k}^{(\tilde{T}^{\epsilon}_{2,k}+1)}+\eta(t-\tilde{T}^{\epsilon}_{2,k}-1)\cdot O(\frac{\epsilon}{K})\\
    &\leq O(\log(K/\epsilon))+\eta O(\frac{K \log\left(K\epsilon^{-\frac{1}{2}}\right)}{\epsilon\eta})\cdot O(\frac{\epsilon}{K})\\
    &=O(\log(K/\epsilon)).
\end{align*}
The second claim follows from \Cref{p3b} and the analysis similar to that for \Cref{end2}. 

\end{proof}

\subsection{Proof of Theorem~\ref{thm: bal K} for Balanced Case}\label{app:bal:proof}
\begin{theorem}[Restatement of \Cref{thm: bal K} for balanced features]
Suppose $p_k=\Theta\left(\frac{1}{K}\right)$ for each $k \in [K]$.
For any $0<\epsilon<1$, suppose $N\geq \poly(K)$ and $\polylog(K)\gg \log(\frac{1}{\epsilon}) $. We apply GD to train the loss function given in \cref{eq:obj}.
Then with at most $T^*=O(\frac{\log(K)K^2}{\eta}+\frac{K \log\left(K\epsilon^{-\frac{1}{2}}\right)}{\epsilon\eta})$ iterations, we have
\begin{enumerate}[label={\arabic*.}]
\item The loss converges: $L(\theta^{(T^*)})-L^* \leq \epsilon$, where $L^{*}=\Theta(e^{-\poly(K)})$ is the global minimum of the population loss in \cref{eq:obj}. 
\item Attention score concentrates: if $\xq=v_k$, with probability at least  $1-e^{-\Omega(\poly(K))}$\footnote{The randomness originates from the first $N$ input tokens in the test prompt.},  
 the one-layer transformer nearly ``pays all attention" to input tokens featuring $v_k$, i.e., $(1-\Attn_k^{(T^*)})^2 \leq O(\epsilon)$. 
\end{enumerate}
\end{theorem}
\begin{proof}

Denote $T^{*}=\max_{k\in[K]}\tilde{T}^{\epsilon}_{2,k}+1=O(\frac{\log(K)K^2}{\eta}+\frac{K \log\left(K\epsilon^{-\frac{1}{2}}\right)}{\epsilon\eta})$. Thus for any $k$, at iteration $T^{*}$, it is in stage II of the convergence phase, i.e.,  $T^{*}\in(\tilde{T}^{\epsilon}_{2,k}, T^{\epsilon}_{2,k}]$. Then by \Cref{p3c,end3}, for any $k\in[K]$, we obtain:
\begin{align*}
    \tilde{L}_{k}(\theta^{(T^{*})})\leq \frac{2p_k\epsilon}{3}.
\end{align*} 
Therefore
\begin{align*}
    L(\theta^{(T^{*})})- \Lol&= \sum_{k=1}^K (L_k(\theta^{(T^{*})})- \Lol_k)\\
    &\leq \sum_{k=1}^{K} \left(\tilde{L}_{k}(\theta^{(T^{*})})+3 p_k\exp \left(-\frac{\cb^2 N}{25 K^2}\right)\right)\\
    &\leq  \sum_{k=1}^{K} \frac{2p_k\epsilon}{3}+ 3\exp \left(-\frac{\cb^2 N}{25 K^2}\right)\\
    &\leq \frac{2\epsilon}{3}+ 3\exp \left(-\frac{\cb^2 N}{25 K^2}\right)\\& \leq \epsilon,
\end{align*} 
where the first inequality follows from \Cref{app:lem:optb2}.

Finally, by \Cref{app:lem:optb1},
\begin{align*}
     L(\theta^{(T^{*})})- L^* \leq L(\theta^{(T^{*})})-\Lol \leq \epsilon.
\end{align*}
\end{proof}

\newpage
\section{Analysis for the Imbalanced Case: Under-represented Features}\label{app:im:vk}
In this section, we present the analysis of the prediction error when the query token features an under-represented feature $v_k$ with $k>1$ in the imbalanced case. We first discuss the outline of our proof.
\subsection{Roadmap of the Proof}

We will analyze the convergence of the training process via four phases of dynamics.  At the beginning of each phase, we will establish an induction hypothesis, which we expect to remain valid throughout that phase. Subsequently, we will analyze the dynamics under such a hypothesis within the phase, aiming to provide proof of the hypothesis by the end of the phase.

The main idea of the proof lies in analyzing the GD dynamics of $A_{k}^{(t)}$ and $B_{k,n}^{(t)}$. From \Cref{app:def:dynamics} and \Cref{app:lem:gd}, we have
\begin{align*}
   A_{k}^{(t+1)}&= A_{k}^{(t)}+\eta \alpha_{k}^{(t)},\\
   B_{k,n}^{(t+1)}&= B_{k,n}^{(t)}+\eta \beta_{k,n}^{(t)},
\end{align*}
where 
    \begin{align*}
&\alpha_{k}^{(t)}=\mathbb{E}\left[\mathbf{1}\{\xq=v_k\}\Attn^{(t)}_{k }\cdot \left(
 \sum_{m\not= k}{\Attn^{(t)}_{m}}^2+(1-\Attn^{(t)}_k)^2\right)\right],\\
 &\beta_{k,n}^{(t)}=\mathbb{E}\left[\mathbf{1}\{\xq=v_k\}\Attn^{(t)}_{n}\cdot \left(
 \sum_{m\not= k}{\Attn^{(t)}_{m}}^2-\Attn^{(t)}_{n} -\Attn^{(t)}_{k}(1-\Attn^{(t)}_k)\right)\right].
    \end{align*}
We divide the learning process of the under-represented feature $v_k$ with $k>1$ into the following four phases. 
\begin{itemize}
    \item \textbf{Phase I} ($t\in[0,T_{1,k}]$, \Cref{app:im:p1}): At initialization, $B_{k,1}^{(t)}$ enjoys a much larger reduction rate, i.e., $\beta_{k,1}<0$ and $|\beta_{k,1}|$ is large. Therefore, the decrease of $B_{k,1}^{(t)}$ will dominate the dynamics during phase I.
    \item \textbf{Phase II} ($t\in(T_{1,k},T_{2,k}]$, \Cref{app:im:p2}):  At time $T_{2,k}+1$, the decrease  of $B_{k,1}^{(t)}$ becomes slower, and the same happens to $|\beta^{(t)}_{k,1}|$. Their decreasing rate drops to be closer to the increasing rate of $\alpha_{k}^{(t)}$.      
This marks the beginning of phase II. Shortly after entering this phase, the previous dominance of reduction of $B_{k,1}^{(t)}$  diminishes, as $|\beta^{(t)}_{k,1}|$ approaches a comparable order of the magnitude to $\alpha_{k}^{(t)}$. At this point, there is a shift in the leading influence, with the growth of $A_{k}^{(t)}$ taking over. 
    \item \textbf{Phase III} ($t\in(T_{2,k},T_{3,k}]$, \Cref{app:im:p3}): Following the transitional phase, $\alpha^{(t)}_{k}$ grows from the value of $\Theta(\frac{1}{K^{1.5}})$, whereas $|\beta^{(t)}_{k,1}|$ and $|\beta^{(t)}_{k,n}|$ for $n\neq k,1$ stay at much lower values ($\leq O(\frac{1}{K^{1.98}})$ and $\leq O\left(\frac{1}{K^3}\right)$ respectively). This consistent gap in magnitude between $\alpha_{k}^{(t)}$ and $\beta_{k,n}^{(t)}$ leads to the continuously rapid growth of $A_{k}^{(t)}$, while $B_{k,n}^{(t)}$ remains relatively unchanged.
                    \item \textbf{Phase IV} ($t\in(T_{3,k},T^{\epsilon}_{4,k}]$, \Cref{app:im:p4}): At $t=T_{3,k}+1$, we achieve the desired attention structures for query tokens featuring the under-represented feature $v_k$. Then we establish a connection between $\alpha^{(t)}_k$ and the prediction error via analyzing the change of $1-\Attn^{(t)}_{k}$ that diminishes, leading to the subsequent proof of convergence.
\end{itemize}
We finally combine all results in the above four phases to prove the main \Cref{thm: unblc} for underrepresented features (\Cref{app:sec:un}).
\subsection{Phase I: Decrease of Dominant Feature}\label{app:im:p1}
In this section, we will delve into the initial phase of learning dynamics, aiming at mitigating the high occurrence bias of the dominant feature $v_1$. Specifically, for $k>1$, $B_{k,1}$ will undergo significant decrease during this phase. Let us begin by defining phase I.


For the $k$-th feature $v_k$ with $k>1$, we define 
 phase I as all iterations $t\leq T_{1,k}$, where
$$
T_{1,k} \triangleq \max \left\{t: B_{k,1}^{(t)} \geq -0.49\log(K)\right\}.
$$
We state the following induction hypothesis, which will hold throughout phase I:
\begin{hypothesis}\label{hpk1}
    Given $k>1$, for each $0 \leq t \leq T_{1,k}$, the following holds:
    \begin{enumerate}[label={\alph*}.]
        \item $A_{k}^{(t)}$ is monotonically increasing and $A_{k}^{(t)}\in [0,O(\frac{\log(K)}{K^{0.02}})]$;
        \item $B_{k,1}^{(t)}$ is monotonically decreasing and $B_{k,1}^{(t)}\in [-0.49\log(K),0]$;
        \item $|B_{k,n}^{(t)}|=O(\frac{A_{k}^{(t)}-B_{k,1}^{(t)}}{K})$ and $B_{k,n}^{(t)}>B_{k,1}^{(t)}$ for any $n\not=k,1$.
    \end{enumerate}
    \end{hypothesis} 
 
\subsubsection{Technical Lemmas}
We first introduce several technical lemmas that will be used for the proof of \Cref{hpk1}.
\begin{lemma}\label{lemk1t1}
    If \Cref{hpk1} holds at iteration $0 \leq t\leq T_{1,k}$, 
    for the prompt satisfying $\xq=v_k$ and $\pit\in\esi$, the following holds
    \begin{enumerate}
    \item  $\Attn^{(t)}_k=\Theta\left(\frac{1}{K}\right)$;
    \item  $\Attn^{(t)}_1=\Omega\left(\frac{1}{K^{0.49}}\right)$;
        \item $1-\Attn^{(t)}_{1}-\Attn^{(t)}_{k}\geq \Omega(1)$.
    \end{enumerate}
\end{lemma}
\begin{proof}
    Since $\xq=v_k$, and $|\cV_k|>0$ for $\pit\in\esi$, we have
    \begin{align*}
\Attn^{(t)}_{k}&=\frac{|\cV_k|e^{{v_k}^{\top}Q^{(t)}v_{k}}}{\sum_{j\in [N]}e^{{E^{x}_{j}}^{\top}Q^{(t)}v_{k}}}\\
&=\frac{|\cV_k|\exp(A^{(t)}_k)}{\sum_{m\not=k}|\cV_m|\exp(B^{(t)}_{k,m})+|\cV_k|\exp(A^{(t)}_k)}\\
        &=\frac{1}{\sum_{m\not=k}\frac{|\cV_m|}{|\cV_k|}\exp(B^{(t)}_{k,m}-A^{(t)}_k)+1}.
    \end{align*}

    By \Cref{hpk1},  we have 
    \begin{itemize}
        \item  for $m\not=1,k$, $e^{-O(\frac{\log(K)}{K^{0.02}})}\leq \exp(B^{(t)}_{k,m}-A^{(t)}_k)\leq e^{O(\frac{\log(K)}{K})}$;
        \item for $m=1$,   $e^{\left(-0.49\log(K)-O(\frac{\log(K)}{K^{0.02}})\right)}\leq \exp(B^{(t)}_{k,1}-A^{(t)}_k)\leq e^{0}$.
    \end{itemize}
Combining with the fact that $\sum_{m\not=k}\frac{|\cV_m|}{|\cV_k|}=\Theta(K)$ for $\pit\in\esi$, we have
     \begin{align*}
\Attn^{(t)}_{k} 
\geq \Omega\left(\frac{1}{K}\right).
    \end{align*}
On the other hand, since $\frac{N-|\cV_1|}{|\cV_k|}$ is still $\Theta(K)$, we have
\begin{align*}
    &\Attn^{(t)}_{k} \leq\frac{1}{e^{-O(\frac{\log(K)}{K^{0.02}})}(\frac{N-|\cV_1|}{|\cV_k|}-1)+e^{\left(-0.49\log(K)-O(\frac{\log(K)}{K^{0.02}})\right)}\frac{|\cV_1|}{|\cV_k|}+1}\leq O\left(\frac{1}{K}\right).
        \end{align*}
By similar analysis, we have 
            \begin{align*}
                \Attn^{(t)}_{1}
                &=\frac{|\cV_1|\exp(B^{(t)}_{k,1})}{\sum_{m\not=k}|\cV_m|\exp(B^{(t)}_{k,m})+|\cV_k|\exp(A^{(t)}_k)}\\
                        &=\frac{1}{\sum_{m\not=1,k}\frac{|\cV_m|}{|\cV_k|}\exp(B^{(t)}_{k,m}-B^{(t)}_{k,1})+\frac{|\cV_k|}{|\cV_1|}\exp(A^{(t)}_{k}-B^{(t)}_{k,1})+1}.
                    \end{align*}
                    By \Cref{hpk1}, 
                    \begin{itemize}
                        \item   for $m\not=1,k$, we have $e^{0}\leq \exp(B^{(t)}_{k,m}-B^{(t)}_{k,1})\leq e^{0.49\log(K)+O(\frac{\log(K)}{K})}$;
                        \item 
                     $e^{0}\leq \exp(A^{(t)}_k- B^{(t)}_{k,1})\leq e^{0.49\log(K)+O(\frac{\log(K)}{K})}$.
                    \end{itemize}
                Hence,
                     \begin{align*}
                \Attn^{(t)}_{1} &\geq\frac{1}{e^{0.49\log(K)+O(\frac{\log(K)}{K})}(\frac{N}{|\cV_1|}-1)+1}\geq 
                \Omega\left(\frac{1}{K^{0.49}}\right),
                    \end{align*}
                    where the last inequality holds since $\frac{N}{|\cV_1|}=\Theta(1)$ for $\pit\in\esi$.

        
                            For the last statement, 
                            \begin{align*}
        & 1-\Attn^{(t)}_{1}\geq    \frac{e^{0}(\frac{N}{|\cV_1|}-1)}{e^{0}(\frac{N}{|\cV_1|}-1)+1}\geq \Omega(1).
                            \end{align*}
                          
                            Combining with the fact that $\Attn_{k}^{(t)}=\Theta\left(\frac{1}{K}\right)$, we have 
                            \begin{align*}
                               1- \Attn^{(t)}_{k} -\Attn^{(t)}_{1} \geq \Omega(1).
                                    \end{align*}
\end{proof}
\begin{lemma}\label{lemk1t2}
    If \Cref{hpk1} holds at iteration $0\leq t\leq T_{1,k}$,  for the prompt satisfying $\xq=v_k$ and $\pit\in\esi$, the following holds 
 $$\Attn^{(t)}_n=O\left(\frac{1-\Attn^{(t)}_{k}-\Attn^{(t)}_{1}}{K}\right).$$
\end{lemma}
\begin{proof}
     Since $\xq=v_k$, 
     we have
    \begin{align*}
\Attn^{(t)}_{n}&=\frac{|\cV_n|e^{{v_n}^{\top}Q^{(t)}v_{k}}}{\sum_{j\in [N]}e^{{E^{x}_{j}}^{\top}Q^{(t)}v_{k}}}\\&=\frac{|\cV_n|\exp(B^{(t)}_{k,n})}{\sum_{m\not=k}|\cV_m|\exp(B^{(t)}_{k,m})+|\cV_k|\exp(A^{(t)}_k)}.
    \end{align*}
    By \Cref{hpk1}, 
    for $m,n\not=1$, $$e^{-O(\frac{\log(K)}{K})}\leq \exp(B^{(t)}_{k,m}-B^{(t)}_{k,n})\leq e^{O(\frac{\log(K)}{K})}.$$ 
    Combining with the fact that $\frac{|\cV_m|}{|\cV_n|}=\Theta(1)$ for $\pit\in\esi$, we have
    \begin{align*}
       \frac{\Attn^{(t)}_n}{1-\Attn^{(t)}_{k}-\Attn^{(t)}_{1}} &=\frac{|\cV_n|\exp(B^{(t)}_{k,n})}{\sum_{m\not=1,k}|\cV_m|\exp(B^{(t)}_{k,m})}\\&= \frac{1}{\sum_{m\not=k,1}\frac{|\cV_m|}{|\cV_n|}\exp(B^{(t)}_{k,m}-B^{(t)}_{k,n})
       }
       \\
       &\leq O\left(\frac{1}{K}\right).
    \end{align*}
\end{proof}

\subsubsection{Controlling  Gradient Updates in Phase I}
\begin{lemma}\label{pk1a}
    Given $k>1$,  if \Cref{hpk1} holds at iteration $0\leq t\leq T_{1,k}$, then $\alpha_k^{(t)}\geq0$ and satisfies 
    \begin{align*}
        \alpha_{k}^{(t)}= \Theta\left(\frac{1}{K^2}\right).
    \end{align*}
\end{lemma}
\begin{proof}
    By the gradient expression in \Cref{app:lem:gd}, we have 
    \begin{align*}
 \alpha_{k}^{(t)}&=
\mathbb{E}\left[\mathbf{1}\{\xq=v_k\}\Attn^{(t)}_{k }\cdot \left(
 \sum_{m\not= k}{\Attn_{m}^{(t)}}^2+(1-\Attn^{(t)}_k)^2\right)\right]\\
&=\mathbb{E}\left[\mathbf{1}\{\xq=v_k\cap\esi\}\Attn^{(t)}_{k }\cdot \left(
 \sum_{m\not= k}{\Attn_{m}^{(t)}}^2+(1-\Attn^{(t)}_k)^2\right)\right]\\
 &\quad+\mathbb{E}\left[\mathbf{1}\{\xq=v_k\cap{\esi}^c\}\Attn^{(t)}_{k }\cdot \left(
 \sum_{m\not= k}{\Attn_{m}^{(t)}}^2+(1-\Attn^{(t)}_k)^2\right)\right]\\
 &\stackrel{(a)}{\leq} p_k\cdot\mathbb{P}(\pit\in\esi)\mathbb{E}\left[\Attn^{(t)}_{k }\cdot \left(
 \sum_{m\not= k}{\Attn_{m}^{(t)}}^2+(1-\Attn^{(t)}_k)^2\right)\bigg|\{\xq=v_k\} \cap \esi\right]\\
 &\quad + 2 p_k\cdot \mathbb{P}(\pit\in{\esi}^c)\\
&\stackrel{(b)}{\leq} p_k\cdot\mathbb{E}\left[\Attn^{(t)}_{k }\cdot\left({O\left(\frac{1}{K}\right)+ \Attn_{1}^{(t)}}^2+ (1-\Attn^{(t)}_k)^2\right)\bigg|\{\xq=v_k\} \cap \esi\right]\\
&\quad + 2 p_k\cdot \mathbb{P}(\pit\in{\esi}^c)\\
&\stackrel{(c)}{\leq} O\left(\frac{1}{K^2}\right),
\end{align*}
where $(a)$ follows from the fact that $\xq$ and $\pit$ are independently sampled, and $\Attn^{(t)}_{k}\cdot( \sum_{m\not= k}{\Attn_{m}^{(t)}}^2+(1-\Attn^{(t)}_k)^2)$ is upper-bounded by $2$ on the event $\{\pit\in{\esi}^{c}\}$, $(b)$ follows by applying \Cref{lemk1t2} to $\Attn^{(t)}_{m}$ for $m\not=1,k$, and $(c)$ follows from \Cref{lemk1t1}, our choice of  $p_k$, \Cref{app:lem:prob-im}, and the evident bound:
$$
3 \exp \left(-\frac{\ci^2 N}{25 K^2}\right)\ll O\left(\frac{1}{K}\right).
$$
Similarly, we can show that $\alpha_{k}^{(t)}\geq  \Omega\left(\frac{1}{K^2}\right)$.
\end{proof}
\begin{lemma}\label{pk1b}
    Given $k>1$,  if \Cref{hpk1} holds at iteration $0\leq t\leq T_{k,1}$, then $\beta_{k,1}^{(t)}<0$ satisfies 
    \begin{align*}
        |\beta_{k,1}^{(t)}|\geq \Omega\left(\frac{1}{K^{1.98}}\right).
    \end{align*}
\end{lemma}
\begin{proof}
    We first derive 
\begin{align} \sum_{m\not= k}&{\Attn^{(t)}_{m}}^2-\Attn^{(t)}_{1} -\Attn^{(t)}_{k}(1-\Attn^{(t)}_k)\nonumber\\
&= \sum_{m\not= 1,k}{\Attn^{(t)}_{m}}^2-\Attn^{(t)}_{1}(1-\Attn^{(t)}_{1}) -\Attn^{(t)}_{k}(1-\Attn^{(t)}_k)\nonumber\\&\leq  \max_{m\not=1,k}\Attn^{(t)}_{m}(1-\Attn^{(t)}_{1}-\Attn^{(t)}_{k})-\Attn^{(t)}_{1}(1-\Attn^{(t)}_{1})-\Attn^{(t)}_{k}(1-\Attn^{(t)}_k)\nonumber\\
&\leq -(1-\Attn^{(t)}_k-\Attn^{(t)}_1)(\Attn^{(t)}_{1}+\Attn^{(t)}_{k}-\max_{m\not=1,k}\Attn^{(t)}_{m}).\label{app:eq:p1:beta}
\end{align}
Therefore,  by the gradient expression in \Cref{app:lem:gd}, we have 
\begin{align*}
\beta_{k,1}^{(t)}\leq& \mathbb{E}\left[\mathbf{1}\{\xq=v_k\cap \pit\in\esi\}\Attn^{(t)}_{1}\cdot \left(
  \sum_{m\not= k}{\Attn^{(t)}_{m}}^2-\Attn^{(t)}_{1} -\Attn^{(t)}_{k}(1-\Attn^{(t)}_k)\right)\right]\\&+\mathbb{E}\left[\mathbf{1}\{\xq=v_k\cap \pit\in {\esi}^{c}\}\Attn^{(t)}_{1}\cdot \left(
  \sum_{m\not= k}{\Attn^{(t)}_{m}}^2\right)\right]\\
  \stackrel{(a)}{\leq}&p_k\cdot\mathbb{P}(\pit\in{\esi}^{c}) + p_k\cdot\mathbb{P}(\pit\in\esi) \\
  & \qquad \cdot \mathbb{E}\left[-\Omega(\frac{(\Attn^{(t)}_{1}+\Attn^{(t)}_{k}-\max_{m\not=1,k}\Attn^{(t)}_{m})}{K^{0.49}})\bigg|\{\xq=v_k\} \cap \esi\right]\\
{\leq}& p_k\cdot \left(-\Omega\left(\frac{1}{K^{0.98}}\right)\right)+3 p_k\exp \left(-\frac{\ci^2 N}{25 K^2}\right)\\
=&-\Omega\left(\frac{1}{K^{1.98}}\right),
\end{align*}
where $(a)$ follows from \cref{app:eq:p1:beta} and \Cref{lemk1t1},
 and the last equality holds since
\begin{align*}
\frac{1}{K^{0.98}}\gg  \exp \left(-\frac{\ci^2 N}{25 K^2}\right).
\end{align*}
\end{proof}
\begin{lemma}\label{pk1c}
    If \Cref{hpk1} holds at iteration $0\leq t\leq T_{k,1}$,  for any $n\not=1,k$, $\beta_{k,n}^{(t)}$ satisfies 
    \begin{align*}
        |\beta_{k,n}^{(t)}|\leq O\left(\frac{\alpha_{k}^{(t)}-\beta_{k,1}^{(t)}}{K}\right).
    \end{align*}
\end{lemma}
\begin{proof}
        By the gradient expression in  \Cref{app:lem:gd}, we have 
    \begin{align}
        \beta_{k,n}^{(t)}&\leq \mathbb{E}\left[\mathbf{1}\{\xq=v_k\}\Attn^{(t)}_{n}\cdot \left(
 \sum_{m\not= k}{\Attn^{(t)}_{m}}^2\right)\right]\label{k1I1}\\
 -\beta_{k,n}^{(t)}&\leq \mathbb{E}\left[\mathbf{1}\{\xq=v_k\}\Attn^{(t)}_{n}\cdot \left(
 \Attn^{(t)}_{n} +\Attn^{(t)}_{k}(1-\Attn^{(t)}_k)\right)\right]\label{k1I2}
    \end{align}
    We further upper-bound \cref{k1I1} as,  
    \begin{align*}
      \beta_{k,n}^{(t)}&\leq \mathbb{E}\left[\mathbf{1}\{\xq=v_k\cap \pit\in\esi\}\Attn^{(t)}_{n}\cdot \left(
        \sum_{m\not= k}{\Attn^{(t)}_{m}}^2\right)\right]\\
        &\quad +\mathbb{E}\left[\mathbf{1}\{\xq=v_k\cap \pit\in{\esi}^{c}\}\Attn^{(t)}_{n}\cdot \left(
        \sum_{m\not= k}{\Attn^{(t)}_{m}}^2\right)\right]\\
        &
        {\leq}p_k\cdot\mathbb{P}(\pit\in\esi)\cdot\mathbb{E}\left[\Attn^{(t)}_{n}\cdot \left(
{\Attn^{(t)}_{1}}^2+O\left(\frac{1}{K}\right)\right)\mid\{\xq=v_k\} \cap \esi\right]\\
& \quad +p_k\cdot\mathbb{P}(\pit\in{\esi}^{c})\\
 &\stackrel{(a)}
 {\leq} 
 O\left(\frac{1}{K^3}\right)+O\left(\frac{|\beta_{k,1}^{(t)}|}{K}\right)+3 p_k\exp \left(-\frac{\ci^2 N}{25 K^2}\right)\\
 &{\leq}O\left(\frac{1}{K^3}\right)+O\left(\frac{|\beta_{k,1}^{(t)}|}{K}\right).
    \end{align*}
where $(a)$ follows from the following two observations from \Cref{lemk1t2}:
    \begin{align*}
|\beta_{k,1}^{(t)}|\geq p_k\cdot\mathbb{P}(\pit\in\esi)\cdot\mathbb{E}\left[\Omega \left(
{\Attn^{(t)}_{1}}^2\right)\bigg|\{\xq=v_k\} \cap \esi\right],
    \end{align*}
   and $\Attn_{n}^{(t)}\leq O\left(\frac{1}{K}\right)$.

    To further upper-bound \cref{k1I2}, we have 
    \begin{align*}
      -&\beta_{k,n}^{(t)}\\&\leq p_k\mathbb{E}\left[\Attn^{(t)}_{n}\cdot \left(
 \Attn^{(t)}_{n} +\Attn^{(t)}_{k}(1-\Attn^{(t)}_k)\right)\bigg|\{\xq=v_k\} \cap \esi\right]+p_k\cdot\mathbb{P}({\esi}^{c})\\
 &\stackrel{(a)}
 {\leq}p_k\cdot\mathbb{P}(\pit\in{\esi}^{c})+ p_k\cdot\mathbb{P}({\esi})\mathbb{E}\left[O\left(\frac{1-\Attn^{(t)}_{k}}{K}\right)\right.\\
 &\qquad\qquad\qquad\left.\cdot \left( O\left(\frac{1-\Attn^{(t)}_{k}}{K}\right)+
\Attn^{(t)}_{k}(1-\Attn^{(t)}_k)\right)\bigg|\{\xq=v_k\} \cap \esi\right]\\
 &
 {\leq}p_k\cdot\mathbb{P}(\esi)\mathbb{E}\left[O(\frac{\Attn^{(t)}_{k}(1-\Attn^{(t)}_{k})^2}{K})\bigg|\{\xq=v_k\} \cap \esi\right]+3p_k\exp \left(-\frac{\ci^2 N}{25 K^2}\right)\\
 &\leq O\left(\frac{\alpha_{k}^{(t)}}{K}\right),
    \end{align*}
    where $(a)$ follows from  \Cref{lemk1t2}, 
  and the last inequality follows from the analysis in the proof of \Cref{pk1a}, and from the fact that 
    \begin{align*}
       \alpha_{k}^{(t)}&\geq \Omega\left(\frac{1}{K^2}\right)\gg 3\exp \left(-\frac{\ci^2 N}{25 K^2}\right).
    \end{align*}
    Thus, we obtain 
    $$
|\beta_{k,n}^{(t)}|\leq O\left(\frac{\alpha_{k}^{(t)}-\beta_{k,1}^{(t)}}{K}\right).
    $$
\end{proof}

\subsubsection{End of Phase I}
\begin{lemma}
    Given $k\geq 2$, \Cref{hpk1} holds for all  $t\leq T_{1,k}=O(\frac{\log(K)K^{1.98}}{\eta})$, and at iteration $t=T_{1,k}+1$, we have 
    \begin{enumerate}[label={\alph*}.]
    \item $B_{k,1}^{(T_{1,k}+1)}\leq -0.49\log(K)$;
    \item $\Attn_{1}=O\left(\frac{1}{K^{0.49}}\right)$ if $\xq=v_k$ and $\pit\in\esi$.
    \end{enumerate}
\end{lemma}
\begin{proof}
The  existence of $T_{1,k}=O(\frac{\log(K)K^{1.98}}{\eta})$ directly follows from \Cref{pk1a}.

It is easy to verify that \Cref{hpk1} holds at $t=0$.  Now we suppose \Cref{hpk1} holds for all iterations $\leq t-1$, and prove it holds at $t$. 

By \Cref{pk1a}, we have $\alpha_{k}^{(t-1)}\geq 0$. Thus $A_{k}^{(t)}= A_{k}^{(t-1)}+\eta\alpha_{k}^{(t-1)}\geq 0$. Moreover, 
combining \Cref{pk1a,pk1b},  we obtain $A_{k}^{(t)}- A_{k}^{(0)}\leq O(\frac{|B_{k,1}^{(t)}- B_{k,1}^{(0)}|}{K^{0.02}})$
which further implies $A_{k}^{(t)}\leq O(\log(K)/K^{0.02})$. 

For $m\not=1,k$, by \Cref{pk1c}, we have $$|B_{k,m}^{(t)}|\leq O(\frac{A_{k}^{(t)}- A_{k}^{(0)}+|B_{k,1}^{(t)}- B_{k,1}^{(0)}|}{K})\leq O(\log(K)/K).$$

The proof for the second statement is deferred to the next phase (\Cref{lemk2t1}).
\end{proof}

\subsection{Phase II: Switching of Leading Influence}\label{app:im:p2}
During phase I, $B_{k,1}^{(t)}$ significantly decreases, resulting in a decrease in $\Attn_{1}^{(t)}$, while other $\Attn_{n}^{(t)}$ with $n>1$ remain approximately at the order of $\Theta\left(\frac{1}{K}\right)$. By the end of phase I, $(\Attn_{1}^{(t)})^2$ decreases to $O(\frac{1}{K^{0.98}})$, leading to a decrease in $|\beta^{(t)}_{k,1}|$ as it approaches towards $\alpha_{k}^{(t)}$. At this point, phase II begins. Shortly after entering this phase, the prior dominant role of the decrease of $B_{k,1}^{(t)}$ in learning dynamics diminishes as $|\beta^{(t)}_{k,1}|$ reaches the same order of magnitude as $\alpha_{k}^{(t)}$.

For $k>1$, 
    define 
  \begin{align*}
    T_{2,k}\triangleq \max\{t>T_{1,k}: A_{k}^{(t)}- B_{k,1}^{(t)}\leq 1.01\log(K) \}.
  \end{align*}
We next state the following induction hypothesis which holds during phase II.
  \begin{hypothesis}\label{hpk2}
    For $T_{1,k}<t\leq T_{2,k}$, 
    the following holds
\begin{enumerate}[label={\alph*}.]
    \item $A_{k}^{(t)}$ is monotonically increasing and $A_{k}^{(t)}\in [0, 0.52\log(K)]$;
    \item $B_{k,1}^{(t)}$ is monotonically decreasing and $B_{k,1}^{(t)}\in [-0.51\log(K),-0.49\log(K)]$;
    \item $|B_{k,n}^{(t)}|=O(\frac{A_{k}^{(t)}+|B_{k,1}^{(t)}|}{K})$ for any $n\not=1,k$.
\end{enumerate}
\end{hypothesis}
\subsubsection{Technical Lemmas}
We first introduce several technical lemmas that will be used for the proof of \Cref{hpk2}.
\begin{lemma}\label{lemk2t1}
    Suppose \Cref{hpk2} holds at iteration $T_{1,k}<t\leq T_{2,k}$. 
    If $\xq=v_k$ and $\pit\in\esi$, the following holds
    \begin{enumerate}
    \item  $\Attn^{(t)}_k\in [\Omega\left(\frac{1}{K}\right), O(\frac{1}{K^{0.48}})]$;
    \item  $\Attn^{(t)}_1\in [\Omega\left(\frac{1}{K^{0.51}}\right), O\left(\frac{1}{K^{0.49}}\right)]$;
        \item $1-\Attn^{(t)}_{1}-\Attn^{(t)}_{k}\geq \Omega(1)$.
    \end{enumerate}
\end{lemma}
\begin{proof}
Since $\xq=v_k$, and  $|\cV_k|>0$ for $\pit\in\esi$, we have
    \begin{align*}
\Attn^{(t)}_{k}&=\frac{|\cV_k|e^{{v_k}^{\top}Q^{(t)}v_{k}}}{\sum_{j\in [N]}e^{{E^{x}_{j}}^{\top}Q^{(t)}v_{k}}}\\
&=\frac{|\cV_k|\exp(A^{(t)}_k)}{\sum_{m\not=k}|\cV_m|\exp(B^{(t)}_{k,m})+|\cV_k|\exp(A^{(t)}_k)}\\
        &=\frac{1}{\sum_{m\not=k}\frac{|\cV_m|}{|\cV_k|}\exp(B^{(t)}_{k,m}-A^{(t)}_k)+1}.
    \end{align*}
    By \Cref{hpk2},
    \begin{itemize}
        \item  for $m\not=1,k$, we have $e^{-O(\frac{\log(K)}{K})-0.52\log(K)}\leq \exp(B^{(t)}_{k,m}-A^{(t)}_k)\leq e^{O(\frac{\log(K)}{K})}$; 
        \item  for $m=1$, $e^{-1.01\log(K)}\leq \exp(B^{(t)}_{k,1}-A^{(t)}_k)\leq e^{0}$.
    \end{itemize}
Combining with the fact that $\sum_{m\not=k}\frac{|\cV_m|}{|\cV_k|}=\Theta(K)$ for $\pit\in\esi$, we obtain
     \begin{align*}
\Attn^{(t)}_{k} 
\geq \Omega\left(\frac{1}{K}\right).
    \end{align*}

Moreover, since $\frac{N-|\cV_1|}{|\cV_k|}$ is still at the order of $\Theta(K)$, we have
\begin{align*}
    &\Attn^{(t)}_{k} \leq\frac{1}{e^{-O(\frac{\log(K)}{K})-0.52\log(K)}(\frac{N-|\cV_1|}{|\cV_k|}-1)+e^{-1.01\log(K)}\frac{|\cV_1|}{|\cV_k|}+1}\leq O\left(\frac{1}{K^{0.48}}\right).
        \end{align*}
We next analyze $\Attn_1^{(t)}$ as
            \begin{align*}
                \Attn^{(t)}_{1}
                &=\frac{|\cV_1|\exp(B^{(t)}_{k,1})}{\sum_{m\not=k}|\cV_m|\exp(B^{(t)}_{k,m})+|\cV_k|\exp(A^{(t)}_k)}\\
                        &=\frac{1}{\sum_{m\not=1,k}\frac{|\cV_m|}{|\cV_k|}\exp(B^{(t)}_{k,m}-B^{(t)}_{k,1})+\frac{|\cV_k|}{|\cV_1|}\exp(A^{(t)}_{k}-B^{(t)}_{k,1})+1}.
                    \end{align*}
                    By \Cref{hpk2},
                    \begin{itemize}
                        \item for $m\not=1,k$, we have $$e^{0.49\log(K)-O(\frac{\log(K)}{K})}\leq \exp(B^{(t)}_{k,m}-B^{(t)}_{k,1})\leq e^{0.51\log(K)+O(\frac{\log(K)}{K})};$$
                        \item for $m=1$,  
                $e^{0.49\log(K)}\leq \exp(A^{(t)}_k- B^{(t)}_{k,1})\leq e^{1.01\log(K)}$.
                    \end{itemize}  Thus, we obtain
                     \begin{align*}
                \Attn^{(t)}_{1} &\geq\frac{1}{e^{0.51\log(K)+O(\frac{\log(K)}{K})}(\frac{N-|\cV_k|}{|\cV_1|}-1)+e^{1.01\log(K)+O(\frac{\log(K)}{K})}\frac{|\cV_k|}{|\cV_1|}+1}\geq 
                \Omega\left(\frac{1}{K^{0.51}}\right).
                    \end{align*}

        
                            On the other hand, 
                            \begin{align*}
        & 1-\Attn^{(t)}_{1}\geq    \frac{e^{0.49\log(K)-O(\frac{\log(K)}{K})}(\frac{N}{|\cV_1|}-1)}{e^{0.49\log(K)-O(\frac{\log(K)}{K})}(\frac{N}{|\cV_1|}-1)+1}\geq \Omega(1).
                            \end{align*}
                          
                            Thus, we obtain
                            \begin{align*}
                               1- \Attn^{(t)}_{k} -\Attn^{(t)}_{1} \geq \Omega(1).
                                    \end{align*}
\end{proof}
\begin{lemma}\label{lemk2t2}
    Suppose \Cref{hpk2} holds at iteration $T_{1,k}<t\leq T_{2,k}$.  If $\xq=v_k$ and $\pit\in\esi$,  for $n\not=1,k$, the following holds
  $$\Attn^{(t)}_n=O\left(\frac{1-\Attn^{(t)}_{k}-\Attn^{(t)}_{1}}{K}\right).$$
\end{lemma}
\begin{proof}
    Since $\xq=v_k$, we have
    \begin{align*}
\Attn^{(t)}_{n}&=\frac{|\cV_n|e^{{v_n}^{\top}Q^{(t)}v_{k}}}{\sum_{j\in [N]}e^{{E^{x}_{j}}^{\top}Q^{(t)}v_{k}}}\\&=\frac{|\cV_n|\exp(B^{(t)}_{k,n})}{\sum_{m\not=k}|\cV_m|\exp(B^{(t)}_{k,m})+|\cV_k|\exp(A^{(t)}_k)}.
    \end{align*}
    By \Cref{hpk2}, for $m,n\not=1$, $$e^{-O(\frac{\log(K)}{K})}\leq \exp(B^{(t)}_{k,m}-B^{(t)}_{k,n})\leq e^{O(\frac{\log(K)}{K})}.$$
Combining with the fact that $\frac{|\cV_m|}{|\cV_n|}=\Theta(1)$ for $\pit\in\esi$, 
    we obtain
    \begin{align*}
       \frac{\Attn^{(t)}_n}{1-\Attn^{(t)}_{k}-\Attn^{(t)}_{1}} &=\frac{|\cV_n|\exp(B^{(t)}_{k,n})}{\sum_{m\not=1,k}|\cV_m|\exp(B^{(t)}_{k,m})}\\&= \frac{1}{\sum_{m\not=k,1}\frac{|\cV_m|}{|\cV_n|}\exp(B^{(t)}_{k,m}-B^{(t)}_{k,n})
       }
       \\
       &\leq O\left(\frac{1}{K}\right).
    \end{align*}
\end{proof}

\subsubsection{Controlling Gradient Updates in Phase II}
\begin{lemma}\label{pk2a}
    Given $k>1$,  if \Cref{hpk2} holds at iteration $T_{1,k}<t \leq T_{2,k}$, then $\alpha_k^{(t)}\geq0$ and satisfies 
    \begin{align*}
        \alpha_{k}^{(t)}\geq \Omega\left(\frac{1}{K^2}\right).
    \end{align*}
\end{lemma}

\begin{proof}
    By the gradient expression in \Cref{app:lem:gd}, we have
    \begin{align*}
  \alpha_{k}^{(t)}&=
\mathbb{E}\left[\mathbf{1}\{\xq=v_k\}\Attn^{(t)}_{k }\cdot \left(
 \sum_{m\not= k}{\Attn_{m}^{(t)}}^2+(1-\Attn^{(t)}_k)^2\right)\right]\\
&=\mathbb{E}\left[\mathbf{1}\{\xq=v_k\cap\esi\}\Attn^{(t)}_{k }\cdot \left(
 \sum_{m\not= k}{\Attn_{m}^{(t)}}^2+(1-\Attn^{(t)}_k)^2\right)\right]\\
 &\quad +\mathbb{E}\left[\mathbf{1}\{\xq=v_k\cap{\esi}^c\}\Attn^{(t)}_{k }\cdot \left(
 \sum_{m\not= k}{\Attn_{m}^{(t)}}^2+(1-\Attn^{(t)}_k)^2\right)\right]\\
 &\geq p_k\cdot\mathbb{P}(P\in\esi)\mathbb{E}\left[\Attn^{(t)}_{k }\cdot \left(
 \sum_{m\not= k}{\Attn_{m}^{(t)}}^2+(1-\Attn^{(t)}_k)^2\right)\bigg|\{\xq=v_k\} \cap \esi\right]\\
&\geq \Omega\left(\frac{1}{K^2}\right),
\end{align*}
where the last inequality follows from \Cref{app:lem:prob-im},  \Cref{lemk2t1} and our choice of $p_k$.
\end{proof}
\begin{lemma}\label{pk2b}
    Given $k>1$,  if \Cref{hpk2} holds at iteration $T_{k,1} \leq t\leq T_{k,2}$, then $\beta_{k,1}^{(t)}<0$ and  satisfies 
    \begin{align*}
        |\beta_{k,1}^{(t)}|\in \left[\Omega\left(\frac{1}{K^{2.02}}\right), O\left(\frac{1}{K^{1.97}}\right)\right].
    \end{align*}
\end{lemma}
\begin{proof}
Following the computations similar to those in \Cref{pk1b}, we have
\begin{align*} &\sum_{m\not= k}{\Attn^{(t)}_{m}}^2-\Attn^{(t)}_{1} -\Attn^{(t)}_{k}(1-\Attn^{(t)}_k)\\
&\quad\leq -(1-\Attn^{(t)}_k-\Attn^{(t)}_1)(\Attn^{(t)}_{1}+\Attn^{(t)}_{k}-\max_{m\not=1,k}\Attn^{(t)}_{m}).
\end{align*}
Therefore,
\begin{align*}
&\beta_{k,1}^{(t)}\\&\leq \mathbb{E}\left[\mathbf{1}\{\xq=v_k\cap \esi\}\Attn^{(t)}_{1}\cdot \left(
  \sum_{m\not= k}{\Attn^{(t)}_{m}}^2-\Attn^{(t)}_{1} -\Attn^{(t)}_{k}(1-\Attn^{(t)}_k)\right)\right]\\
  &\quad +\mathbb{E}\left[\mathbf{1}\{\xq=v_k\cap {\esi}^{c}\}\Attn^{(t)}_{1}\cdot \left(
  \sum_{m\not= k}{\Attn^{(t)}_{m}}^2\right)\right]\\
  &\stackrel{(a)}{\leq}p_k\cdot\mathbb{P}(\pit\in\esi)\cdot\mathbb{E}\left[-\Omega(\frac{(\Attn^{(t)}_{1}+\Attn^{(t)}_{k}-\max_{m\not=1,k}\Attn^{(t)}_{m})}{K^{0.51}})\mid\{\xq=v_k\} \cap \esi\right]\\
  &\quad+p_k\cdot\mathbb{P}(\pit\in {\esi}^{c})\\
&\stackrel{(b)}{\leq} p_k\cdot \left(-\Omega\left(\frac{1}{K^{1.02}}\right)\right)+3 p_k\exp \left(-\frac{\ci^2 N}{25 K^2}\right)\\
&=-\Omega\left(\frac{1}{K^{2.02}}\right),
\end{align*}
where $(a)$ follows from \Cref{lemk2t1}, $(b)$ follows from  \Cref{lemk2t1} and \Cref{app:lem:prob-im}, 
and the last inequality holds since
\begin{align*}
\frac{1}{K^{1.02}}\gg  \exp \left(-\frac{\ci^2 N}{25 K^2}\right).
\end{align*}

Moreover,

\begin{align*}
    -\beta_{k,1}^{(t)}\leq& \mathbb{E}\left[\mathbf{1}\{\xq=v_k\cap \esi\}\Attn^{(t)}_{1}\cdot \left(
     \Attn^{(t)}_{1} +\Attn^{(t)}_{k}(1-\Attn^{(t)}_k)\right)\right]\\
      &+\mathbb{E}\left[\mathbf{1}\{\xq=v_k\cap {\esi}^{c}\}\Attn^{(t)}_{1}\cdot \left(
        \Attn^{(t)}_{1} +\Attn^{(t)}_{k}(1-\Attn^{(t)}_k)\right)\right]\\
      \stackrel{(a)}{\leq}&p_k\cdot\mathbb{P}(\pit\in\esi)\cdot\mathbb{E}\left[\Attn^{(t)}_{1}\cdot O(\Attn^{(t)}_{1}+\Attn^{(t)}_{k})\mid\{\xq=v_k\} \cap \esi\right]\\
      &+2p_k\cdot\mathbb{P}(\pit\in{\esi}^{c})\\
    \stackrel{(b)}{\leq} &p_k\cdot \left(O\left(\frac{1}{K^{0.97}}\right)\right)+6 p_k\exp \left(-\frac{\ci^2 N}{25 K^2}\right)\\
    =&O\left(\frac{1}{K^{1.97}}\right),
    \end{align*}
  where $(a)$ follows because $\Attn^{(t)}_{1} +\Attn^{(t)}_{k}(1-\Attn^{(t)}_k)$ is upper-bounded by $2$ on the event $\{\pit\in{\esi}^c\}$, and $(b)$ follows from  \Cref{lemk2t1}.
\end{proof}
\begin{lemma}\label{pk2c}
    If \Cref{hpk2} holds at iteration $T_{1,k}<t \leq T_{2,k}$,  for any $n\not=1,k$, $\beta_{k,n}^{(t)}$ satisfies 
    \begin{align*}
        |\beta_{k,n}^{(t)}|\leq O\left(\frac{\alpha_{k}^{(t)}-\beta_{k,1}^{(t)}}{K}\right).
    \end{align*}
\end{lemma}
\begin{proof}
    By the gradient expression in  \Cref{app:lem:gd}, we have 
    \begin{align}
        \beta_{k,n}^{(t)}&\leq \mathbb{E}\left[\mathbf{1}\{\xq=v_k\}\Attn^{(t)}_{n}\cdot \left(
 \sum_{m\not= k}{\Attn^{(t)}_{m}}^2\right)\right],\label{k2I1}\\
 -\beta_{k,n}^{(t)}&\leq \mathbb{E}\left[\mathbf{1}\{\xq=v_k\}\Attn^{(t)}_{n}\cdot \left(
 \Attn^{(t)}_{n} +\Attn^{(t)}_{k}(1-\Attn^{(t)}_k)\right)\right].\label{k2I2}
    \end{align}
   To further bound \cref{k2I1},  we have
    \begin{align*}
      \beta_{k,n}^{(t)}&\leq \mathbb{E}\left[\mathbf{1}\{\xq=v_k\cap \pit\in\esi\}\Attn^{(t)}_{n}\cdot \left(
        \sum_{m\not= k}{\Attn^{(t)}_{m}}^2\right)\right]\\
        &\quad +\mathbb{E}\left[\mathbf{1}\{\xq=v_k\cap \pit\in{\esi}^{c}\}\Attn^{(t)}_{n}\cdot \left(
        \sum_{m\not= k}{\Attn^{(t)}_{m}}^2\right)\right]\\
        &
        {\leq}p_k\cdot\mathbb{P}(\pit\in\esi)\cdot\mathbb{E}\left[\Attn^{(t)}_{n}\cdot \left(
{\Attn^{(t)}_{1}}^2+O\left(\frac{1}{K}\right)\right)\bigg|\{\xq=v_k\} \cap \esi\right]\\
&\quad+p_k\cdot\mathbb{P}(\pit\in{\esi}^{c})\\
 &
 {\leq} 
 O\left(\frac{1}{K^3}\right)+O\left(\frac{|\beta_{k,1}^{(t)}|}{K}\right)+3 p_k\exp \left(-\frac{\ci^2 N}{25 K^2}\right)\\
 &{\leq}O\left(\frac{1}{K^3}\right)+O\left(\frac{|\beta_{k,1}^{(t)}|}{K}\right).
    \end{align*}

    To further bound \cref{k2I2}, we have 
    \begin{align*}
      -&\beta_{k,n}^{(t)}\\
      &\leq p_k\mathbb{E}\left[\Attn^{(t)}_{n}\cdot \left(
 \Attn^{(t)}_{n} +\Attn^{(t)}_{k}(1-\Attn^{(t)}_k)\right)\mid\{\xq=v_k\} \cap \esi\right]+2p_k\cdot\mathbb{P}({\esi}^{c})\\
 &{=} 2p_k\cdot\mathbb{P}({\esi}^{c})+p_k\cdot\mathbb{P}({\esi})\mathbb{E}\left[O\left(\frac{1-\Attn^{(t)}_{k}}{K}\right)\right.\\
&\qquad \qquad \qquad\qquad\left.\cdot \left( O\left(\frac{1-\Attn^{(t)}_{k}}{K}\right)+
\Attn^{(t)}_{k}(1-\Attn^{(t)}_k)\right)\bigg|\{\xq=v_k\} \cap \esi\right]\\
 &
 {\leq}p_k\cdot\mathbb{P}({\esi})\mathbb{E}\left[O\left(\frac{\Attn^{(t)}_{k}(1-\Attn^{(t)}_{k})^2}{K}\right)\bigg|\{\xq=v_k\} \cap \esi\right]+6p_k\exp \left(-\frac{\ci^2 N}{25 K^2}\right)\\
 &\leq O(\frac{\alpha_{k}^{(t)}}{K}).
    \end{align*}
  Following from the analysis in \Cref{pk2a}, we have 
    \begin{align*}
       \alpha_{k}^{(t)}&\geq \Omega\left(\frac{1}{K^2}\right)\gg 6\exp \left(-\frac{\ci^2 N}{25 K^2}\right).
    \end{align*}
    Thus, we obtain
    $$
|\beta_{k,n}^{(t)}|\leq O\left(\frac{\alpha_{k}^{(t)}-\beta_{k,1}^{(t)}}{K}\right).
    $$
\end{proof}

\subsubsection{End of Phase II}
\begin{lemma}
    Given $k\geq 2$, \Cref{hpk2} holds for all  $T_{1,k}<t\leq T_{2,k}= T_{1,k}+O(\frac{\log(K)K^{2}}{\eta})$, and at iteration $t=T_{2,k}+1$, we have 
    \begin{enumerate}[label={\alph*}.]
    \item $A_{k}^{(T_{2,k}+1)}\geq 0.5\log(K)$;
    \item $B_{k}^{(T_{2,k}+1)}\geq -0.51\log(K)$.
    \end{enumerate}
\end{lemma}
\begin{proof}
The  existence of $T_{2,k}=T_{1,k}+O(\frac{\log(K)K^{2}}{\eta})$ directly follows from \Cref{pk2a,pk2b}.

It is easy to verify that \Cref{hpk2} holds at $T_{1,k}+1$.  Now we suppose \Cref{hpk2} holds for all iterations $\leq t-1$, and prove that it holds at $t$. 


For $m\not=1,k$, by \Cref{pk2c}, we have $$|B_{k,m}^{(t)}|\leq |B_{k,m}^{(T_{1,k}+1)}|+ O(\frac{A_{k}^{(T_{2,k})}- A_{k}^{(T_{1,k}+1)}+|B_{k,1}^{(T_{2,k})}- B_{k,1}^{(T_{1,k}+1)}|}{K})\leq O(\log(K)/K).$$

Now suppose $A_{k}^{(T_{2,k}+1)}<0.5\log(K)$, then $B_{k,1}^{(T_{2,k}+1)}<-0.51\log(K)$. Denote the first time that $B_{k,1}^{(t)}$ reaches $-0.501\log(K)$ as $\tilde{T}$. Note that $\tilde{T}<T_{2,k}^{(t)}$ since $\beta_{k,1}^{(t)}$, the change  of $B_{k,1}^{(t)}$, satisfies $|\beta_{k,1}^{(t)}|\ll \log(K)$. Then for $t\geq \tilde{T}$, if $\xq=v_k$ and $\pit\in\esi$, the following holds:
    \begin{enumerate}
    \item  $\Attn^{(t)}_k\in [\Omega\left(\frac{1}{K}\right), O(\frac{1}{K^{0.5}})]$;
    \item  $\Attn^{(t)}_1\leq O(\frac{1}{K^{0.501}})$.
    \end{enumerate}
Therefore, following the  analysis similar to those for  \Cref{pk2b}, we have
\begin{align*}
    |\beta_{k,1}^{(t)}|&\leq \mathbb{E}\left[\mathbf{1}\{\xq=v_k\cap \esi\}\Attn^{(t)}_{1}\cdot \left(
     \Attn^{(t)}_{1} +\Attn^{(t)}_{k}(1-\Attn^{(t)}_k)\right)\right]\\
      &\quad +\mathbb{E}\left[\mathbf{1}\{\xq=v_k\cap {\esi}^{c}\}\Attn^{(t)}_{1}\cdot \left(
        \Attn^{(t)}_{1} +\Attn^{(t)}_{k}(1-\Attn^{(t)}_k)\right)\right]\\
      &{\leq}p_k\cdot\mathbb{P}(P\in\esi)\cdot\mathbb{E}\left[\Attn^{(t)}_{1}\cdot O(\Attn^{(t)}_{1}+\Attn^{(t)}_{k})\mid\{\xq=v_k\} \cap \esi\right]\\
      &\quad +2p_k\cdot\mathbb{P}({\esi}^{c})\\
    &{\leq} p_k\cdot \left(O\left(\frac{1}{K^{1.02}}\right)\right)+O\left(\frac{\alpha_{k}^{(t)}}{K^{0.501}}\right)+6 p_k\exp \left(-\frac{\ci^2 N}{25 K^2}\right)\\
    &\leq O\left(\frac{\alpha_{k}^{(t)}}{K^{0.01}}\right),
    \end{align*}
where the last inequality follows from \Cref{pk2a}.

Since $|B_{k,1}^{(T_{2,k}+1)}-B_{k,1}^{(\tilde{T})}|\geq \Omega(\log(K))$, we have  $$A_{k}^{(T_{2,k}+1)}\geq  |B_{k,1}^{(T_{2,k}+1)}-B_{k,1}^{(\tilde{T})}|\cdot \Omega(K^{0.01})+A_{k}^{(\tilde{T})}\gg \Omega(K^{0.01}\log(K)),$$
which contradicts the assumption that $A_{k}^{(T_{2,k}+1)}<0.5\log(K)$. Therefore, $A_{k}^{(T_{2,k}+1)}\geq 0.5\log(K)$. Noting that once $B_{k,1}^{(t)}$ drops below  $-0.501\log(K)$, it will change much smaller compared to the increase of $A_{k}^{(t)}$. Thus, $B_{k,1}^{(T_{2,k}+1)}\geq -0.51\log(K)$.

\end{proof}

\subsection{Phase III: Growth of Target Feature}\label{app:im:p3}
After the transition phase, $A_{k}^{(t)}$ will experience a larger gradient, with the growth of $A_{k}^{(t)}$ becoming the dominant effect in this phase. For the $k$-th feature $v_k$, we define phase III as all iterations $T_{2,k}<t \leq T_{3,k}$, where
$$
T_{3,k} \triangleq \max \left\{t>T_{2,k}: A_{k}^{(t)} \leq \log(K)\right\}.
$$
We state the following induction hypothesis, which will hold throughout phase III.
\begin{hypothesis}\label{hpk3}
For each $T_{2,k}<t \leq T_{3,k}$, the following holds:
\begin{enumerate}[label={\alph*}.]
    \item $A_{k}^{(t)}$ is monotonically increasing and $A_{k}^{(t)}\in [0.5\log(K),\log(K)]$;
    \item $B_{k,1}^{(t)}$ is monotonically decreasing and  $B_{k,1}^{(t)}\in [-0.51\log(K)-O(\frac{\log(K)}{K^{0.48}}),-0.49\log(K) ]$;
    \item $|B_{k,n}^{(t)}|=O(\frac{A_{k}^{(t)}+|B_{k,1}^{(t)}|}{K})$ for any $n\not=1,k$.  
\end{enumerate}
\end{hypothesis} 
 
\subsubsection{Technical Lemmas}
We first introduce several useful technical lemmas.
\begin{lemma}\label{lemk3t1}
    Suppose \Cref{hpk3} holds at iteration $T_{k,2} < t\leq T_{k,3}$. 
    If $\xq=v_k$ and $\pit\in\esi$, then the following holds
    \begin{enumerate}
    \item  $\Attn^{(t)}_k=\Omega\left(\frac{1}{K^{0.5}}\right)$;
    \item $\Attn^{(t)}_1 \in [\Omega\left(\frac{1}{K^{0.51}}\right), O\left(\frac{1}{K^{0.49}}\right)]$;
        \item $1-\Attn^{(t)}_{k}\geq \Omega(1)$.
    \end{enumerate}
\end{lemma}
\begin{proof}
Since $\xq=v_k$, and  $|\cV_k|>0$ for $\pit\in\esi$, we have
    \begin{align*}
\Attn^{(t)}_{k}&=\frac{|\cV_k|e^{{v_k}^{\top}Q^{(t)}v_{k}}}{\sum_{j\in [N]}e^{{E^{x}_{j}}^{\top}Q^{(t)}v_{k}}}\\
&=\frac{|\cV_k|\exp(A^{(t)}_k)}{\sum_{m\not=k}|\cV_m|\exp(B^{(t)}_{k,m})+|\cV_k|\exp(A^{(t)}_k)}\\
        &=\frac{1}{\sum_{m\not=k}\frac{|\cV_m|}{|\cV_k|}\exp(B^{(t)}_{k,m}-A^{(t)}_k)+1}.
    \end{align*}
    By \Cref{hpk3}, we have
    \begin{itemize}
        \item   for $m\not = 1$, $e^{-\left(\log(K)+O(\frac{\log(K)}{K})\right)}\leq \exp(B^{(t)}_{k,m}-A^{(t)}_k)\leq e^{O(\frac{\log(K)}{K})-0.5\log(K)}$;
        \item  $e^{-\left(1.51\log(K)+O(\frac{\log(K)}{K})\right)}\leq \exp(B^{(t)}_{k,1}-A^{(t)}_k)\leq e^{-1.01\log(K)}$.
    \end{itemize}
Thus,
     \begin{align*}
\Attn^{(t)}_{k} &\geq\frac{1}{e^{O(\frac{\log(K)}{K})-0.5\log(K)}(\frac{N-|\cV_1|}{|\cV_k|}-1)+ e^{-1.01\log(K)}\frac{|\cV_1|}{|\cV_k|}+ 1}\geq \Omega\left(\frac{1}{K^{0.5}}\right),
    \end{align*}
    where the second inequality follows from the fact that $\pit\in\esi$.

    On the other hand, 
    \begin{align*}
\Attn^{(t)}_{k} &\leq\frac{1}{e^{-\left(\log(K)+O(\frac{\log(K)}{K})\right)}(\frac{N-|\cV_1|}{|\cV_k|}-1)+1}\leq \frac{1}{e^{-1}(\frac{1}{\Ui_k}-\frac{1}{K})+1}.
    \end{align*}
    Thus,
    \begin{align*}
       1- \Attn^{(t)}_{k} \geq \frac{e^{-\left(\log(K)+O(\frac{\log(K)}{K})\right)}(\frac{N-|\cV_1|}{|\cV_k|}-1)+e^{-1.01\log(K)}\frac{|\cV_1|}{|\cV_k|}}{e^{-\left(\log(K)+O(\frac{\log(K)}{K})\right)}(\frac{N-|\cV_1|}{|\cV_k|}-1)+e^{-1.01\log(K)}\frac{|\cV_1|}{|\cV_k|}+1}\geq \Omega(1).
            \end{align*}

            We next analyze $\Attn_1^{(t)}$ as follows.
            \begin{align*}
                \Attn^{(t)}_{1}
                &=\frac{|\cV_1|\exp(B^{(t)}_{k,1})}{\sum_{m\not=k}|\cV_m|\exp(B^{(t)}_{k,m})+|\cV_k|\exp(A^{(t)}_k)}\\
                        &=\frac{1}{\sum_{m\not=1,k}\frac{|\cV_m|}{|\cV_k|}\exp(B^{(t)}_{k,m}-B^{(t)}_{k,1})+\frac{|\cV_k|}{|\cV_1|}\exp(A^{(t)}_{k}-B^{(t)}_{k,1})+1}.
                    \end{align*}
                    By \Cref{hpk3}, 
                    \begin{itemize}
                        \item for $m\not=1,k$, we have $$e^{0.49\log(K)-O(\frac{\log(K)}{K})}\leq \exp(B^{(t)}_{k,m}-B^{(t)}_{k,1})\leq e^{0.51\log(K)+O(\frac{\log(K)}{K})};$$
                        \item for $m=1$, $e^{1.01\log(K)}\leq \exp(A^{(t)}_k- B^{(t)}_{k,1})\leq e^{1.51\log(K)+O(\frac{\log(K)}{K})}$.
                    \end{itemize} 
                      Thus,
                     \begin{align*}
                \Attn^{(t)}_{1} &\leq\frac{1}{e^{0.49\log(K)-O(\frac{\log(K)}{K})}(\frac{N-|\cV_k|}{|\cV_1|}-1)+e^{1.01\log(K)}\frac{|\cV_k|}{|\cV_1|}+1}\leq 
                O\left(\frac{1}{K^{0.49}}\right).\\
                \Attn^{(t)}_{1} &\geq\frac{1}{e^{0}(\frac{N-|\cV_k|}{|\cV_1|}-1)+e^{1.51\log(K)+O(\frac{\log(K)}{K})}\frac{|\cV_k|}{|\cV_1|}+1}\geq 
                \Omega\left(\frac{1}{K^{0.51}}\right).
                    \end{align*}
\end{proof}
\begin{lemma}\label{lemk3t2}
     Suppose \Cref{hpk3} holds at iteration $T_{2,k}<t\leq T_{3,k}$.  If $\xq=v_k$ and $\pit\in\esi$,  for $n\not=1,k$, then the following holds
$$\Attn^{(t)}_n=O\left(\frac{1-\Attn^{(t)}_{k}-\Attn^{(t)}_{1}}{K}\right).$$
\end{lemma}
\begin{proof}
    By \Cref{hpk3}, we have $$e^{-O(\frac{\log(K)}{K})}\leq \exp(B^{(t)}_{k,m}-B^{(t)}_{k,n})\leq e^{O(\frac{\log(K)}{K})}.$$ 
    Combining with the fact that $\frac{|\cV_m|}{|\cV_n|}=\Theta(1)$ when $\pit\in\esi$, 
    we have
    \begin{align*}
       \frac{\Attn^{(t)}_n}{1-\Attn^{(t)}_{k}-\Attn^{(t)}_{1}} =\frac{|\cV_n|\exp(B^{(t)}_{k,n})}{\sum_{m\not=1,k}|\cV_m|\exp(B^{(t)}_{k,m})}= \frac{1}{\sum_{m\not=1,k}\frac{|\cV_m|}{|\cV_n|}\exp(B^{(t)}_{k,m}-B^{(t)}_{k,n})}\leq O\left(\frac{1}{K}\right).
    \end{align*}
\end{proof}
\subsubsection{Controlling  Gradient Updates in Phase III}
\begin{lemma}\label{pk3a}
    At each iteration $T_{2,k}<t\leq T_{3,k}$, if \Cref{hpk3} holds, then $\alpha_k^{(t)}\geq0$ and satisfies 
    \begin{align*}
        \alpha_{k}^{(t)}\geq \Omega\left(\frac{1}{K^{1.5}}\right).
    \end{align*}
\end{lemma}
\begin{proof}
    By the gradient expression in \Cref{app:lem:gd}, we have
    \begin{align*}
  \alpha_{k}^{(t)}&=
\mathbb{E}\left[\mathbf{1}\{\xq=v_k\}\Attn^{(t)}_{k }\cdot \left(
 \sum_{m\not= k}{\Attn_{m}^{(t)}}^2+(1-\Attn^{(t)}_k)^2\right)\right]\\
&=\mathbb{E}\left[\mathbf{1}\{\xq=v_k\cap\esi\}\Attn^{(t)}_{k }\cdot \left(
 \sum_{m\not= k}{\Attn_{m}^{(t)}}^2+(1-\Attn^{(t)}_k)^2\right)\right]\\
 &\quad+\mathbb{E}\left[\mathbf{1}\{\xq=v_k\cap{\esi}^c\}\Attn^{(t)}_{k }\cdot \left(
 \sum_{m\not= k}{\Attn_{m}^{(t)}}^2+(1-\Attn^{(t)}_k)^2\right)\right]\\
 &\geq p_k\cdot\mathbb{P}(\pit\in\esi)\mathbb{E}\left[\Attn^{(t)}_{k }\cdot \left(
 \sum_{m\not= k}{\Attn_{m}^{(t)}}^2+(1-\Attn^{(t)}_k)^2\right)\bigg|\{\xq=v_k\} \cap \esi\right]\\
&\geq p_k\cdot\mathbb{P}(\pit\in\esi)\mathbb{E}\left[\Attn^{(t)}_{k }\cdot (1-\Attn^{(t)}_k)^2\mid\{\xq=v_k\} \cap \esi\right]\\
&\geq \Omega\left(\frac{1}{K^{1.5}}\right)
\end{align*}
where the last inequality follows from \Cref{app:lem:prob-im},  \Cref{lemk3t1} and our choice of $p_k$.
\end{proof}
\begin{lemma}\label{pk3b}
    Given $k>1$,  if \Cref{hpk3} holds at iteration $T_{k,2} \leq t\leq T_{k,3}$, then $\beta_{k,1}^{(t)}<0$  satisfies 
    \begin{align*}
        |\beta_{k,1}^{(t)}|\leq  \left[\Omega\left(\frac{1}{K^{2.01}}\right), O\left(\frac{\alpha^{(t)}_k}{K^{0.48}}\right)\right].
    \end{align*}
\end{lemma}
\begin{proof}
  Following the computations similar to those for  \Cref{pk1b}, we have
\begin{align*} &\sum_{m\not= k}{\Attn^{(t)}_{m}}^2-\Attn^{(t)}_{1} -\Attn^{(t)}_{k}(1-\Attn^{(t)}_k)\\
&\quad \leq -(1-\Attn^{(t)}_k-\Attn^{(t)}_1)(\Attn^{(t)}_{1}+\Attn^{(t)}_{k}-\max_{m\not=1,k}\Attn^{(t)}_{m}).
\end{align*}
Therefore,
\begin{align*}
\beta_{k,1}^{(t)}&\leq \mathbb{E}\left[\mathbf{1}\{\xq=v_k\cap \esi\}\Attn^{(t)}_{1}\cdot \left(
  \sum_{m\not= k}{\Attn^{(t)}_{m}}^2-\Attn^{(t)}_{1} -\Attn^{(t)}_{k}(1-\Attn^{(t)}_k)\right)\right]\\
  &\quad +\mathbb{E}\left[\mathbf{1}\{\xq=v_k\cap {\esi}^{c}\}\Attn^{(t)}_{1}\cdot \left(
  \sum_{m\not= k}{\Attn^{(t)}_{m}}^2\right)\right]\\
  &\stackrel{(a)}{\leq} p_k\cdot\mathbb{P}(\pit\in{\esi}^{c})+ 
  p_k\cdot\mathbb{P}(\pit\in\esi)\\
  &\qquad\qquad\qquad\cdot\mathbb{E}\left[-\Omega(\frac{(\Attn^{(t)}_{1}+\Attn^{(t)}_{k}-\max_{m\not=1,k}\Attn^{(t)}_{m})}{K^{0.51}})\bigg|\{\xq=v_k\} \cap \esi\right]\\
&\stackrel{(b)}{\leq} p_k\cdot \left(-\Omega\left(\frac{1}{K^{1.01}}\right)\right)+3 p_k\exp \left(-\frac{\ci^2 N}{25 K^2}\right)\\
&=-\Omega\left(\frac{1}{K^{2.01}}\right),
\end{align*}
where both $(a)$ and $(b)$ follow from \Cref{lemk3t1},
and the last inequality holds since 
\begin{align*}
\frac{1}{K^{1.01}}\gg  \exp \left(-\frac{\ci^2 N}{25 K^2}\right).
\end{align*}

Moreover, we have 

\begin{align*}
    -\beta_{k,1}^{(t)}&\leq \mathbb{E}\left[\mathbf{1}\{\xq=v_1\cap \pit\in\esi\}\Attn^{(t)}_{1}\cdot \left(
     \Attn_{1} +\Attn_{k}(1-\Attn_k)\right)\right]\\
      &\quad +\mathbb{E}\left[\mathbf{1}\{\xq=v_k\cap \pit\in{\esi}^{c}\}\Attn^{(t)}_{1}\cdot \left(
        \Attn_{1} +\Attn_{k}(1-\Attn_k)\right)\right]\\
      &{\leq}p_k\cdot\mathbb{P}(\pit\in\esi)\cdot\mathbb{E}\left[\Attn_{1}\cdot O(\Attn_{1}+\Attn_{k})\mid\{\xq=v_k\} \cap \esi\right]\\
      &\quad +2p_k\cdot\mathbb{P}(\pit\in{\esi}^{c})\\
    &{\leq} p_k\cdot \left(O\left(\frac{1}{K^{0.98}}\right)\right)+ O\left(\frac{\alpha^{(t)}_k}{K^{0.49}}\right)+ 6 p_k\exp \left(-\frac{\ci^2 N}{25 K^2}\right)\\
    &\leq O\left(\frac{\alpha^{(t)}_k}{K^{0.48}}\right)
    \end{align*}
    where the last inequality follows from \Cref{pk3a}.
\end{proof}
\begin{lemma}\label{pk3c}
    If \Cref{hpk3} holds at iteration $T_{2,k}<t \leq T_{3,k}$,  then for any $n\not=1,k$, $\beta_{k,n}^{(t)}$ satisfies 
    \begin{align*}
        |\beta_{k,n}^{(t)}|\leq O\left(\frac{\alpha_{k}^{(t)}-\beta_{k,1}^{(t)}}{K}\right).
    \end{align*}
\end{lemma}
\begin{proof}
    By the gradient computation in \Cref{app:lem:gd}, we have 
    \begin{align}
        \beta_{k,n}^{(t)}&\leq \mathbb{E}\left[\mathbf{1}\{\xq=v_k\}\Attn^{(t)}_{n}\cdot \left(
 \sum_{m\not= k}{\Attn^{(t)}_{m}}^2\right)\right],\label{k3I1}\\
 -\beta_{k,n}^{(t)}&\leq \mathbb{E}\left[\mathbf{1}\{\xq=v_k\}\Attn^{(t)}_{n}\cdot \left(
 \Attn^{(t)}_{n} +\Attn^{(t)}_{k}(1-\Attn^{(t)}_k)\right)\right].\label{k3I2}
    \end{align}
    We further bound \cref{k3I1} as
    \begin{align*}
      \beta_{k,n}^{(t)}&\leq \mathbb{E}\left[\mathbf{1}\{\xq=v_k\cap\pit\in\esi\}\Attn^{(t)}_{n}\cdot \left(
        \sum_{m\not= k}{\Attn^{(t)}_{m}}^2\right)\right]\\
        &\quad+\mathbb{E}\left[\mathbf{1}\{\xq=v_k\cap \pit\in {\esi}^{c}\}\Attn^{(t)}_{n}\cdot \left(
        \sum_{m\not= k}{\Attn^{(t)}_{m}}^2\right)\right]\\
        &
        {\leq}p_k\cdot\mathbb{P}(\pit\in\esi)\cdot\mathbb{E}\left[\Attn^{(t)}_{n}\cdot \left(
{\Attn^{(t)}_{1}}^2+O\left(\frac{1}{K}\right)\right)\bigg|\{\xq=v_k\} \cap \pit\in\esi\right]\\
&\quad+p_k\cdot\mathbb{P}(\pit\in {\esi}^{c})\\
 &
 {\leq} 
 O\left(\frac{1}{K^3}\right)+O\left(\frac{|\beta_{k,1}^{(t)}|}{K}\right)+6 p_k\exp \left(-\frac{\ci^2 N}{25 K^2}\right)\\
 &{\leq}O\left(\frac{1}{K^3}\right)+O\left(\frac{|\beta_{k,1}^{(t)}|}{K}\right).
    \end{align*}

We then further bound \cref{k3I2} as
    \begin{align*}
      -\beta_{k,n}^{(t)}&\leq p_k\mathbb{E}\left[\Attn^{(t)}_{n}\cdot \left(
 \Attn^{(t)}_{n} +\Attn^{(t)}_{k}(1-\Attn^{(t)}_k)\right)\mid\{\xq=v_k\} \cap \pit\in\esi\right]\\
 &\quad+p_k\cdot\mathbb{P}(\pit\in {\esi}^{c})\\
 &
 {=} p_k\cdot\mathbb{P}(\pit\in{\esi}^{c}) +p_k\cdot\mathbb{P}(\pit\in\esi)\mathbb{E}\left[O\left(\frac{1-\Attn^{(t)}_{k}}{K}\right)\right.\\&\left.\qquad\qquad \cdot \left( O\left(\frac{1-\Attn^{(t)}_{k}}{K}\right)+
\Attn^{(t)}_{k}(1-\Attn^{(t)}_k)\right)\bigg|\{\xq=v_k\} \cap \esi\right]\\
 &
 {\leq}p_k\cdot\mathbb{P}({\esi})\mathbb{E}\left[O\left(\frac{\Attn^{(t)}_{k}(1-\Attn^{(t)}_{k})^2}{K}\right)\bigg|\{\xq=v_k\} \cap \esi\right]+6p_k\exp \left(-\frac{\ci^2 N}{25 K^2}\right)\\
 &\leq O\left(\frac{\alpha_{k}^{(t)}}{K}\right).
    \end{align*}
  Following from the analysis in \Cref{pk3a}, we have 
    \begin{align*}
       \alpha_{k}^{(t)}&\geq \Omega\left(\frac{1}{K^{1.5}}\right).
    \end{align*}
    Thus, we obtain
    $$
|\beta_{k,n}^{(t)}|\leq O\left(\frac{\alpha_{k}^{(t)}-\beta_{k,1}^{(t)}}{K}\right).
    $$
\end{proof}
\subsubsection{End of Phase III}
\begin{lemma}
    Given $k>1$, \Cref{hpk3} holds for all  $T_{2,k}<t\leq T_{3,k}=T_{2,k}+O(\frac{\log(K)K^{1.5}}{\eta})$, and at iteration $t=T_{3,k}+1$, we have 
    \begin{enumerate}[label={\alph*}.]
    \item $A_{k}^{(T_{3,k}+1)}\geq \log(K)$;
    \item $\Attn_{k}=\Omega(1)$ if $\xq=v_k$ and $\pit\in\esi$.
    \end{enumerate}
\end{lemma}
\begin{proof}
The  existence of $T_{3,k}=T_{2,k}+O(\frac{\log(K)K^{1.5}}{\eta})$ directly follows from  \Cref{pk3a}.

It is easy to verify \Cref{hpk3} holds at $t=T_{2,k}+1$.  Now we suppose \Cref{hpk3} holds for all iterations $\leq t-1$, and prove it holds at $t$. 

By \Cref{pk3a}, we have $\alpha_{k}^{(t-1)}\geq 0$. Thus $A_{k}^{(t)}= A_{k}^{(t-1)}+\eta\alpha_{k}^{(t-1)}\geq 0.5\log(K)$. Morover, 
by \Cref{pk3b}, we have $|B_{k,1}^{(t)}- B_{k,1}^{(T_{2,k}+1)}|\leq O(\frac{A_{k}^{(t)}- A_{k}^{(T_{2,k}+1)}}{K^{0.48}})$
which immediately implies that $$B_{k,1}^{(t)}\geq -O(\log(K)/K^{0.48})-0.51\log(K).$$

For $m\not=1,k$, by \Cref{pk3c}, we have $$|B_{k,m}^{(t)}|\leq O(\frac{A_{k}^{(t)}- A_{k}^{(T_{2,k}+1)}+|B_{k,1}^{(t)}- B_{k,1}^{(T_{2,k}+1)}|}{K})\leq O(\log(K)/K).$$

The proof for the second statement is deferred to the next phase (\Cref{lemk4t1}).
\end{proof}

\subsection{Phase IV: Convergence}\label{app:im:p4}
At $t=T_{3,k}+1$, the desired attention structure for the query token associated with feature $v_k$ has already been achieved. In this final phase, we establish that these structures, including each under-represented feature, indeed represent the solutions toward which the algorithm converges. 

Given any  $0<\epsilon <1$, 
    for $k\geq 2$, 
    define 
  \begin{align*}
    T^{\epsilon}_{4,k}\triangleq \max\left\{t>T_{3,k}: A_{k}^{(t)}\leq \log\left(\left(\frac{e(1-\Li_1)K+\Ui_1K^{0.51}}{\Li_k}-1\right)\left(\left(\frac{3}{\epsilon}\right)^{\frac{1}{2}}-1\right)\right) \right\}.
  \end{align*}
\begin{hypothesis}\label{hpk4}
    For $T_{3,k}<t\leq T^{\epsilon}_{4,k}$, suppose $\operatorname{polylog}(K)\gg \log(\frac{1}{\epsilon})$. Then the following holds.
\begin{enumerate}[label={\alph*}.]
    \item $A_{k}^{(t)}$ is monotonically increasing and $A_{k}^{(t)}\in [\log(K), O(\log(K/\epsilon))]$;
   \item $B_{k,1}^{(t)}$ is monotonically decreasing and $$B_{k,1}^{(t)}\in \left[-0.51\log(K)-O\left(\frac{\log(K)}{K^{0.48}}\right),-0.49\log(K)\right]$$
    \item $B_{k,n}^{(t)}$ is monotonically decreasing and $|B_{k,n}^{(t)}|=O(\frac{\log(K/\epsilon)}{K})$ for any $n\not=1,k$.
\end{enumerate}
\end{hypothesis}
\subsubsection{Technical Lemmas}
We first introduce several useful technical lemmas.
\begin{lemma}\label{lemk4t1}
    Suppose \Cref{hpk4} holds at iteration $T_{3,k}<t\leq T_{4,k}^{\epsilon}$. 
    If $\xq=v_k$ and $\pit\in\esi$, then the following holds.
    \begin{enumerate}
    \item  $\Attn^{(t)}_k=\Omega(1)$;
        \item $(1-\Attn^{(t)}_{k})^2\geq \Omega(\epsilon)=\Omega(
            \exp\left(-\operatorname{polylog}(K)\right))$.
    \end{enumerate}
\end{lemma}
\begin{proof}
    Since $\xq=v_k$,  we have
    \begin{align*}
\Attn^{(t)}_{k}&=\frac{|\cV_k|\exp(A^{(t)}_k)}{\sum_{m\not=k}|\cV_m|\exp(B^{(t)}_{k,m})+|\cV_k|\exp(A^{(t)}_k)}\\
        &=\frac{1}{\sum_{m\not=k}\frac{|\cV_m|}{|\cV_k|}\exp(B^{(t)}_{k,m}-A^{(t)}_k)+1}.
    \end{align*}
    By \Cref{hpk4}, we have 
    \begin{itemize}
        \item   for $m\not=1,k$: $$\exp(B^{(t)}_{k,m}-A^{(t)}_k)\leq e^{O(\frac{\log(K/\epsilon)}{K})-\log(K)}\leq e^{O(\frac{\log(K)+\operatorname{polylog}(K)}{K})-\log(K)}\leq O\left(\frac{1}{K}\right).$$ 
        \item for $m=1$, $\exp(B^{(t)}_{k,1}-A^{(t)}_k)\leq  O(\frac{1}{K^{1.49}}).$
    \end{itemize}
  
Therefore,
    \begin{align*}
\Attn^{(t)}_{k} &\geq\frac{1}{O\left(\frac{1}{K}\right)(\frac{N-|\cV_1|}{|\cV_k|}-1)+O(\frac{1}{K^{1.49}})\frac{|\cV_1|}{|\cV_k|}+1}\geq \Omega(1).
    \end{align*}
   On the other hand, 
   we have
   \begin{align*}
   1-\Attn^{(t)}_{k}
            &=\frac{\sum_{m\not=k}\frac{|\cV_m|}{|\cV_k|}\exp(B^{(t)}_{k,m}-A^{(t)}_k)}{\sum_{m\not=k}\frac{|\cV_m|}{|\cV_k|}\exp(B^{(t)}_{k,m}-A^{(t)}_k)+1}\\
            &\overset{(a)}{\geq} \frac{ \exp( \min_{m\not= 1,k} B^{(t)}_{k,m}-A^{(t)}_k) (\frac{N-|\cV_1|}{|\cV_k|}-1)+\exp(B^{(t)}_{k,1}-A^{(t)}_k)\frac{|\cV_1|}{|\cV_k|} }{\exp( \min_{m\not= 1, k} B^{(t)}_{k,m}-A^{(t)}_k) (\frac{N-|\cV_1|}{|\cV_k|}-1)+\exp(B^{(t)}_{k,1}-A^{(t)}_k)\frac{|\cV_1|}{|\cV_k|}+1}\\
            &\geq \frac{ \exp( \min_{m\not= k} B^{(t)}_{k,m}-A^{(t)}_k) (\frac{(1-\Ui_{1})K}{\Ui_{k}}-1)+ \exp(B^{(t)}_{k,1}-A^{(t)}_k)\cdot\frac{\Li_{1}}{\Ui_{k}} }{\exp( \min_{m\not= k} B^{(t)}_{k,m}-A^{(t)}_k) (\frac{(1-\Ui_{1})K}{\Ui_{k}}-1)+\exp(B^{(t)}_{k,1}-A^{(t)}_k)\cdot\frac{\Li_{1}}{\Ui_{k}}+1}\\
            &= \frac{(\frac{(1-\Ui_{1})K}{\Ui_{k}}-1+\frac{\Li_{1}K^{0.49}}{\Ui_{k}}) \exp(-A^{(t)}_k) }{(\frac{(1-\Ui_{1})K}{\Ui_{k}}-1+\frac{\Li_{1}K^{0.49}}{\Ui_{k}}) \exp(-A^{(t)}_k)+1}\\
            &\geq \Omega(\epsilon^{\frac{1}{2}}),
        \end{align*}
        where $(a)$ follows from the fact that $\frac{x}{1+x}$ increases w.r.t. $x>0$.
\end{proof}
\begin{lemma}\label{lemk4t2}
     Suppose \Cref{hpk4} holds at iteration $ T_{3,k}< t\leq T^{\epsilon}_{4,k}$.  If $\xq=v_k$ and $\pit\in\esi$, then the following holds.
     \begin{enumerate}
         \item  $\Attn^{(t)}_n=\Theta\left(\frac{1-\Attn^{(t)}_{k}}{K}\right)$ for $n\not=1,k$;
         \item $\Attn^{(t)}_1\in\left[\Omega(\frac{1-\Attn^{(t)}_{k}}{K^{0.51}}), O\left(\frac{1-\Attn^{(t)}_{k}}{K^{0.49}}\right)\right]$.
     \end{enumerate}
\end{lemma}
\begin{proof}
We first have
    \begin{align*}
        \frac{\Attn^{(t)}_n}{1-\Attn^{(t)}_{k}} =\frac{|\cV_n|\exp(B^{(t)}_{k,n})}{\sum_{m\not=k}|\cV_m|\exp(B^{(t)}_{k,m})}.
    \end{align*}
    If $n\not=1$, by \Cref{hpk4}, we have 
    \begin{itemize}
        \item for $m\not=1,k$, $e^{-O(\frac{\log(K)-\log(\epsilon)}{K})}\leq \exp(B^{(t)}_{k,m}-B^{(t)}_{k,n})\leq e^{O(\frac{\log(K)-\log(\epsilon)}{K})}$,
        \item for $m=1$, $e^{-0.51\log(K)-O(\frac{\log(K/\epsilon)}{K})}\leq \exp(B^{(t)}_{k,1}-B^{(t)}_{k,n})\leq 0$.
    \end{itemize}
  Note that when $\pit\in\esi$, we have $\frac{|\cV_m|}{|\cV_n|}=\Theta(1)$, and $\frac{|\cV_1|}{|\cV_n|}=\Theta(K)$. Then combining with the fact that  $-\log(\epsilon)\ll \operatorname{polylog}(K)$,  
    we obtain
    \begin{align*}
       \frac{\Attn^{(t)}_n}{1-\Attn^{(t)}_{k}} =\frac{|\cV_n|\exp(B^{(t)}_{k,n})}{\sum_{m\not=k}|\cV_m|\exp(B^{(t)}_{k,m})}= \frac{1}{\sum_{m\not=k}\frac{|\cV_m|}{|\cV_n|}\exp(B^{(t)}_{k,m}-B^{(t)}_{k,n})}=\Theta\left(\frac{1}{K}\right).
    \end{align*}
For $n=1$, by \Cref{hpk4}, we have $$e^{0.49\log(K)-O(\frac{\log(K/\epsilon)}{K})}\leq \exp(B^{(t)}_{k,m}-B^{(t)}_{k,1})\leq ^{0.51\log(K)+O(\frac{\log(K/\epsilon)}{K})},$$ for $m\not=1$. 
Combining with the fact that $\frac{|\cV_m|}{|\cV_1|}=\Theta\left(\frac{1}{K}\right)$ when $\pit\in\esi$, and $-\log(\epsilon)\ll \operatorname{polylog}(K)$, we have
\begin{align*}
    \frac{\Attn^{(t)}_1}{1-\Attn^{(t)}_{k}}= \frac{1}{\sum_{m\not=k}\frac{|\cV_m|}{|\cV_n|}\exp(B^{(t)}_{k,m}-B^{(t)}_{k,n})}\leq O(\frac{1}{K\cdot\frac{1}{K}\cdot e^{0.49\log(K)-O(\frac{\log(K/\epsilon)}{K})} +1}) =O\left(\frac{1}{K^{0.49}}\right);
 \end{align*}
 and 
 \begin{align*}
    \frac{\Attn^{(t)}_1}{1-\Attn^{(t)}_{k}}= \frac{1}{\sum_{m\not=k}\frac{|\cV_m|}{|\cV_n|}\exp(B^{(t)}_{k,m}-B^{(t)}_{k,n})}\geq O(\frac{1}{K\cdot\frac{1}{K}\cdot e^{0.51\log(K)+O(\frac{\log(K/\epsilon)}{K})} +1}) \geq \Omega\left(\frac{1}{K^{0.51}}\right).
 \end{align*}
\end{proof}
\subsubsection{Controlling Gradient Updates in Phase IV}
\begin{lemma}\label{pk4a}
    At each iteration $T_{3,k}<t \leq T_{4,k}^{\epsilon}$, if \Cref{hpk4} holds, then $\alpha_k^{(t)}\geq0$ and satisfies 
    \begin{align*}
        \alpha_{k}^{(t)}\geq \Omega\left(\frac{\epsilon}{K}\right).
    \end{align*}
\end{lemma}
\begin{proof}
    The analysis is similar to that for \Cref{pk3a}, but we need to be more careful about the lower bound of $1-\Attn^{(t)}_{k}$.
    By the gradient expression, we have
    \begin{align*}
          \alpha_{k}^{(t)}&=
        \mathbb{E}\left[\mathbf{1}\{\xq=v_k\}\Attn^{(t)}_{k }\cdot \left(
         \sum_{m\not= k}{\Attn_{m}^{(t)}}^2+(1-\Attn^{(t)}_k)^2\right)\right]\\
         &\geq p_k\cdot\mathbb{P}(P\in\cE^*)\mathbb{E}\left[\Attn^{(t)}_{k }\cdot \left(
         \sum_{m\not= k}{\Attn_{m}^{(t)}}^2+(1-\Attn^{(t)}_k)^2\right)\mid\{\xq=v_k\} \cap \cE^*\right]\\
        &\geq p_k\cdot\mathbb{P}(P\in\cE^*)\mathbb{E}\left[\Attn^{(t)}_{k }\cdot (1-\Attn^{(t)}_k)^2\mid\{\xq=v_k\} \cap \cE^*\right]\\
        &\geq \Omega\left(\frac{\epsilon}{K}\right)
        \end{align*}
where the last inequality follows from \Cref{lemk4t1} and our choice of $p_k$.
\end{proof}
\begin{lemma}\label{pk4b}
    At each iteration $ T_{3,k}<t\leq T_{4,k}^{\epsilon}$, if \Cref{hpk4} holds, then given $k\geq 2$, 
     $\beta_{k,1}^{(t)}$ satisfies 
    \begin{align*}
       - O\left(\frac{\alpha^{(t)}_k}{K^{0.49}}\right)\leq  \beta_{k,n}^{(t)}\leq 0.
    \end{align*}
\end{lemma}
\begin{proof}
 Following the computations similar to those for  \Cref{pk1b}, we have
\begin{align*} &\sum_{m\not= k}{\Attn^{(t)}_{m}}^2-\Attn^{(t)}_{1} -\Attn^{(t)}_{k}(1-\Attn^{(t)}_k)\\
&\quad \leq -(1-\Attn^{(t)}_k-\Attn^{(t)}_1)(\Attn^{(t)}_{1}+\Attn^{(t)}_{k}-\max_{m\not=1,k}\Attn^{(t)}_{m}).
\end{align*}
Therefore,
\begin{align*}
\beta_{k,1}^{(t)}&\leq \mathbb{E}\left[\mathbf{1}\{\xq=v_k\cap \esi\}\Attn^{(t)}_{1}\cdot \left(
  \sum_{m\not= k}{\Attn^{(t)}_{m}}^2-\Attn^{(t)}_{1} -\Attn^{(t)}_{k}(1-\Attn^{(t)}_k)\right)\right]\\
  &\quad +\mathbb{E}\left[\mathbf{1}\{\xq=v_k\cap {\esi}^{c}\}\Attn^{(t)}_{1}\cdot \left(
  \sum_{m\not= k}{\Attn^{(t)}_{m}}^2\right)\right]\\
  &\stackrel{(a)}
  {\leq}p_k\cdot\mathbb{P}(\pit\in\esi)\cdot\mathbb{E}\left[-\Omega(\frac{(1-\Attn^{(t)}_{k})^2}{K^{0.51}})\bigg|\{\xq=v_k\} \cap \esi\right]\\
  &\quad +p_k\cdot\mathbb{P}({\esi}^{c})\\
&{\leq} p_k\cdot \left(-\Omega\left(\frac{\epsilon}{K^{0.51}}\right)\right)+3 p_k\exp \left(-\frac{\ci^2 N}{25 K^2}\right)\\
&<  0.
\end{align*}
where $(a)$ follows from \Cref{lemk4t2}, 
and the last inequality holds since
\begin{align*}
\frac{\epsilon}{K^{0.51}}\geq \frac{\exp(-\operatorname{polylog}(K))}{K^{0.51}} \gg  \exp \left(-\frac{\ci^2 N}{25 K^2}\right).
\end{align*}

Moreover,

\begin{align*}
    -\beta_{k,1}^{(t)}&\leq \mathbb{E}\left[\mathbf{1}\{\xq=v_k\cap \esi\}\Attn^{(t)}_{1}\cdot \left(
     \Attn^{(t)}_{1} +\Attn^{(t)}_{k}(1-\Attn^{(t)}_k)\right)\right]\\
      &\quad +\mathbb{E}\left[\mathbf{1}\{\xq=v_k\cap {\esi}^{c}\}\Attn^{(t)}_{1}\cdot \left(
        \Attn^{(t)}_{1} +\Attn^{(t)}_{k}(1-\Attn^{(t)}_k)\right)\right]\\
      &
      {\leq}p_k\cdot\mathbb{P}(\pit\in\esi)\cdot\mathbb{E}\left[\cdot O(\Attn^{(t)}_{1}(1-\Attn^{(t)}_{k}))\mid\{\xq=v_k\} \cap \esi\right]\\
      &\quad +2p_k\cdot\mathbb{P}({\esi}^{c})\\
    &
    {\leq} O\left(\frac{\alpha^{(t)}_k}{K^{0.49}}\right)+ 6 p_k\exp \left(-\frac{\ci^2 N}{25 K^2}\right)\\
    &=O\left(\frac{\alpha^{(t)}_k}{K^{0.49}}\right).
    \end{align*}
\end{proof}
\begin{lemma}\label{pk4c}
    At each iteration $ T_{3,k}<t\leq T_{4,k}^{\epsilon}$, if \Cref{hpk4} holds, then given $k\geq 2$, for any $n\not=1,k$, $\beta_{k,n}^{(t)}$ satisfies 
    \begin{align*}
       - O\left(\frac{\alpha^{(t)}_k}{K}\right)\leq  \beta_{k,n}^{(t)}\leq 0.
    \end{align*}
\end{lemma}
\begin{proof}
        Note that  conditioned on the event $\{\xq=v_k\} \cap \esi$, by \Cref{lemk4t1,lemk4t2}, we have $\Attn_k^{(t)}=\Omega(1)$, $\max_{m\not=k}\Attn_{m}=O\left(\frac{1}{K^{0.49}}\right)$. Thus, we obtain
\begin{align} \sum_{m\not= k}{\Attn^{(t)}_{m}}^2-\Attn^{(t)}_{n} -\Attn^{(t)}_{k}(1-\Attn^{(t)}_k)&\leq  \max_{m\not=k}\Attn^{(t)}_{m}\sum_{m\not= k}\Attn^{(t)}_{m}-\Attn^{(t)}_{k}(1-\Attn^{(t)}_k)\nonumber\\
&=-(1-\Attn^{(t)}_k)(\Attn^{(t)}_{k}-\max_{m\not=k}\Attn^{(t)}_{m})\nonumber\\
&\leq -\Omega (1-\Attn^{(t)}_k). \label{ik:eq3}
\end{align}
Therefore,
\begin{align*}
    \beta_{k,n}^{(t)}&\leq \mathbb{E}\left[\mathbf{1}\{\xq=v_k\cap \cE^*\}\Attn^{(t)}_{n}\cdot \left(
      \sum_{m\not= k}{\Attn^{(t)}_{m}}^2-\Attn^{(t)}_{n} -\Attn^{(t)}_{k}(1-\Attn^{(t)}_k)\right)\right]\\
      &\quad +\mathbb{E}\left[\mathbf{1}\{\xq=v_k\cap {\esi}^{c}\}\Attn^{(t)}_{n}\cdot \left(
      \sum_{m\not= k}{\Attn^{(t)}_{m}}^2\right)\right]\\
      &
      {\leq}p_k\cdot\mathbb{P}(\pit\in\esi)\cdot\mathbb{E}\left[-\Omega(\frac{(1-\Attn_k)^2}{K})\bigg|\{\xq=v_k\} \cap \cE^*\right]+p_k\cdot\mathbb{P}(\pit\in {\esi}^{c})\\
&
{\leq} p_k\cdot \left(-\Omega\left(\frac{\epsilon}{K}\right)\right)+3 p_k\exp \left(-\frac{\ci^2 N}{25 K^2}\right)\\
&{\leq} 0,
  \end{align*}
where the last inequality holds since
\begin{align*}
\frac{\epsilon}{K}\gg \frac{\exp(-\operatorname{polylog}(K))}{K}\gg \exp \left(-\frac{\ci^2 N}{25 K^2}\right).
\end{align*}
Moreover, 
we have 
    \begin{align*}
     - \beta_{k,n}^{(t)}&\leq p_k\mathbb{E}\left[\Attn^{(t)}_{n}\cdot \left(
 \Attn^{(t)}_{n} +\Attn^{(t)}_{k}(1-\Attn^{(t)}_k)\right)\mid\{\xq=v_k\} \cap \esi\right]+2p_k\mathbb{P}({\esi}^{c})\\
 &\leq p_k\mathbb{E}\left[\Theta(\frac{1-\Attn^{(t)}_{k}}{K})\cdot O\left(
\Attn^{(t)}_{k}(1-\Attn^{(t)}_k)\right)\bigg|\{\xq=v_k\} \cap \esi\right]\\
&\qquad+6 p_k\exp \left(-\frac{\ci^2 N}{25 K^2}\right)\\
 &=p_k\mathbb{E}\left[O\left(\frac{\Attn^{(t)}_{k}(1-\Attn^{(t)}_{k})^2}{K}\right)\bigg|\{\xq=v_k\} \cap \esi\right]+6 p_k\exp \left(-\frac{\ci^2 N}{25 K^2}\right)\\
 &\leq O\left(\frac{\alpha_{k}^{(t)}}{K}\right).
    \end{align*}
\end{proof}
\subsubsection{End of Phase IV}
\begin{lemma}\label{end4}
    Given $k>1$, and $0<\epsilon<1$, suppose $\operatorname{polylog}(K)\gg \log(\frac{1}{\epsilon})$. Then \Cref{hpk4} holds for all  $T_{3,k}<t\leq T^{\epsilon}_{4,k}=T_{3,k}+O(\frac{K\log(K\epsilon^{-\frac{1}{2}})}{\eta\epsilon})$, and at iteration $t=T^{\epsilon}_{4,k}+1$, we have 
    \begin{enumerate}
        \item  $\tilde{\cL}_{k}(\theta^{T^{\epsilon}_{4,k}+1})<\frac{\epsilon}{2}$;
        \item If $\xq=v_k$ and $\pit\in\esi$, we have  $(1-\Attn_{k}^{(T^{\epsilon}_{4,k}+1)})^2\leq O(\epsilon)$.
    \end{enumerate}
       
\end{lemma}
\begin{proof}

The  existence of $T_{4,k}^{\epsilon}=T_{3,k}+O(\frac{K\log(K\epsilon^{-\frac{1}{2}})}{\eta\epsilon})$ directly follows from  \Cref{pk4a}.

It is easy to verify \Cref{hpk4} holds at $t=T_{3,k}+1$.  Now we suppose \Cref{hpk4} holds for all iterations $\leq t-1$, and prove it holds at $t$. 

By \Cref{pk4a}, we have $\alpha_{k}^{(t-1)}\geq 0$. Thus $A_{k}^{(t)}= A_{k}^{(t-1)}+\eta\alpha_{k}^{(t-1)}\geq \log(K)$. Moreover, 
by \Cref{pk4b}, we have $$|B_{k,1}^{(t)}- B_{k,1}^{(T_{3,k}+1)}|\leq O(\frac{A_{k}^{(t)}- A_{k}^{(T_{3,k}+1)}}{K^{0.49}}),$$
which immediately implies $$B_{k,1}^{(t)}\geq -O(A_{k}^{(t)} /K^{0.49})-O(\log(K)/K^{0.48})-0.51\log(K).$$ 

For $m\not=1,k$, by \Cref{pk4c}, we have $$|B_{k,m}^{(t)}-B_{k,m}^{(T_{3,k}+1)}|\leq O(\frac{A_{k}^{(t)}- A_{k}^{(T_{3,k}+1)}}{K})\leq O(\log(K/\epsilon)/K).$$ Thus $$|B_{k,m}^{(t)}|\leq O(\log(K/\epsilon)/K)+O(\log(K)/K)=O(\log(K/\epsilon)/K).$$

At iteration $t=T^{\epsilon}_{4,k}+1$, we have $$A_{k}^{(t)}\geq \log\left( \left(\frac{e(1-\Li_1)K+\Ui_1K^{0.51}}{\Li_k}-e\right)\left(\left(\frac{3}{\epsilon}\right)^{\frac{1}{2}}-1\right)\right).$$ Thus when $\{\xq=v_{k}\}\cap\{ \pit\in\esi\}$, we obtain
\begin{align*}
    1-\Attn^{(t)}_{k}
             &=\frac{\sum_{m\not=k}\frac{|\cV_m|}{|\cV_k|}\exp(B^{(t)}_{k,m}-A^{(t)}_k)}{\sum_{m\not=k}\frac{|\cV_m|}{|\cV_k|}\exp(B^{(t)}_{k,m}-A^{(t)}_k)+1}\\
             &\leq \frac{ \exp( \max_{m\not= 1,k} B^{(t)}_{k,m}-A^{(t)}_k) (\frac{N-|\cV_1|}{|\cV_k|}-1)+\exp( B^{(t)}_{k,1}-A^{(t)}_k) \frac{|\cV_1|}{|\cV_k|}}{\exp( \max_{m\not= 1,k} B^{(t)}_{k,m}-A^{(t)}_k) (\frac{N-|\cV_1|}{|\cV_k|}-1)+\exp( B^{(t)}_{k,1}-A^{(t)}_k) \frac{|\cV_1|}{|\cV_k|}+1}\\
             &\leq \frac{ \exp(1-A^{(t)}_k) (\frac{(1-\Li_1)K}{\Li_{k}}-1)+\exp(-0.49\log(K)-A^{(t)}_k) \frac{\Ui_1K}{\Li_{k}}}{ \exp(1-A^{(t)}_k) (\frac{(1-\Li_1)K}{\Li_{k}}-1)+\exp(-0.49\log(K)-A^{(t)}_k) \frac{\Ui_1K}{\Li_{k}}+1}\\
             &=  \frac{ \left( (\frac{e(1-\Li_1)K+\Ui_1K^{0.51}}{L_k}-e) \right) \exp(-A^{(t)}_k)  }{ \left( (\frac{e(1-\Li_1)K+\Ui_1K^{0.51}}{\Li_k}-e) \right) \exp(-A^{(t)}_k)+1}\\
             &\leq \frac{((\frac{3}{\epsilon})^{\frac{1}{2}}-1)^{-1}}{((\frac{3}{\epsilon})^{\frac{1}{2}}-1)^{-1}+1}\\
             &=(\epsilon/3)^{\frac{1}{2}}.
         \end{align*}
We further derive
\begin{align*}
  \tilde{\cL}_{k}(\theta^{(t)})
    &=\frac{1}{2}\mathbb{E}\left[\mathbf{1}\{\pit\in\esi\}\left(
        \sum_{m\not= k}{\Attn^{(t)}_{m}}^2+(1-\Attn^{(t)}_k)^2\right)\bigg| \xq=v_{k} \right]\\
        &\leq \frac{1}{2}\mathbb{P}\left( \pit\in\esi\right)\cdot\mathbb{E}\left[\left(O\left(\frac{1}{K^{0.49}}\right)+1\right)(1-\Attn^{(t)}_k)^2\bigg| \xq=v_{k}\cap \pit\in\esi \right]\\
        &\leq  \frac{1}{2}\left(1+O\left(\frac{1}{K^{0.49}}\right)\right)\cdot \frac{\epsilon}{3} \\
        &\leq \frac{\epsilon}{2}.
\end{align*} 

\end{proof}

\subsection{Proof of Theorem~\ref{thm: unblc} for Under-represented Features}\label{app:sec:un}
\begin{theorem}[Restatement of \Cref{thm: unblc} for Under-represented Features]
Suppose $p_1=\Theta(1)$ and $p_k=\Theta\left(\frac{1}{K}\right)$ for $2\leq k \leq K$. For any $0<\epsilon<1$, suppose $N\geq \poly(K)$, and $\polylog(K)\gg \log(\frac{1}{\epsilon})$. We apply GD to train the loss function given in \cref{eq:obj}. Then the following results hold.
\begin{enumerate}[label={\arabic*.}]
    \item The prediction error for \textbf{under-represented} feature converges: for $v_k$ with $2\leq k \leq K$, 
    with at most 
    $T_k=O(\frac{\log(K)K^2}{\eta}+\frac{K \log\left(K\epsilon^{-\frac{1}{2}}\right)}{\epsilon\eta})$ GD iterations, 
    $\cL_k(\theta^{(T_k)})\leq \cL_{k}^{*}+\epsilon$, where $\cL^{*}_{k}=\Theta(e^{-\poly(K)})$ is the global minimum of \cref{eq-obj-k};
    \item Attention score concentrates: for each
 $2\leq k\leq K$, if the query token is $v_k$, then after $T_{k}$ 
 iterations, with probability at least  $1-e^{-\Omega(\poly(K))}$, the one-layer transformer nearly ``pays all attention" to input tokens featuring $v_k$:  $(1-\Attn^{(T_{k})}_k)^2 \leq O(\epsilon)$.
\end{enumerate}
    \end{theorem}
\begin{proof}
  The first statement is obtained by letting $T_k=T^{\epsilon}_{4,k}+1$, and  combining \Cref{end4}, \Cref{app:lem:opti1} and   \Cref{app:lem:opti2}, which lead to 
  \begin{align*}
     \cL_k(\theta^{(T_k)})-\cL_{k}^{*}\leq \cL_k(\theta^{(T_k)})-\Loi_k\leq \tilde{\cL}_{k}(\theta^{(T_k)})+3 \exp \left(-\frac{\ci^2 N}{25 K^2}\right)<\epsilon.
  \end{align*}
  The second statement directly follows from  \Cref{end4}.
\end{proof}
\newpage
\section{Analysis for Imbalanced Case: Dominant Feature}\label{app:im:v1}

In this section, we delve into the analysis of prediction error when the query token features the dominant feature $v_1$. The training dynamics for the dominant feature $v_1$ are relatively straightforward, comprising only a single phase.

Note that at the beginning $t=0$, we already have the following lemma.
\begin{lemma}
    If $\xq=v_1$ and $\pit\in\esi$, at $t=0$, we have $\Attn_{1}^{(0)}=\Omega(1)$, $\Attn_{k}^{(0)}=O\left(\frac{1}{K}\right)$ for $k>1$. 
\end{lemma}

Thus, the learning process directly enters the convergence phase, which is defined as follows.  Given any  $0<\epsilon <1$, define 
  \begin{align*}
    T^{\epsilon}_{1,*}\triangleq \max\left\{t\geq 0: A_{1}^{(t)}-\max_{m\not=1} B_{1,m}^{(t)}\leq \log\left(\left(\frac{1}{\Li_1}-1\right)\left(\left(\frac{2}{\epsilon}\right)^{\frac{1}{2}}-1\right)\right) \right\}.
  \end{align*}
\begin{hypothesis}\label{hpi1}
    For $0 \leq t\leq T^{\epsilon}_{1,*}$, 
    {suppose $\operatorname{polylog}(K)\gg \log(\frac{1}{\epsilon})$}. Then the following holds.
\begin{enumerate}[label={\alph*}.]
    \item $A_{1}^{(t)}$ is monotonically increasing and $A_{1}^{(t)}\in [0, O(\log(1/\epsilon))]$;
    \item $B_{k,n}^{(t)}$ is monotonically decreasing and $-O(\frac{A_{1}^{(t)}}{K}) \leq B_{1,n}^{(t)}\leq 0$ for any $n\not=1$.
\end{enumerate}
\end{hypothesis}
\subsection{Technical Lemmas}
We first introduce several useful technical lemmas.
\begin{lemma}\label{lemit1}
    Suppose \Cref{hpi1} holds at iteration $0<t\leq T_{1,*}^{\epsilon}$. 
    If $\xq=v_1$ and $\esi\in\pit$, then the following holds
    \begin{enumerate}
    \item  $\Attn^{(t)}_1=\Omega(1)$;
        \item $(1-\Attn^{(t)}_{1})^2\geq \Omega(\epsilon)=\Omega(
            \exp\left(-\operatorname{polylog}(K)\right))$.
    \end{enumerate}
\end{lemma}
\begin{proof}
    Since $\xq=v_1$, we have
    \begin{align*}
\Attn^{(t)}_{1}&=\frac{|\cV_1|\exp(A^{(t)}_k)}{\sum_{m\not=1}|\cV_m|\exp(B^{(t)}_{1,m})+|\cV_1|\exp(A^{(t)}_k)}\\
        &=\frac{1}{\sum_{m\not=k}\frac{|\cV_m|}{|\cV_k|}\exp(B^{(t)}_{k,m}-A^{(t)}_k)+1}.
    \end{align*}
    By \Cref{hpi1}, we have
    \begin{align*}
        \Attn^{(t)}_{1}&
                \geq \frac{1}{\sum_{m\not=k}\frac{|\cV_m|}{|\cV_k|}\exp(B^{(0)}_{k,m}-A^{(0)}_k)+1}\\
                &\geq \frac{1}{(\frac{N}{\Li_1N}-1)+1}\geq \Omega(1).
            \end{align*}
    
   On the other hand, by the definition of $T_{1,*}^{\epsilon}$, we have
   \begin{align*}
   1-\Attn^{(t)}_{1}
            &=\frac{\sum_{m\not=1}\frac{|\cV_m|}{|\cV_1|}\exp(B^{(t)}_{1,m}-A^{(t)}_k)}{\sum_{m\not=1}\frac{|\cV_m|}{|\cV_1|}\exp(B^{(t)}_{1,m}-A^{(t)}_1)+1}\\
            &\overset{(a)}{\geq} \frac{ \exp( \min_{m\not= 1} B^{(t)}_{1,m}-A^{(t)}_1) (\frac{N}{|\cV_1|}-1)}{\exp( \min_{m\not= 1} B^{(t)}_{1,m}-A^{(t)}_1) (\frac{N}{|\cV_1|}-1)+1}\\
            &\geq \frac{ \exp( \min_{m\not= 1} B^{(t)}_{1,m}-A^{(t)}_1) (\frac{1}{\Ui_{1}}-1)}{\exp( \min_{m\not= 1} B^{(t)}_{1,m}-A^{(t)}_1) (\frac{1}{\Ui_{1}}-1)+1}\\
            &= \frac{ \exp( \max_{m\not= 1} B^{(t)}_{1,m}-A^{(t)}_1-\Delta B_{1}^{(t)}) (\frac{1}{\Ui_{1}}-1)}{\exp( \max_{m\not= 1} B^{(t)}_{1,m}-A^{(t)}_1-\Delta B_{1}^{(t)}) (\frac{1}{\Ui_{1}}-1)+1}\\
            &\geq \frac{(\frac{1}{p_1L_1}-1)^{-1}((\frac{2}{\epsilon})^{\frac{1}{2}}-1)^{-1}\cdot e^{-O(\frac{\operatorname{polylog}(K)}{K})}(\frac{1}{\Ui_{1}}-1)}{(\frac{1}{\Li_1}-1)^{-1}((\frac{2}{\epsilon})^{\frac{1}{2}}-1)^{-1}(\frac{1}{\Ui_{1}}-1)e^{-O(\frac{\operatorname{polylog}(K)}{K})}+1}\\
            &\geq \Omega(\epsilon^{\frac{1}{2}}).
        \end{align*}
        where $\Delta B_{k}^{(t)}=\max_{m\not= k}B^{(t)}_{k,m}-\min_{m\not= k}B^{(t)}_{k,m}=O(\frac{ A_{k}^{(t)}}{K})$, $(a)$ follows from the fact that $\frac{x}{1+x}$ increases w.r.t. $x>0$.
\end{proof}
\begin{lemma}\label{lemit2}
     Suppose \Cref{hpi1} holds at iteration $ 0\leq  t\leq T^{\epsilon}_{1,*}$.  If $x_{\tau,query}=v_1$ and $P\in\cE^*$,  for $n\not=1$, the following holds
  $$\Attn^{(t)}_n=\Theta\left(\frac{(1-\Attn^{(t)}_{1})}{K}\right).$$
\end{lemma}
\begin{proof}

We first have
    \begin{align*}
\Attn^{(t)}_{n}&=\frac{|\cV_n|\exp(B^{(t)}_{1,n})}{\sum_{m\not=k}|\cV_m|\exp(B^{(t)}_{1,m})+|\cV_1|\exp(A^{(t)}_1)}.
    \end{align*}
    By \Cref{hpi1}, we have 
    $$e^{-O(\frac{p\log(\frac{1}{\epsilon})}{K})}\leq \exp(B^{(t)}_{k,m}-B^{(t)}_{k,n})\leq e^{O(\frac{p\log(\frac{1}{\epsilon})}{K})}.$$ Combining with the fact that $-\log(\epsilon)\ll \operatorname{polylog}(K)$,  
    we obtain
    \begin{align*}
       \frac{\Attn^{(t)}_n}{1-\Attn^{(t)}_{k}} =\frac{|\cV_n|\exp(B^{(t)}_{k,n})}{\sum_{m\not=k}|\cV_m|\exp(B^{(t)}_{k,m})}= \frac{1}{\sum_{m\not=k}\frac{|\cV_m|}{|\cV_n|}\exp(B^{(t)}_{k,m}-B^{(t)}_{k,n})}=\Theta\left(\frac{1}{K}\right).
    \end{align*}
\end{proof}
\subsection{Controlling  Gradient Updates}
\begin{lemma}\label{pi2a}
    At each iteration $0 \leq t \leq T_{1,*}^{\epsilon}$, if \Cref{hpi1} holds then $\alpha_1^{(t)}\geq0$ and satisfies 
    \begin{align*}
        \alpha_{k}^{(t)}\geq \Omega({\epsilon}).
    \end{align*}
\end{lemma}
\begin{proof}
    By the gradient expression, we have
    \begin{align*}
         \alpha_{1}^{(t)}&=
        \mathbb{E}\left[\mathbf{1}\{\xq=v_1\}\Attn^{(t)}_{1 }\cdot \left(
         \sum_{m\not= 1}{\Attn_{m}^{(t)}}^2+(1-\Attn^{(t)}_1)^2\right)\right]\\
         &\geq p_1\cdot\mathbb{P}(\pit\in\esi)\mathbb{E}\left[\Attn^{(t)}_{k }\cdot \left(
         \sum_{m\not= k}{\Attn_{m}^{(t)}}^2+(1-\Attn^{(t)}_k)^2\right)\bigg|\{\xq=v_k\} \cap \esi\right]\\
        &\geq p_1\cdot\mathbb{P}(\pit\in\esi)\mathbb{E}\left[\Attn^{(t)}_{k }\cdot (1-\Attn^{(t)}_k)^2\bigg|\{\xq=v_k\} \cap \cE^*\right]\\
        &\geq \Omega({\epsilon})
        \end{align*}
where the last inequality follows from \Cref{lemit1} and our choice of $p_1$.
\end{proof}
\begin{lemma}\label{pi2b}
    At each iteration $ 0 \leq t\leq T_{1,*}^{\epsilon}$, if \Cref{hpi1} holds, then for any $n\not=1$, $\beta_{1,n}^{(t)}$ satisfies 
    \begin{align*}
       - O\left(\frac{\alpha^{(t)}_1}{K}\right)\leq  \beta_{1,n}^{(t)}\leq 0.
    \end{align*}
\end{lemma}
\begin{proof}
        Note that  conditioned on the event $\{\xq=v_1\} \cap \esi$,  by \Cref{lemit1,lemit2}, we have $\Attn_1^{(t)}=\Omega(1)$, and $\max_{m\not=1}\Attn^{(t)}_{m}=O\left(\frac{1}{K}\right)$. Thus, we further obtain
\begin{align} \sum_{m\not= 1}&{\Attn^{(t)}_{m}}^2-\Attn^{(t)}_{n} -\Attn^{(t)}_{1}(1-\Attn^{(t)}_1)\nonumber\\
&\leq  \max_{m\not=1}\Attn^{(t)}_{m}\sum_{m\not= 1}\Attn^{(t)}_{m}-\Attn^{(t)}_{1}(1-\Attn^{(t)}_1)\nonumber\\
&=-(1-\Attn^{(t)}_1)(\Attn^{(t)}_{1}-\max_{m\not=1}\Attn^{(t)}_{m})\nonumber\\
&\leq -\Omega (1-\Attn^{(t)}_1). \label{i1:eq3}
\end{align}
Therefore,
\begin{align*}
    \beta_{1,n}^{(t)}&\leq \mathbb{E}\left[\mathbf{1}\{\xq=v_1\cap \esi\}\Attn^{(t)}_{n}\cdot \left(
      \sum_{m\not= 1}{\Attn^{(t)}_{m}}^2-\Attn^{(t)}_{n} -\Attn^{(t)}_{1}(1-\Attn^{(t)}_1)\right)\right]\\
      &\quad +\mathbb{E}\left[\mathbf{1}\{\xq=v_1\cap {\esi}^{c}\}\Attn^{(t)}_{n}\cdot \left(
      \sum_{m\not= 1}{\Attn^{(t)}_{m}}^2\right)\right]\\
      &\stackrel{(a)}{\leq}p_1\cdot\mathbb{P}(\pit\in\esi)\cdot\mathbb{E}\left[-\Omega\left(\frac{(1-\Attn^{(t)}_1)^2}{K}\right)\bigg|\{\xq=v_1\} \cap \esi\right]+p_1\cdot\mathbb{P}({\esi}^{c})\\
&\stackrel{(b)}{\leq} p_1\cdot \left(-\Omega\left(\frac{\epsilon}{K}\right)\right)+3 p_1\exp \left(-\frac{\ci^2 N}{25 K^2}\right)\\
&{\leq} 0
  \end{align*}
  where $(a)$ follows from \cref{i1:eq3} and \Cref{lemit2}, $(b)$ follows from \Cref{lemit1}, and the last inequality holds since
\begin{align*}
\frac{\epsilon}{K}\gg \frac{\exp(-\operatorname{polylog}(K))}{K}\gg \exp \left(-\frac{\ci^2 p_2 N}{25 K^2}\right).
\end{align*}
Moreover, 
we have 
    \begin{align*}
     - \beta_{1,n}^{(t)}&\leq p_1\mathbb{E}\left[\Attn^{(t)}_{n}\cdot \left(
 \Attn^{(t)}_{n} +\Attn^{(t)}_{1}(1-\Attn^{(t)}_1)\right)\mid\{\xq=v_1\} \cap \esi\right]+2p_1\mathbb{P}({\esi}^{c})\\
 &\leq p_1\mathbb{E}\left[\Theta\left(\frac{1-\Attn^{(t)}_{1}}{K}\right)\cdot O\left(
\Attn^{(t)}_{k}(1-\Attn^{(t)}_1)\right)\bigg|\{\xq=v_1\} \cap \esi\right]+6 p_1\exp \left(-\frac{\ci^2 N}{25 K^2}\right)\\
 &=p_1\mathbb{E}\left[O\left(\frac{\Attn^{(t)}_{1}(1-\Attn^{(t)}_{1})^2}{K}\right)\bigg|\{\xq=v_1\} \cap \esi\right]+6 p_1\exp \left(-\frac{\ci^2 p^2N}{25 K^2}\right)\\
 &\leq O\left(\frac{\alpha_{1}^{(t)}}{K}\right).
    \end{align*}
\end{proof}
\subsection{End of the Phase}
\begin{lemma}\label{i1end}
    Given $0<\epsilon<\frac{1}{2}$, suppose $\operatorname{polylog}(K)\gg \log(\frac{1}{\epsilon})$. Then \Cref{hpi1} holds for all  $ 0\leq t\leq T^{\epsilon}_{1,*}=O(\frac{\log(\epsilon^{-\frac{1}{2}})}{\eta\epsilon})$, and at iteration $t=T^{\epsilon}_{1,*}+1$, we have
    \begin{enumerate}
        \item $\tilde{\cL}_{1}(\theta^{T^{\epsilon}_{1,*}+1})<\epsilon/2$;
        \item If $\xq=1$ and $\pit\in\esi$, we have $(1-\Attn_{1}^{(T^{\epsilon}_{1,*}+1)})^2\leq O(\epsilon)$.
    \end{enumerate}
\end{lemma}
\begin{proof}
    We first prove the existence of $T_{1,*}^{\epsilon}$. 
    Recall that
    \begin{align*}
        T^{\epsilon}_{1,*}= \max\left\{t\geq 0: A_{1}^{(t)}-\max_{m\not=1} B_{1,m}^{(t)}\leq \log\left(\left(\frac{1}{\Li_1}-1\right)\left(\left(\frac{2}{\epsilon}\right)^{\frac{1}{2}}-1\right)\right) \right\}.
      \end{align*}
    
      When $t\in[0,T_{1,*}^{\epsilon}]$,  we can simply lower bound the update of $A_{k}^{(t)}-\max_{m\not=k} B_{k,m}^{(t)}$ as
      \begin{align*}
        A_{k}^{(t+1)}-\max_{m\not=k} B_{k,m}^{(t+1)}
        \geq A_{k}^{(t+1)}\geq A_{k}^{(t)}+\Omega\left(\frac{\eta\epsilon}{K}\right).
      \end{align*}
      Therefore,  at most $T_{1,*}^{\epsilon}=O(\frac{\log\left((\frac{1}{\Li_1}-1)((\frac{2}{\epsilon})^{\frac{1}{2}}-1)\right)}{\eta\epsilon})=O(\frac{\log(\epsilon^{-\frac{1}{2}})}{\eta\epsilon})$ iterations are needed before $A_{k}^{(t)}-\max_{m\not=k} B_{k,m}^{(t)}$ exceeds $\log\left((\frac{1}{\Li_1}-1)((\frac{2}{\epsilon})^{\frac{1}{2}}-1)\right)$.

      It is easy to verify \Cref{hpi1} holds at $t=0$.  Now we suppose \Cref{hpi1} holds for all iterations $0\leq t-1$, and prove it holds at $t$.

      By \Cref{pi2a}, we have $\alpha_{1}^{(t-1)}\geq 0$. Thus $A_{1}^{(t)}\geq A_{1}^{(t-1)}\geq 0$. By \Cref{pi2b}, we have $-O\left(\frac{\alpha_{1}^{(t-1)}}{K}\right) \leq \beta_{1,n}^{(t-1)}\leq 0$. Thus,
      \begin{align*}
         -B_{1,n}^{(t)}&\leq -B_{1,n}^{(t-1)}+\eta O\left(\frac{\alpha_{1}^{(t-1)}}{K}\right)\\
         &\leq O\left(\frac{A_{1}^{(t-1)}}{K}\right)+\eta O\left(\frac{\alpha_{1}^{(t-1)}}{K}\right)\\
         &\leq O\left(\frac{A_{1}^{(t)}}{K}\right).
      \end{align*}

 Moreover, by the definition of $T^{\epsilon}_{1,*}$, for any $t\leq T^{\epsilon}_{1,*}$, we immediately have 

\begin{align*}
A_{1}^{(t)}\leq A_{1}^{(t)}-\max_{m\not=1} B_{1,m}^{(t)}\leq \log\left(\left(\frac{1}{\Li_1}-1\right)\left(\left(\frac{2}{\epsilon}\right)^{\frac{1}{2}}-1\right)\right).
\end{align*}

Therefore, $A_{1}^{(t)}\leq O(\log(\frac{1}{\epsilon}))$.

At iteration $t=T^{\epsilon}_{1,*}+1$, we have $A_{1}^{(t)}-\max_{m\not=1} B_{1,m}^{(t)}> \log\left((\frac{1}{\Li_1}-1)((\frac{2}{\epsilon})^{\frac{1}{2}}-1)\right)$. Thus, when $\{\xq=v_{1}\}\cap \{\pit\in\esi\}$, we obtain
\begin{align*}
    1-\Attn^{(t)}_{1}
             &=\frac{\sum_{m\not=1}\frac{|\cV_m|}{|\cV_1|}\exp(B^{(t)}_{1,m}-A^{(t)}_1)}{\sum_{m\not=1}\frac{|\cV_m|}{|\cV_1|}\exp(B^{(t)}_{1,m}-A^{(t)}_1)+1}\\
             &\leq \frac{ \exp( \max_{m\not= 1} B^{(t)}_{1,m}-A^{(t)}_1) (\frac{N}{|\cV_1|}-1)}{\exp( \max_{m\not= 1} B^{(t)}_{1,m}-A^{(t)}_1) (\frac{N}{|\cV_1|}-1)+1}\\
             &\leq \frac{ \exp( \max_{m\not= 1} B^{(t)}_{1,m}-A^{(t)}_1) (\frac{1}{\Li_{1}}-1)}{\exp( \max_{m\not= 1} B^{(t)}_{1,m}-A^{(t)}_1) (\frac{1}{\Li_{1}}-1)+1}\\
             &\leq  \frac{ \left((\frac{1}{\Li_1}-1)((\frac{2}{\epsilon})^{\frac{1}{2}}-1)\right)^{-1}(\frac{1}{\Li_{1}}-1)}{\left((\frac{1}{\Li_1}-1)((\frac{2}{\epsilon})^{\frac{1}{2}}-1)\right)^{-1}(\frac{1}{\Li_{1}}-1)+1}\\
             &=(\epsilon/2)^{\frac{1}{2}}.
         \end{align*}
Similarly, we have
\begin{align*}
  \tilde{\cL}_{1}(\theta^{(t)})
    &=\frac{1}{2}\mathbb{E}\left[\mathbf{1}\{\pit\in{\esi}\}\left(
        \sum_{m\not= k}{\Attn^{(t)}_{m}}^2+(1-\Attn^{(t)}_k)^2\right)\bigg| \xq=v_{k}\right]\\
        &\leq \frac{1}{2}\mathbb{P}\left( \pit\in\esi \right)\cdot \mathbb{E}\left[\left(O\left(\frac{1}{K}\right)+1\right)(1-\Attn^{(t)}_k)^2\bigg|\xq=v_{k}\cap \pit\in\esi \right]\\
        &\leq \frac{1}{2}\cdot \left(1+O\left(\frac{1}{K}\right)\right)\cdot \frac{\epsilon}{2} \\
        &\leq \epsilon/2.
\end{align*} 
\end{proof}

\subsection{Proof of Theorem~\ref{thm: unblc} for Dominant Feature}
\begin{theorem}[Restatement of \Cref{thm: unblc} for Dominant Feature]
Suppose $p_1=\Theta(1)$ and $p_k=\Theta\left(\frac{1}{K}\right)$ for $2\leq k \leq K$. For any $0<\epsilon<1$, suppose $N\geq \poly(K)$, and $\polylog(K)\gg \log(\frac{1}{\epsilon})$. We apply GD to train the loss function given in \cref{eq:obj}. Then the following results hold.
\begin{enumerate}[label={\arabic*.}]
    \item The prediction error for {\bf dominant} feature converges: for $v_1$, with at most $T_1=O(\frac{\log(\epsilon^{-\frac{1}{2}})}{\eta\epsilon})$ GD iterations, 
    $\cL_1(\theta^{(T_1)})\leq \cL_{1}^{*}+\epsilon$, where $\cL^{*}_{1}=\Theta(e^{-\poly(K)})$ is the global minimum of \cref{eq-obj-k}; 
    \item Attention score concentrates: 
 $k=1$, if the query token is $v_k$, then after $T_{k}$ 
 iterations, with probability at least  $1-e^{-\Omega(\poly(K))}$, the one-layer transformer nearly ``pays all attention" to input tokens featuring $v_k$:  $(1-\Attn^{(T_{k})}_k)^2 \leq O(\epsilon)$.
\end{enumerate}
    \end{theorem}
\begin{proof}
  The first statement is obtained by letting $T_1=T^{\epsilon}_{1,*}+1$, and  combining \Cref{i1end}, \Cref{app:lem:opti1} and   \Cref{app:lem:opti2}, which lead to
  \begin{align*}
     \cL_1(\theta^{(T_1)})-\cL_{1}^{*}\leq \cL_1(\theta^{(T_1)})-\Loi_1\leq \tilde{\cL}_{1}(\theta^{(T_1)})+3 \exp \left(-\frac{\ci^2 N}{25 K^2}\right)<\epsilon.
  \end{align*}
  The second statement directly follows from \Cref{i1end}.
\end{proof}

\end{document}